\documentclass[letterpaper,onecolumn]{IEEEtran}

\usepackage{tikz}
\usetikzlibrary{positioning}
\usetikzlibrary{arrows.meta}
\usepackage{amsmath,amsfonts}
\usepackage{array}
\usepackage[caption=false,labelfont=sf,textfont=sf]{subfig}
\usepackage{textcomp}
\usepackage{stfloats}
\usepackage{url}
\usepackage{verbatim}
\usepackage{graphicx}
\usepackage{cite}

\usepackage[linesnumbered,ruled,vlined]{algorithm2e}

\usepackage[nolist,nohyperlinks]{acronym}
  \begin{acronym}
     \acro{SR}{Sharpe ratio}
     \acro{TS}{Thompson sampling}
     \acro{MAB}{Multi-armed bandit}
     \acro{UCB}{upper confidence bound}
     \acro{LCB}{lower confidence bound}
     \acro{KL}{Kullback-Leibler}
     \acro{MV}{mean-variance}
\end{acronym}

\usepackage{xcolor}
\usepackage{float}
\usepackage{svg}
\usepackage{pdfpages}
\usepackage{comment}
\usepackage{graphicx,wrapfig,lipsum}
\usepackage{rotating}
\usepackage{amsmath, amssymb, amsthm}
\usepackage{stmaryrd}
\usepackage{tikz}
\usepackage{verbatim}
\usepackage{url}
\usepackage[utf8]{inputenc}
\usepackage[english]{babel}

\newtheorem{theorem}{Theorem}[]

\newtheorem{remark}{Remark}[]

\newtheorem{lemma}[]{Lemma}

\newtheorem{assumption}{Assumption}[]
\usepackage{multicol}
\usepackage{scalerel}
\usepackage{multicol}
\usepackage{setspace}
\newtheorem{definition}{Definition}
\usepackage{booktabs}
\usepackage{bbm}
\usepackage[normalem]{ulem}
\usepackage{color}
\usepackage{dsfont}
\usepackage{bm}
\usepackage{setspace}
\usepackage{mathtools}
\usepackage{csquotes}
\usepackage{dirtytalk}
\allowdisplaybreaks

\usepackage[margin=0.85cm]{caption}
\makeatletter
\newcommand{\svast}{\bBigg@{3}}
\newcommand{\vast}{\bBigg@{4}}
\newcommand{\Vast}{\bBigg@{5}}
\makeatother

\usepackage{accents}
\newcommand\tbar[1]{\accentset{\rule{.4em}{.5pt}}{#1}}

\usepackage{pifont}

\DeclareMathOperator*{\argmax}{argmax}

\DeclarePairedDelimiterX{\infdivx}[2]{(}{)}{%
  #1\;\delimsize\|\;#2%
}

\title{Order Optimal Regret Bounds for Sharpe Ratio Optimization under Thompson Sampling}

\author{
\IEEEauthorblockN{Mohammad~Taha~Shah}, {\it Graduate Student Member, IEEE},
\IEEEauthorblockN{Sabrina~Khurshid}, {\it Graduate Student Member, IEEE},
\IEEEauthorblockN{Gourab~Ghatak}, {\it Member, IEEE},
\thanks{M.T. Shah is with the CEMSE Division, King Abdullah University of Science and Technology, Thuwal, Saudi Arabia, 23955, (e-mail: mohammadtaha.shah@kaust.edu.sa); Sabrina Khurshid and G. Ghatak are with the Department of Electrical Engineering, IIT Delhi, New Delhi, India, 110016, (e-mail: \{eez218683; gghatak\}@ee.iitd.ac.in)}
}

\begin{document}
\maketitle

\begin{abstract}
In this paper, we study sequential decision-making for maximizing the Sharpe ratio (SR) in a stochastic multi-armed bandit (MAB) setting. Unlike standard bandit formulations that maximize cumulative reward, SR optimization requires balancing expected return and reward variability. As a result, the learning objective depends jointly on the mean and variance of the reward distribution and takes a fractional form. To address this problem, we propose the Sharpe Ratio Thompson Sampling \texttt{SRTS}, a Bayesian algorithm for risk-adjusted exploration. For Gaussian reward models, the algorithm employs a Normal-Gamma conjugate posterior to capture uncertainty in both the mean and the precision of each arm. In contrast to additive mean-variance (MV) formulations, which often require different algorithms across risk regimes, the fractional SR objective yields a single sampling rule that applies uniformly across risk tolerances. On the theoretical side, we develop a regret decomposition tailored to the SR objective and introduce a decoupling approach that separates the contributions of mean and variance uncertainty. This framework allows us to control the interaction between the Gaussian mean samples and the Gamma precision samples arising in the posterior. Using these results, we establish a finite-time distribution-dependent $\mathcal{O}(\log n)$ upper bound on the expected regret. We further derive a matching information-theoretic lower bound using a change-of-measure argument, showing that the proposed algorithm is order-optimal. Finally, experiments on synthetic bandit environments illustrate the performance of \texttt{SRTS} and demonstrate improvements over existing risk-aware bandit algorithms across a range of risk-return settings.
\end{abstract}


\section{Introduction}
The \ac{MAB} problem~\cite{slivkins2019introduction} provides a fundamental framework for sequential decision-making under uncertainty. In the classical stochastic \ac{MAB} formulation, an agent repeatedly selects from a set of $K$ arms, each associated with an unknown reward distribution, and observes a stochastic reward. The objective is to maximize the cumulative reward by balancing exploration and exploitation. A large body of work has characterized this problem, establishing logarithmic upper and lower bounds on regret, asymptotically efficient allocation rules, and finite-time algorithms such as \ac{UCB} and \ac{TS} that achieve near-optimal logarithmic regret for expected reward maximization~\cite{auer2002finite, lai1985asymptotically, thompson1933likelihood}.

Most existing formulations assume risk neutrality, in which arms are evaluated solely by their expected rewards. However, in many practical applications, including quantitative finance, autonomous robotics, and clinical trials, decision makers must also account for the variability of outcomes~\cite{thompson1933likelihood, schwartz2017customer}. For example, in portfolio allocation or algorithmic trading, strategies are evaluated not only by their expected return but also by the associated risk. Early work on risk-aware bandits incorporated risk through additive \ac{MV} objectives, typically written as $\rho \mu - \sigma^2$, where $\rho$ is a risk-tolerance parameter~\cite{sani2012risk, vakili2015risk, zhu2020thompson}. Although such objectives are analytically convenient, they exhibit structural limitations in extreme risk regimes. In particular, as $\rho \to \infty$ or $\rho \to 0$, one component of the objective dominates the other, effectively eliminating the contribution of either the mean or the variance. To address this issue, several works propose switching between separate algorithms, for instance, alternating between mean-based and variance-based \ac{TS} depending on the value of $\rho$~\cite{zhu2020thompson}.

An alternative approach is to use the \ac{SR}, defined as the ratio between expected return and a measure of variability in the reward distribution~\cite{varlashova2024optimal}. Because the \ac{SR} takes a fractional form, it provides a scale-invariant measure of risk-adjusted performance that simultaneously accounts for reward and variability. This formulation naturally couples the mean and variance, avoiding the need for algorithmic switching across risk regimes. Recent work has therefore considered the \ac{SR} and related risk-adjusted metrics directly as bandit objectives~\cite{cassel2018general, khurshid2025optimizing, 11408087}, revealing new statistical and algorithmic challenges compared to standard reward-maximization problems. Despite its conceptual appeal, optimizing the \ac{SR} introduces significant technical difficulties. From a statistical perspective, the algorithm must estimate both the mean and the variance of each arm. From a mathematical standpoint, the presence of the variance in the denominator produces a nonlinear objective that is not sub-Gaussian. Since empirical variance estimates follow Gamma or chi-squared distributions and exhibit heavier tails than Gaussian statistics, classical concentration arguments used in standard bandit analysis are no longer directly applicable. As a result, most existing theoretical results for \ac{SR} optimization rely on frequentist \ac{UCB}-based methods~\cite{cassel2018general, khurshid2025optimizing, 11408087}, which typically employ conservative union bounds that increase the exploration cost.

In this work, we study \ac{SR} maximization from a Bayesian perspective and propose a \ac{TS} algorithm tailored to this objective, called Sharpe Ratio Thompson Sampling -- \texttt{SRTS}. The algorithm maintains posterior distributions over the mean and precision of each arm and selects actions based on sampled \acp{SR}. Using this framework, we develop a regret analysis for the fractional \ac{SR} objective. In particular, we introduce a regret decomposition that relates cumulative \ac{SR} regret to problem-dependent quantities such as the sub-optimal gap and the divergence between reward distributions. For Gaussian rewards, this analysis yields explicit distribution-dependent regret bounds that characterize how both mean and variance affect learning efficiency. We complement the theoretical analysis with experiments on synthetic bandit environments. The results illustrate the behavior of \texttt{SRTS} across different risk regimes and show that it performs favorably compared with existing risk-aware bandit algorithms.

\subsection{Related Literature}
The \ac{MAB} framework has been widely studied for sequential decision-making under uncertainty, with classical formulations focusing on maximizing expected cumulative rewards~\cite{auer2002finite,bubeck2012regret}. A substantial body of work has analyzed exploration strategies using information-theoretic tools such as mutual information, entropy, and \ac{KL} divergence to obtain distribution-dependent regret bounds~\cite{russo2016information,lattimore2020bandit}. However, these formulations assume risk neutrality and evaluate arms solely by their expected rewards, ignoring higher-order distributional characteristics such as reward variability. In contrast, optimizing the \ac{SR} requires estimating both the mean and the variance of the reward distribution, which introduces additional statistical and algorithmic challenges.

Although the \ac{SR} originates from financial portfolio theory, it has also been applied in other machine learning settings. For example, in fairness-aware learning, the Sharpe predictor proposed in~\cite{liu2021sharpe} evaluates classifiers by balancing predictive accuracy and fairness. Similarly, in reinforcement learning from human feedback (RLHF), a Sharpe-inspired criterion has been used to guide prompt selection by balancing expected model improvement and uncertainty~\cite{belakaria2025sharpe}. These examples illustrate how the \ac{SR} can serve more generally as a metric for balancing performance and variability.

\subsubsection{Risk-aware bandits}
Several works have studied risk-sensitive bandit objectives based on \ac{MV} trade-offs~\cite{sani2012risk,vakili2015risk}. In \cite{sani2012risk}, Sani \textit{et al.} proposed the \texttt{MV-UCB} algorithm, which uses variance as the measure of risk and achieves regret of order $\mathcal{O}(\sqrt{n})$. This result was later improved by \cite{vakili2015risk}, which established logarithmic regret bounds. The \ac{MV} criterion has also been examined in the full-information setting by \cite{even2006risk}, where it was shown that achieving sublinear regret under this risk measure is impossible. Alternative risk formulations have also been considered. For instance, \cite{maillard2013robust} introduced the \texttt{RA-UCB} algorithm based on the logarithm of the moment-generating function of the reward distribution and derived high-probability regret guarantees.

In the pure exploration setting, percentile-based risk measures such as Value-at-Risk (VaR) and Conditional Value-at-Risk (CVaR) have also been studied. VaR and CVaR were originally introduced in \cite{artzner1999coherent} and \cite{rockafellar2000optimization}, respectively. In the bandit context, \cite{vakili2015mean} investigated risk-averse \ac{MAB} problems using the VaR of the cumulative reward as the risk criterion and proposed algorithms with poly-logarithmic regret bounds. The work of \cite{kolla2019concentration} established strong concentration guarantees for VaR without requiring assumptions on the tail behavior of the reward distribution. CVaR-based bandits have also been studied under bounded reward assumptions \cite{galichet2013exploration,tamkin2019distributionally}. In \cite{galichet2013exploration}, the guarantees are derived within a risk-neutral bandit framework under the assumption that the arm with the highest mean reward is also optimal under the CVaR objective. In a related direction, \cite{david2016pure} studied the problem of identifying the arm with the largest $\alpha$-quantile of the reward distribution. \ac{TS} approaches for CVaR bandits were later developed in \cite{baudry2021optimal,chang2020risk}. In particular, \cite{baudry2021optimal} established asymptotic optimality for CVaR bandits with bounded rewards, while \cite{chang2020risk} compared CVaR-based \ac{TS} with \ac{LCB}/\ac{UCB} methods and reported improved empirical performance. A comprehensive overview of risk-aware bandit formulations is provided in \cite{tan2022survey}.

\subsubsection{Thompson sampling and Sharpe ratio}
Among Bayesian approaches, \ac{TS} \cite{thompson1933likelihood} has received considerable attention due to its strong empirical performance and theoretical guarantees. Regret analyses for \ac{TS} are often expressed in terms of information gain about the reward distributions \cite{kaufmann2012bayesian,honda2014optimality}. The algorithm has been extensively studied in both Bayesian and frequentist settings \cite{korda2013thompson,agrawal2012analysis,agrawal2017near,russo2014learning,dong2018information}. In the context of risk-aware bandits, \cite{liu2020risk} proposed a \ac{TS} algorithm for \ac{MV} objectives. In contrast, theoretical results for \ac{SR} optimization in bandits remain limited. Existing approaches such as \texttt{UCB-RSSR} \cite{khurshid2025optimizing} and \texttt{UCB-SR} \cite{11408087} provide logarithmic regret upper bounds using \ac{UCB}-based methods. Some recent works have applied \ac{TS} heuristically to \ac{SR}-based portfolio strategies \cite{zhu2019adaptive,morariufinancial,shen2016portfolio}, but these studies do not provide formal regret guarantees.

A key challenge in this setting is that minimizing regret with respect to the \ac{SR} requires learning both the mean and the variance of the reward distributions. Consequently, the analysis must account for uncertainty in multiple parameters. In the related \ac{MV} setting, this theoretical picture is largely complete: algorithms such as \texttt{MV-UCB} \cite{sani2012risk,vakili2015risk} and \texttt{MVTS} \cite{zhu2020thompson} achieve $\mathcal{O}(\log n)$ regret that matches the known $\Omega(\log n)$ lower bounds. For the \ac{SR} objective, however, the corresponding information-theoretic lower bounds have not been established. In this work, we address this gap by developing both upper and lower regret bounds for \ac{SR} maximization using a Bayesian \ac{TS} framework. In particular, we show that the proposed \texttt{SRTS} algorithm achieves $\mathcal{O}(\log n)$ regret and matches the corresponding lower bound up to constant factors.

\subsection{Contributions}
Our main theoretical and algorithmic contributions are summarized below.
\begin{itemize}
    \item \textbf{Regret decomposition for the \ac{SR} objective: } We develop a regret decomposition tailored to the fractional \ac{SR} objective. In classical bandits, regret can be expressed as a weighted sum of the expected number of pulls of suboptimal arms, where the weights correspond to the mean gaps. For the \ac{SR} objective, the fractional structure couples estimation errors in the mean and variance, preventing such a direct decomposition. We introduce a decoupling framework that separates these contributions and enables the expected regret to be expressed as a weighted sum of expected suboptimal pulls, with weights that explicitly capture the joint effect of mean and variance estimation errors.
    \item \textbf{Finite-time regret bounds for \texttt{SRTS}: } We analyze the proposed \texttt{SRTS} algorithm and establish distribution-dependent finite-time regret bounds. In particular, we show that \texttt{SRTS} achieves $\mathcal{O}(\log n)$ expected regret under the \ac{SR} objective, despite the non-sub-Gaussian nature of the fractional metric.
    \item \textbf{Information-theoretic lower bounds: } We derive a model-specific lower bound for \ac{SR} optimization using a change-of-measure argument. The result shows that any consistent policy must incur logarithmic regret, matching the order of the upper bound achieved by \texttt{SRTS} up to constant factors.
    \item \textbf{Empirical evaluation: } We conduct experiments on synthetic bandit environments to evaluate the proposed algorithm across different risk regimes. The results illustrate the behavior of \texttt{SRTS} under varying levels of risk tolerance and show improved performance compared with existing risk-aware bandit methods.
\end{itemize}

\section{Problem Formulation}
\subsection{The Sequential Decision-Making Environment}
We consider a stochastic \ac{MAB} problem with a finite set of arms indexed by $i \in [K] := \{1, \dots, K\}$, where $K \ge 2$. The learning process unfolds over a discrete time horizon of length $n$. At each time step $t \in \{1, 2, \dots, n\}$, the agent selects an action (arm) $\pi(t) \in [K]$ governed by a policy $\pi$, and observes a scalar reward $X_{\pi(t), t}$. For each arm $i \in [K]$, the reward sequence $\{X_{i,t}\}_{t=1}^n$ consists of independent and identically distributed (i.i.d.) random variables drawn from an unknown distribution $\nu_i$. Let $\mathcal{F} = (\nu_1, \dots, \nu_K)$ denote the environment characterized by these distributions. We denote the true mean and variance of arm $i$ as $\mu_i$ and $\sigma_i^2$, respectively. The rewards are assumed to be independent across arms and time steps.

A policy $\pi$ is a sequence of allocation rules where the decision at time $t$ is $\mathcal{H}_{t-1}$-measurable, with $\mathcal{H}_{t-1} = \{\pi(1), X_{\pi(1), 1}, \dots,$\\$ \pi(t-1), X_{\pi(t-1), t-1}\}$ representing the history of actions and observations up to time $t-1$. Let $s_{i,n} = \sum_{t=1}^n \mathbb{I}(\pi(t) = i)$ denote the random number of times arm $i$ is pulled up to horizon $n$.

\subsection{The Risk-Aware Objective: The Sharpe Ratio}
Unlike classical risk-neutral bandits that maximize the cumulative expected reward, our objective is to maximize the risk-adjusted return, formalized via a generalized \ac{SR}.

\begin{definition}
The {\rm \ac{SR}} of an arm $i$ with mean $\mu_i$, variance $\sigma^2_i$, and risk tolerance $\rho$ is 
\begin{align*}
    \xi_i = \frac{ \mu_i}{L_0+ \rho \sigma^2_i},
\end{align*} 
where $L_0$ is the regularization term and is the same for all the arms.
\end{definition}

The regularization term $L_0$ is introduced to stabilize the empirical estimation of the \ac{SR}, particularly during the early exploration phase when empirical variance estimates may be arbitrarily close to zero. We require $L_0 < \min_{i \in [K]} \sigma_i^2$. Without loss of generality, we assume there exists a unique optimal arm, indexed by $1$, such that $\xi_1 = \max_{i \in [K]} \xi_i$. The sub-optimality gap for the \ac{SR} is denoted by $\Delta_i = \xi_1 - \xi_i$, and the mean gap is denoted by $\Lambda_{i,j} = \mu_i - \mu_j$.

\subsection{Algorithmic Reward and Regret}
To evaluate the performance of a policy $\pi$, we define the \textbf{algorithmic empirical \ac{SR}} over the sequence of $n$ plays. The algorithmic mean under policy $\pi$ is the standard sample average:
\begin{align*}
    \tbar{\mu}_n(\pi) = \frac{1}{n} \sum_{t=1}^n X_{\pi(t),t} = \frac{1}{n} \sum_{i=1}^K s_{i,n} \tbar{\mu}_{i,s_{i,n}},
\end{align*}
where $\tbar{\mu}_{i,s_{i,n}}$ is the empirical mean of arm $i$ after $s_{i,n}$ pulls. However, the algorithmic variance under policy $\pi$ inherently couples the empirical variances of individual arms with the variance induced by the algorithm switching between arms with different means:
\begin{align*}
    \tbar{\sigma}_n^2(\pi) = \frac{1}{n} \sum_{t=1}^n \left(X_{\pi(t),t} - \tbar{\mu}_n(\pi)\right)^2 = \frac{1}{n} \sum_{i=1}^K s_{i,n} \tbar{\sigma}_{i,s_{i,n}}^2 + \frac{1}{n} \sum_{i=1}^K s_{i,n} \left(\tbar{\mu}_{i,s_{i,n}} - \tbar{\mu}_n(\pi)\right)^2.
\end{align*}
The algorithmic \ac{SR} achieved by the policy is thus $\tbar{\xi}_n(\pi) = \frac{\tbar{\mu}_n(\pi)}{L_0 + \rho \tbar{\sigma}_n^2(\pi)}$. The optimal policy should choose arm $1$ for all $t \in \{1, 2, \dots, n\}$. For each policy $\pi$, this leads to the definition of the regret, which is the difference of the algorithmic \ac{SR} of the policy and the optimal \ac{SR}.
\begin{definition}
The expected regret of a policy $\pi$ over $n$ rounds is defined as
\begin{align}
    \mathbb{E} [\mathcal{R}_n(\pi)] = n\left(\xi_1 - \mathbb{E}[\tbar{\xi}_n (\pi)]\right).
    \label{eq:regret_1}
\end{align} 
\end{definition}
where $\tbar{\xi}_n (\pi)$ is the algorithmic \ac{SR} under $\pi$.

\subsection{Bounding Pull-Count Variance}
A primary technical challenge in analyzing $\mathbb{E}[\mathcal{R}_n(\pi)]$ stems from the fractional nature of $\tbar{\xi}_n(\pi)$ and the cross-arm dependencies in the algorithmic variance $\tbar{\sigma}_n^2(\pi)$. Unlike cumulative reward, \ac{SR} regret does not decompose linearly into an expected sum over sub-optimal pulls $\mathbb{E}[s_{i,n}]$.

To tractably decompose this regret, it is necessary to control the fluctuations in the pull count $\{s_{i,n}\}_{i=1}^K$. If the variance of the pull count $\mathbb{V}[s_{i,n}]$ is large, concentration arguments for the denominator of the algorithmic \ac{SR} fail. We resolve this by demonstrating that the variance of the arm pull count is tightly bounded regardless of the policy's specific dependency structure, utilizing the Efron-Stein inequality.
\begin{lemma}
\label{lem:efron_s}
For $s_{i,n} = \sum_{t=1}^n \mathbb{I}\!\left(\pi(t) = i\right)$, the variance of $s_{i,n}$ satisfies
\begin{align*}
    \mathbb{V}[s_{i,n}] \leq \frac{n}{2}.
\end{align*}
\end{lemma}
\begin{proof}
The proof is deferred to Appendix~\ref{pro:le1}, leveraging the bounded difference property of independent perturbations on the pull count.
\end{proof}
Lemma~\ref{lem:efron_s} ensures that fluctuations in $s_{i,n}$ remain controlled. Consequently, the dominant contribution to expected regret arises from the expected number of suboptimal pulls $\mathbb{E}[s_{i,n}]$, enabling the development of a pseudo-regret bound in subsequent sections.

\subsection{Notations}
Before detailing the theoretical regret decomposition and the proposed algorithm, we summarize the key mathematical notations used throughout this paper in Table~\ref{tab:notation}. For clarity, the notations are categorized into true environmental parameters, empirical statistics, Bayesian posterior variables, and concentration analysis terms.

\begin{table}[t]
\centering
\caption{Summary of Key Notations}
\label{tab:notation}
\renewcommand{\arraystretch}{1.3} 
\begin{tabular}{ll}
\toprule
\textbf{Symbol} & \textbf{Description} \\
\midrule
\multicolumn{2}{c}{\textit{Environment \& True Parameters}} \\
\midrule
$K, n$ & Number of arms and total time horizon \\
$X_{i,t}$ & Reward of arm $i$ at time step $t$ \\
$\mu_i, \sigma_i^2$ & True mean and variance of arm $i$ \\
$\tau_i$ & True precision of arm $i$ ($\tau_i = \sigma_i^{-2}$) \\
$\rho$ & Risk-tolerance parameter \\
$L_0$ & Variance regularization constant ($L_0 > 0$) \\
$\xi_i$ & True \ac{SR} of arm $i$: $\frac{\mu_i}{L_0 + \rho\sigma_i^2}$ \\
$\Delta_i$ & Sub-optimality gap in \ac{SR}: $\xi_1 - \xi_i$ \\
$\Lambda_{i,j}$ & Pairwise mean difference: $\mu_i - \mu_j$ \\
\midrule
\multicolumn{2}{c}{\textit{Algorithmic \& Empirical Statistics}} \\
\midrule
$s_{i,n}$ & Number of times arm $i$ is pulled up to time $n$ \\
$\tbar{\mu}_{i,q}, \tbar{\sigma}_{i,q}^2$ & Empirical mean and variance of arm $i$ after $q$ pulls \\
$\tbar{\mu}_n(\pi), \tbar{\sigma}_n^2(\pi)$ & Algorithmic mean and variance under policy $\pi$ \\
$\tbar{\xi}_n(\pi)$ & Algorithmic empirical \ac{SR} \\
$D$ & Regularized algorithmic risk (denominator of $\tbar{\xi}_n(\pi)$) \\
$\mathcal{R}_n(\pi)$ & Cumulative regret of policy $\pi$ up to time $n$ \\
\midrule
\multicolumn{2}{c}{\textit{Bayesian \ac{TS} Variables}} \\
\midrule
$\hat{\mu}_{i,t}, s_{i,t}$ & Posterior Gaussian parameters \\
$\alpha_{i,t}, \beta_{i,t}$ & Posterior Gamma parameters \\
$\theta_{i,t}$ & Sampled mean for arm $i$ at time $t$ \\
$\tau_{i,t}$ & Sampled precision for arm $i$ at time $t$ \\
$\hat{\xi}_{i,t}$ & Sampled \ac{SR} for arm $i$ at time $t$ \\
\midrule
\multicolumn{2}{c}{\textit{Regret Analysis \& Concentration}} \\
\midrule
$\varepsilon$ & Global error margin for \ac{SR} concentration \\
$\varepsilon_\mu, \varepsilon_\sigma$ & Partitioned error margins for mean and variance \\
$w_\mu, w_\sigma$ & First-order sensitivity weights for margin partitioning \\
$H_{i,q}$ & Conditional posterior probability of concentration failure \\
$I(\nu, \nu^\prime)$ & \ac{KL} divergence between distributions \\
\bottomrule
\end{tabular}
\end{table}

\section{Regret Decomposition}
A key difference between risk-aware bandits and classical risk-neutral formulations lies in the structure of the regret. In classical cumulative reward maximization, the expected regret decomposes linearly into the sum of expected sub-optimal arm pulls weighted by their respective mean gaps. For the \ac{SR} objective, such a decomposition is not immediate. The algorithmic empirical \ac{SR}, $\tbar{\xi}_n(\pi)$, is a fractional statistic where the numerator (algorithmic mean) and denominator (algorithmic variance) are highly coupled random variables. Consequently, the expected regret $\mathbb{E}[\mathcal{R}_n(\pi)]$ involves the expectation of a ratio, which necessitates controlling a non-trivial covariance penalty between the mean estimates and the reciprocal of the variance estimates.

The following theorem provides a pseudo-regret decomposition for the \ac{SR} objective that expresses the expected regret in terms of the expected number of pulls of each arm.
\begin{theorem}
\label{th:th_dec}
The expected regret of a policy $\pi$ over $n$ rounds for \ac{SR} is given as,
\begin{align*}
    \mathbb{E}\left[\mathcal{R}_n(\pi)\right] \leq \sum_{i=1}^K \mathbb{E}[s_{i,n}] \left( \Delta_i + \frac{\xi_i \rho\left(\frac{\Lambda_{\max}^2}{2} + \sum_{j \neq i} \sigma_j^2 \right)}{L_0 +  \frac{\rho}{2} \Lambda_{\max}^2 + \rho \sum_{j=1}^K \sigma_j^2} \right) + X_1,
\end{align*}
where $X_1$ is a constant independent of $s_{i,n}$ and $n$, $\Lambda_{\max} = \max \Lambda_{i,j}^2$.
\end{theorem}

\subsection{Proof Sketch of Theorem~\ref{th:th_dec}}
The complete proof is deferred to Appendix~\ref{pro:th1}. The core technical challenge in proving Theorem~\ref{th:th_dec} is decoupling the expectation of the algorithmic \ac{SR} into expectations of its constituent moments. By definition, $\mathbb{E}[\mathcal{R}_n(\pi)] = n\mathbb{E}[\xi_1 - \tbar{\xi}_n(\pi)]$. Expanding $\tbar{\xi}_n(\pi)$ yields a ratio where both the numerator and denominator depend heavily on the random pull count $s_{i,n}$. Applying Jensen’s inequality provides a deterministic lower bound on the denominator, but introduces a negative covariance penalty i.e. $\text{Cov} \left(\tbar{\mu}_n(\pi), \frac{1}{D}\right)$, where $D$ represents the empirical denominator of the algorithmic \ac{SR}, defined as $D = L_0 + \rho \tbar{\sigma}_n^2(\pi)$. To bound this penalty, we apply the Cauchy-Schwarz inequality, which reduces the problem to bounding the individual variances of the algorithmic mean and the reciprocal algorithmic variance.
\begin{itemize}
    \item \textbf{Variance of the Algorithmic Mean: } We decompose this using the law of total variance conditioned on the pull counts $s_{i,n}$. Due to the negative correlation among arm pulls induced by the sequential policy, this variance naturally scales as $\mathcal{O}(1/n)$.
    \item \textbf{Variance of the Reciprocal Denominator: } We observe that the reciprocal mapping $D \mapsto 1/D$ is $L_0^{-2}$-Lipschitz on the domain $D \geq L_0$. This allows us to bound the variance of the reciprocal strictly by the variance of the denominator itself.
    \item \textbf{Controlling Pull-Count Fluctuations: } The denominator's variance contains fourth-order moment terms of the sample means. The analytical bottleneck here is the variance of the pull count, $\mathbb{V}[s_{i,n}]$. By invoking the Efron-Stein inequality (Lemma~\ref{lem:efron_s}), we strictly bound $\mathbb{V}[s_{i,n}] \leq n/2$, ensuring that the aggregate variance of the denominator also decays at a rate of $\mathcal{O}(1/n)$.
\end{itemize}
Combining these bounds shows that the covariance term is of order $\mathcal{O}(1/n)$. Multiplying by the horizon $n$ absorbs this contribution into the constant $X_1$, resulting in a decomposition of the expected regret in terms of the expected number of suboptimal pulls.
\begin{remark}
Notice that the coefficient weighting $\mathbb{E}[s_{i,n}]$ in Theorem~\ref{th:th_dec} is not merely the Sharpe gap $\Delta_i$. It includes an additive term $\frac{\xi_i \rho\left(\frac{\Lambda_{\max}^2}{2} + \sum_{j \neq i} \sigma_j^2 \right)}{L_0 +  \frac{\rho}{2} \Lambda_{\max}^2 + \rho \sum_{j=1}^K \sigma_j^2}$. This term rigorously captures the \textbf{cost of risk-awareness}. In standard reward-only bandits ($\rho = 0$), this term vanishes, recovering the classical $\sum \mathbb{E}[s_{i,n}]\Delta_i$ pseudo-regret. For $\rho>0$, the additional term captures the contribution of variance and cross-arm mean differences to the regret.
\end{remark}

\section{\texttt{SRTS} algorithm for the Gaussian distribution}
A primary operational challenge in risk-aware sequential learning is that the optimal exploration-exploitation tradeoff is highly sensitive to the unknown true moments of the reward distributions and the risk tolerance parameter $\rho$. Designing deterministic confidence bounds that are uniformly tight across all risk profiles is mathematically restrictive. We bypass this limitation by proposing the \texttt{SRTS} algorithm, a Bayesian policy that intrinsically balances risk and reward by operating directly on the posterior parameter distributions.

\subsection{The Bayesian Model and Posterior Updates}
We consider an environment where the reward of each arm $i \in [K]$ is drawn from a Gaussian distribution, $X_{i,t} \sim \mathcal{N}(\mu_i, \sigma_i^2)$ where both the mean $\mu_i$ and the precision $\tau_i = \sigma_i^{-2}$ are unknown. An important step in \ac{TS} algorithms is parameter updating based on Bayes' rule. To facilitate exact, computationally tractable sequential updates, we consider the general conjugate prior for a Gaussian distribution with unknown mean and precision, namely the Normal-Gamma prior. A Normal-Gamma distribution, i.e. $\text{Normal-Gamma} (\mu, s, \alpha,\beta)$ is parameterized by four intuitive quantities: a mean estimate $\hat{\mu}$, an effective observation count $s$, a Gamma shape parameter $\alpha$, and a Gamma rate parameter $\beta$. The prior relationship is structured as:
\begin{align*}
    \tau \sim \text{Gamma}(\alpha, \beta), \,\text{and}\, \theta \mid \tau \sim \mathcal{N}\!\left(\tbar{\mu}, \frac{1}{s}\right).
\end{align*}
Because the Normal-Gamma distribution is the exact conjugate prior for a Gaussian likelihood, the posterior distribution retains the same parametric structure. After $s_{i,t-1}$ observations, the parameters of the posterior distribution are updated via Bayes' rule as follows:
\begin{align*}
    \alpha_{i,t-1} = \alpha_{i,0} + \frac{s_{i,t-1}}{2}, \,\text{and}\, \beta_{i,t-1} = \beta_{i,0} + \frac{s_{i,t-1}}{2}\tbar{\sigma}^2_{i,s_{i,t-1}},
\end{align*}
where $\alpha_{i,0}$ and $\beta_{i,0}$ represent the initial non-informative prior values (set to $1/2$ in our implementation). The exact posterior distribution for arm $i$, conditioned on the history $\mathcal{H}_{t-1}$, is therefore straightforwardly defined by the empirical statistics :
\begin{align*}
    \tau_i \mid \mathcal{H}_{t-1} \sim \text{Gamma}(\alpha_{i,t-1}, \beta_{i,t-1}), \quad \theta_i \mid  \mathcal{H}_{t-1} \sim \mathcal{N}\!\left(\tbar{\mu}_{i,t-1}, \frac{1}{s_{i,t-1}}\right).
\end{align*}

\subsection{\texttt{SRTS} Sampling Rule}
At the beginning of each period $t$, the player samples a pair of parameters from each arm's posterior distribution and plays the arm according to the optimal action under these sampled parameters. Specifically, for each arm $i$, \texttt{SRTS} first samples a precision $\tau_{i,t}$ from the posterior Gamma distribution. Conditioned on this sampled precision, it then samples a mean $\theta_{i,t}$ from the posterior Gaussian distribution. Using this sampled mean-precision pair, the algorithm constructs the \textbf{Thompson sample} of \ac{SR}
\begin{align}
    \hat{\xi}_{i,t} = \frac{\theta_{i,t}}{L_0 + \frac{\rho}{\tau_{i,t}}}.
\end{align}
The player then selects the arm $i(t) = \argmax_{i \in [K]} \hat{\xi}_{i,t}$, observes the reward, and updates the sufficient statistics. The complete procedure is summarized in Algorithm \ref{alg:alg1}.

\subsection{Behavior Across Risk Regimes}
A fundamental strength of the \texttt{SRTS} framework is its ability to dynamically unify learning objectives across disparate risk profiles without requiring heuristic algorithmic switching. This differs from additive \ac{MV} formulations~\cite{zhu2020thompson}, where different regimes of $\rho$ may motivate different algorithmic treatments.

This behavior can be seen directly from the sampled statistic $\hat{\xi}_{i,t}$:
\begin{itemize}
    \item \textbf{Reward Maximization Regime ($\rho \to 0$):} When the risk penalty approaches zero, the problem reduces strictly to standard expected reward maximization. The sampled statistic simplifies to $\hat{\xi}_{i,t} = \theta_{i,t} / L_0$. Because the variance samples $\tau_{i,t}$ are suppressed by the negligible $\rho$, the algorithm relies entirely on the sampled Gaussian means $\theta_{i,t}$. Consequently, \texttt{SRTS} structurally recovers the exact behavior and $\mathcal{O}(\log T / \Delta^2)$ (see Remark~\ref{re:rem1}) regret scaling of classical \ac{TS} for standard \acp{MAB}.
    \item \textbf{Extreme Risk-Aversion Regime ($\rho \to \infty$):} Conversely, as $\rho \to \infty$, the variance penalty in the denominator asymptotically dominates the regularizer $L_0$, and the sampling statistic approaches $\hat{\xi}_{i,t} \approx (\theta_{i,t} \cdot \tau_{i,t}) / \rho$. Crucially, this does \textit{not} degrade into a pure variance minimization policy. Since the mean sample $\theta_{i,t}$ remains in the numerator, the algorithm never blindly selects the lowest-variance arm; rather, it aggressively seeks the optimal, scale-free proportion of return per unit of risk.
\end{itemize}
Hence, \texttt{SRTS} performs robustly across all $\rho \in \mathbb{R}^+$, and the same posterior sampling mechanism covers both low- and high-risk regimes without requiring changes to the update rules or the action-selection rule. This property will be reflected in the regret analysis developed in the next section.
\begin{algorithm}[t]
\caption{Sharpe Ratio Thompson Sampling - \texttt{SRTS}}
\label{alg:alg1}
\SetAlgoLined
    \vspace*{0.2cm}
    \KwData{{$K, n, L_0, \rho $ }}
    \textbf{Initialization: } $\hat{\mu}_{i,0} =0 , \alpha_{i,0} =\frac{1}{2}, \beta_{i,0} = \frac{1}{2}, s_{i,0} =0$ \;
    \For{each $t=1, \dots, K$}{
        Play arm $a_i = t$, and observe reward $X_{i(t),t} $.\;
        Update $(\hat{\mu}_{t,t-1}, \alpha_{t,t-1}, \beta_{t,t-1}, s_{t,t-1})$\;
    }
    \For{each $t = K+1, K+2, . . . , n$}{
        Sample $\tau_{i,t}$ from Gamma($\alpha_{i,t-1}, \beta_{i,t-1}$).\;
        Sample $\theta_{i,t}$ from $\mathcal{N}\left(\hat{\mu}_{i,t-1}, \frac{1}{s_{i,t-1}}\right)$.\;
        Play arm $a_i = \arg\max_{i\in[K]} \frac{ \theta_{i,t}} {L_0 + \frac{\rho}{\tau_{i,t}}}$, and observe reward $X_{i(t),t} $.\;
        Update $(\hat{\mu}_{i(t),t-1}, \alpha_{i(t),t-1}, \beta_{i(t),t-1}, s_{i(t),t-1})$.\;
    }
\end{algorithm}

\begin{algorithm}[t]
\caption{Update $\left(\hat{\mu}_{i,t-1}, \alpha_{i,t-1}, \beta_{i,t-1}, s_{i,t-1}\right)$}
\SetAlgoLined
    \vspace*{0.2cm}
    \textbf{Input: } Prior parameters $\hat{\mu}_{i,t-1}, \alpha_{i,t-1}, \beta_{i,t-1}, s_{i,t-1}$ and new sample $X_{i(t),t}$\;
    Update mean: $\hat{\mu}_{i,t} = \frac{s_{i,t-1} \hat{\mu}_{i,t-1} + X_{i(t),t}}{s_{i,t-1} + 1}$. \;
    Update the number of samples, the shape parameter, and the rate parameter: $s_{i,t} = s_{i,t-1} + 1$, $\alpha_{i,t} = \alpha_{i,t-1} + 0.5$, and $\beta_{i,t} = \beta_{i,t-1} + \frac{s_{i,t-1}}{s_{i,t-1} + 1} \cdot \frac{\left(X_{i(t),t} - \hat{\mu}_{i,t-1} \right)^2}{2}$. \;
\end{algorithm}

\section{Regret analysis for Gaussian bandits}
\label{sec:regret_analysis}
In this section, we establish the fundamental theoretical guarantees of the \texttt{SRTS} algorithm. There are two primary analytical difficulties in bounding the regret for \texttt{SRTS}. First, the Thompson sample of the \ac{SR} is not sub-Gaussian; since it is a ratio of random quantities (mean over variance), it exhibits heavier, sub-exponential tails. This necessitates deriving sufficiently tight lower and upper bounds on the posterior distribution's tail probability. Second, a structural difficulty arises from the hierarchical nature of the model parameters: the posterior draws for the mean and variance come from different distributions, and their dependence complicates the analysis compared to standard \acp{MAB}. 

These interacting challenges, visually summarized in Fig.~\ref{fig:srts_hierarchy}, can be formalized as follows:
\begin{enumerate}
    \item \textbf{Hierarchical Parameter Generation:} As illustrated in Fig.~\ref{fig:srts_hierarchy}, the generative process for the Thompson samples operates through a multi-layered hierarchy. The posterior samples for the mean ($\theta_{i,t}$) and precision ($\tau_{i,t}$) are drawn from completely disparate distribution families-Gaussian and Gamma, respectively. While these Bayesian samples are drawn independently based on the number of arm pulls, they are hierarchically driven by their underlying frequentist empirical estimators ($\tbar{\mu}_{i,q}$ and $\tbar{\sigma}_{i,q}^2$). Ultimately, these independent streams are coupled non-linearly within the \ac{SR} objective, requiring the regret analysis to take expectations over the entire interacting chain.
    \item \textbf{Non-sub-Gaussian Fractional Metric:} Second difficulty is the formulation of the sampled \ac{SR} itself, $\hat{\xi}_{i,t} = \frac{\theta_{i,t}}{L_0 + \rho/\tau_{i,t}}$. Because the precision variable $\tau_{i,t}$ (which follows a Gamma distribution) resides in the denominator, the resulting metric is fundamentally non-sub-Gaussian. This renders standard sub-Gaussian tail bounds completely ineffective for evaluating the posterior probability of sufficient exploration.
\end{enumerate}

\begin{figure}[htbp]
\centering
\begin{tikzpicture}[
  node distance=1.2cm and 2cm,
  box/.style={draw, rectangle, align=center, minimum height=1.2cm, minimum width=4cm, thick, rounded corners=3pt},
  arrow/.style={-latex, thick}
]

\node[box] (emp_mean) {Empirical Mean\\ $\tbar{\mu}_{i,q} \sim \mathcal{N}(\mu_i, \sigma_i^2/q)$};
\node[box, right=of emp_mean] (emp_var) {Empirical Variance\\ $\frac{q \tbar{\sigma}_{i,q}^2}{\sigma_i^2} \sim \chi^2_{q-1}$};

\node[box, below=of emp_mean] (ts_mean) {Thompson Mean\\ $\theta_{i,t} \sim \mathcal{N}\left(\tbar{\mu}_{i,q}, \frac{1}{q}\right)$};
\node[box, below=of emp_var] (ts_prec) {Thompson Precision\\ $\tau_{i,t} \sim \mathrm{Gamma}(\alpha_{i,q}, \beta_{i,q})$};

\node[box, below=3cm of emp_mean, xshift=2.75cm] (sr) {Sampled Sharpe Ratio\\ $\hat{\xi}_{i,t} = \frac{\theta_{i,t}}{L_0 + \rho / \tau_{i,t}}$};

\draw[arrow] (emp_mean) -- (ts_mean);
\draw[arrow] (emp_var) -- (ts_prec);
\draw[arrow] (ts_mean) |- (sr);
\draw[arrow] (ts_prec) |- (sr);

\end{tikzpicture}
\caption{The hierarchical dependence structure of the \texttt{SRTS} generative process. The non-linear fractional formulation of the \ac{SR} intrinsically couples the sub-Gaussian mean sample with the heavy-tailed Gamma precision sample, bypassing standard \ac{MAB} concentration bounds.}
\label{fig:srts_hierarchy}
\end{figure}

From Theorem~\ref{th:th_dec}, it suffices to bound the expected number of times a suboptimal arm is chosen, $\mathbb{E}[s_{i,n}]$. To formalize the probability of correctly evaluating an arm, we introduce the concentration event:
\begin{align}
    \mathcal{G}_i(t) := \Big\{\hat{\xi}_{i,t} \leq \xi_1 - \varepsilon \Big\},
\end{align}
which corresponds to the scenario in which the sampled \ac{SR} of suboptimal arm $i$ is correctly identified as at least $\varepsilon$ smaller than that of the optimal arm at period $t$. Intuitively, $\mathcal{G}_i(t)$ occurs with high probability once arm $i$ has been sufficiently explored. However, due to the randomized nature of posterior sampling, arm $i$ may still be selected even when the complement $\mathcal{G}_i^c(t)$ holds.

\textbf{The Decoupling Framework:} To circumvent the aforementioned challenges and bound this concentration event, we introduce a novel decoupling framework. Rather than attempting to bound the intractable joint distribution of the fractional \ac{SR}, we algebraically isolate the mean error from the variance error. By applying the Law of Total Probability and strategically partitioning the integration domain, we break the joint probability into independent marginal tail bounds for the Gaussian mean and the Gamma precision. 

Furthermore, to prevent theoretical inflation of the regret bounds, we do not split the error margin $\varepsilon$ equally. Instead, we introduce dynamically weighted error budgets, $\varepsilon_\mu$ and $\varepsilon_\sigma$, for mean and variance errors, respectively. More details in Section~\ref{app:error_budget}. The proportional partition of error budget explicitly mirrors the first-order sensitivities of the \ac{SR}, ensuring that the heavy-tailed nature of the variance estimation does not act as an analytical bottleneck.

Consequently, leveraging this framework, the regret decomposition for \texttt{SRTS} is organized into two complementary regimes:
\begin{itemize}
    \item \textbf{Exploration Regime:} When posterior uncertainty is still large, the concentration event $\mathcal{G}_i(t)$ may fail with non-negligible probability. In this phase, the algorithm frequently samples suboptimal arms to reduce uncertainty about their respective \ac{SR}s.
    \item \textbf{Exploitation Regime:} Once sufficient evidence has been gathered, $\mathcal{G}_i(t)$ holds with high probability. Nevertheless, arm $i$ can still be chosen with small probability, contributing a residual but well-controlled amount to the cumulative regret.
\end{itemize}
Let $\mathbb{P}_t(\cdot) = \mathbb{P} \left(\cdot \mid \mathcal{H}_{t-1}\right)$ be the probability measure conditioned on the history up to time $t-1$. We define the conditional failure probability as $H_{i,q} := \mathbb{P}_t \left(\mathcal{G}_i^c(t) \mid s_{i,t} = q \right)$. Utilizing the generic decoupling framework of Lattimore and Szepesvári (Theorem 36.2 in \cite{lattimore2020bandit}), we establish the following bound on the suboptimal pull count.
\begin{theorem}
\label{th:pull_bound}
For any suboptimal arm $i$, the expected number of pulls under the \texttt{SRTS} algorithm is bounded by
\begin{align}
    \mathbb{E} [s_{i,n}] &\leq \mathbb{E}\left[\sum_{q=0}^{n-1}\left(\frac{1}{H_{1,q}} - 1\right)\right] + \mathbb{E}\left[\sum_{q=0}^{n-1}\mathbb{I}\left\{H_{i,q}>\frac{1}{n} \right\} \right]
    \label{eq:lemma_1}\\
    &\leq 1 + \max \left\{\frac{2\log{(2n)}}{\big(\Lambda_{1,i} - \varepsilon_\mu L_0\big)^2} , \frac{\log{(2n)}}{h\left(\frac{\sigma_i^2\left(1 - \frac{\varepsilon_\sigma}{\xi_1}\right)}{\sigma_1^2}\right)}\right\} + \frac{X_2}{\varepsilon^3} + \frac{X_3}{\varepsilon^2} + \frac{X_4}{\varepsilon} + X_5.
    \label{eq:res_1}
\end{align}
where $h(x) = (x - 1 - \log x)/2$, the values $\varepsilon_\mu, \varepsilon_\sigma > 0$ are the partitioned error margins for the mean and standard deviation such that $\varepsilon_\mu + \varepsilon_\sigma = \varepsilon$, and $X_2, \dots, X_5$ are finite constants independent of $n$.
\end{theorem}

\subsection{Proof Sketch of Theorem \ref{th:pull_bound}}
Complete proof is present in Appendix~\ref{pro:th2}. Equation~\eqref{eq:lemma_1} serves as the critical bridge between the probability of concentration failure and the total expected pulls. The first summation quantifies the contribution of the exploration regime (where the optimal arm is under-sampled), while the second indicator summation bounds the exploitation regime (the residual probability of incorrectly pulling arm $i$). To bound these terms, we mathematically decouple the hierarchical posterior.
\begin{itemize}
    \item \textbf{Bounding the Exploration Term:} The first term is evaluated by bounding the lower tail of the optimal arm's empirical \ac{SR}. To overcome the coupled fractional metric, we split the global error margin $\varepsilon$, into the mean-budget $\varepsilon_\mu$ and precision-budget $\varepsilon_\sigma$. This isolates the Gaussian mean sample from the Gamma precision sample, allowing us to derive independent marginal tail bounds. We integrate these bounds over the joint posterior to show that the expected exploration penalty converges to a finite $\mathcal{O}(1/\varepsilon^3)$ polynomial limit (formalized in Lemma~\ref{lem:B1} in the Appendix).
    \item \textbf{Bounding the Exploitation Term:} The second term is evaluated by analyzing the upper tail of the suboptimal arm. Once sufficient evidence has been gathered, $\mathcal{G}_i(t)$ holds with high probability. We control the residual selection probability by inverting the decoupled tail bounds to solve for the required number of pulls $q$. This guarantees that the complement event decays rapidly enough to bound the sum by a pure logarithmic threshold $\mathcal{O}(\log n)$ (formalized in Lemma~\ref{lem:B2} in the Appendix).
\end{itemize}
Summing the explicit finite-time constant from Lemma~\ref{lem:B1} and the logarithmic boundary from Lemma~\ref{lem:B2} yields the final, order-optimal inequality in (\ref{eq:res_1}). To ensure the mathematical validity of the exploitation bound in \eqref{eq:res_1}, the partitioned error margins must be strictly bounded such that they do not eclipse the true sub-optimality gaps. Specifically, we require $\varepsilon_\mu < \Lambda_{1,i} / L_0$ to prevent the mean-divergence denominator from vanishing, and $\varepsilon_\sigma < \xi_1$ to ensure the argument of the strictly positive function $h(\cdot)$ remains strictly positive. As we subsequently tune the global margin to decay with the horizon, $\varepsilon = (\log n)^{-1/4}$, these separation conditions are naturally satisfied for any sufficiently large $n$, guaranteeing that the logarithmic leading constants remain strictly finite and well-defined.

\subsection{Finite-Time Regret of \texttt{SRTS}}
To derive the final regret bound, we must balance the constant polynomial penalties with the logarithmic growth. By dynamically tuning the error budget as $\varepsilon = \left(\log {n}\right)^{-1/4}$ and taking the limit as $n \to \infty$, we demonstrate that the fractional exploration terms remain sub-logarithmic. Substituting the bound on $\mathbb{E} [s_{i,n}]$ directly into the pseudo-regret decomposition established in Theorem~\ref{th:th_dec} yields the finite-time expected regret of the algorithm.
\begin{theorem}
\label{th:regret_bound}
The finite-time expected regret of the \texttt{SRTS} algorithm for Gaussian bandits optimizing the \ac{SR} satisfies
\begin{align}
    \mathbb{E} [\mathcal{R}_n (\texttt{SRTS})] &\leq \sum_{i=1}^K \left[1 + \max \left\{\frac{2\log{(2n)}}{\big(\Lambda_{1,i} - \varepsilon_\mu L_0\big)^2} , \frac{\log{(2n)}}{h\left(\frac{\sigma_i^2\left(1 - \frac{\varepsilon_\sigma}{\xi_1}\right)}{\sigma_1^2}\right)}\right\} + X_2 (\log n)^{\frac{3}{4}} + X_3 (\log n)^{\frac{1}{2}} + X_4 (\log n)^{\frac{1}{4}} + X_5 \right] \times \nonumber \\
    &\hspace*{4cm} \left(\Delta_i + \frac{\xi_i \rho\left(\frac{\Lambda_{\max}^2}{2} + \sum_{j \neq i} \sigma_j^2\right)}{L_0 +  \frac{\rho}{2} \Lambda_{\max}^2 + \rho \sum_{i=1}^K \sigma_i^2}\right) + X_1.
\end{align}
where $X_1$ through $X_5$ are distribution-dependent constants independent of the sequence of pulls and the horizon $n$.
\end{theorem}
\begin{remark}
\label{re:rem1}
Theorem \ref{th:regret_bound} explicitly isolates the $\mathcal{O}(\log n)$ leading term. Unlike standard reward-maximization, where the leading constant depends solely on the mean gap, the logarithmic growth here is bottlenecked by the maximum of two quantities: the mean-separation gap and the variance-separation gap (captured by the function $h(\cdot)$). This rigorously confirms that the algorithm pays a quantifiable exploration penalty for learning the second moment.
\end{remark}
\begin{remark}
It is instructive to verify the behavior of our generalized regret bound at the boundary condition $\rho=0$. When the variance penalty is removed, the fractional objective reduces to $\xi_i = \mu_i / L_0$, which is equivalent to classical expected reward maximization scaled by a constant. 

Under this condition, the sub-optimality gap becomes $\Lambda_{1,i} = \Delta_i / L_0$, where $\Delta_i = \mu_1 - \mu_i$. Substituting this into the dominant exploitation term of Theorem~\ref{th:pull_bound} (and taking the limit as the artificial error margin $\varepsilon_\mu \to 0$), the expected number of suboptimal pulls asymptotically scales as:
\begin{align*}
    \mathbb{E}[s_{i,n}] \approx \frac{2 L_0^2 \log(2n)}{\Delta_i^2}.
\end{align*}
This structurally recovers the well-known $\mathcal{O}\left(\frac{\log n}{\Delta^2}\right)$ sub-optimal pull bound of standard Gaussian \ac{TS}. The residual difference in the problem-dependent constants (e.g., the $\mathcal{O}(1/\varepsilon^3)$ exploration overhead) represents the mathematical cost of the union bounds required to decouple the hierarchical Normal-Gamma posterior for general $\rho > 0$.
\end{remark}

\section{Model-specific lower bound regret}
In this section, we establish a \emph{model-specific} lower bound for policies that optimize the \ac{SR}. We establish lower bounds on the model-specific regret achievable by all consistent policies, demonstrating the fundamental optimality of the \texttt{SRTS} algorithm. The proof follows the standard information-theoretic template for bandits, cast as a change-of-measure argument adapted to our \ac{SR} objective. To avoid trivial lower bounds on regret caused by policies that heavily bias toward certain distribution models (e.g., a policy that always plays arm 1), the model-specific setting focuses on the class of consistent policies.
\begin{definition}
A policy $\pi$ is $\alpha$-consistent for $0 < \alpha < 1$ if, for any distribution model and for all suboptimal arms $i \neq 1$, the expected number of pulls satisfies $\mathbb{E}[s_{i,n}] = o(n^\alpha)$.
\end{definition}
To bound the expected number of pulls, we rely on the \ac{KL} divergence to quantify the statistical difficulty of distinguishing between reward distributions.
\begin{definition}
The \ac{KL} divergence between two distributions $\nu$ and $\nu^\prime$ is given by:
\begin{align*}
    I(\nu,\nu^\prime)= \mathbb{E}_{\nu}\left[\log \frac{f{\nu}(X)}{f_{\nu^\prime}(X)}\right],
\end{align*}
where $\mathbb{E}_{\nu}$ denotes the expectation operator with respect to $\nu$, and $f_{\nu}(X)$ and $f_{\nu^\prime}(X)$ are their respective probability density functions.
\end{definition}
Similar to classical lower bounds, we consider the family of one-parameter distribution models. We assume the distribution of arm $i$ is given by $f(\cdot; \theta_i)$, and the entire distribution model $\mathcal{F} = (f(\cdot; \theta_1), \dots, f(\cdot; \theta_K))$ can be parameterized as $\Theta = (\theta_1, \dots, \theta_K)$. The parameters $\theta_i$ take values from a set $\mathcal{U}$ satisfying the following regularity conditions.
\begin{assumption}[Information Continuity]
\label{ass:continuity}
For any parameters $\theta, \lambda, \lambda^\prime \in \mathcal{U}$, and for any $\epsilon > 0$, there exists a $\zeta > 0$ such that $0 < \xi(\lambda^\prime) - \xi(\lambda) < \zeta$ implies $|I(f(\cdot; \theta), f(\cdot; \lambda^\prime)) - I(f(\cdot; \theta), f(\cdot; \lambda))| < \epsilon$. This condition ensures that small changes in the \ac{SR} of distributions result in small changes in their \ac{KL} divergences. Thus, the parameter space $\mathcal{U}$ is sufficiently dense to allow fine distinctions between different arms.
\end{assumption}
\begin{assumption}[Distributional Stability]
\label{ass:stability}
For all $\theta, \lambda \in \mathcal{U}$, let $X$ be a sub-Gaussian random variable with distribution $f(\cdot; \theta)$. The random variable $Y = f(X; \lambda)$ is also sub-Gaussian.
\end{assumption}
Under these assumptions, the expected regret for the \ac{SR} objective can be bounded below by analyzing the probabilistic growth of suboptimal pulls.
\begin{lemma}[Finite-Time Lower Bound]
\label{lem:lem_2}
For any finite horizon $n$, the expected cumulative regret for the \ac{SR} objective is bounded by
\begin{align}
    \mathbb{E}\left[\mathcal{R}_n(\pi)\right] \geq \sum_{i=2}^K \mathbb{P}\left( s_{i,n} \geq \frac{G_1 \log n}{I(f_i,f_1)}\right) \frac{G_1 \log n}{I(f_i,f_1)} \Delta_i \left( \frac{L_0 + \rho \sigma_i^2}{L_0 + \rho \sigma_{\max}^2} \right) - Y_1.
\end{align}
where $G_1$ is a universal constant, $\Delta_i = \xi_1 - \xi_i$ is the exact \ac{SR} sub-optimality gap, $\sigma_{\max}^2 = \max_{i \in [K]} \sigma_i^2$ is the maximum environmental arm variance, and $Y_1 > 0$ is a finite, deterministic penalty constant strictly independent of the horizon $n$ and the policy $\pi$.
\end{lemma}
\begin{proof}
The proof is provided in Appendix~\ref{sub:le4}.
\end{proof}
\begin{remark}
The finite constant $Y_1$ rigorously encapsulates the initial exploration penalties, bounded covariances, and reciprocal variance inflation derived via the distribution-free reciprocal bounds of Wooff~\cite{wooff1985bounds}. Because $Y_1$ is independent of $n$, it serves as a static penalty that is rapidly dominated by the logarithmic growth of the regret.    
\end{remark} 

To utilize Lemma {\ref{lem:lem_2}, we must establish that the probability of pulling a suboptimal arm at least $\mathcal{O}(\log n)$ times approaches $1$ asymptotically, and is strictly bounded below by a constant in finite time.
\begin{lemma}
\label{lem:lem_3}
For any $\alpha$-consistent policy $\pi$ (Assumption~\ref{ass:continuity}), and for any constant $G_1 < 1-\alpha$, the probability of playing a suboptimal arm $i \neq 1$ is
\begin{align*}
    \lim_{n \rightarrow \infty} \mathbb{P}_\mathcal{F}\left[s_{i,n} \geq  \frac{G_1 \log n}{I(f_i,f_{1})}\right]=1.
\end{align*}
Furthermore, with Assumption~\ref{ass:stability}, there exists $n_o \in \mathbb{N}$ such that
\begin{align*}
    \mathbb{P}_\mathcal{F}\left[s_i(n) \geq  \frac{G_1 \log n}{I(f_i,f_{*})}\right]\geq G_2, \quad \textit{for all}\quad  n >n_0
\end{align*}
where constant $0 < G_2 < 1$ is independent of $n$ and $\mathcal{F}$.
\end{lemma}
\begin{proof}
The proof is provided in Appendix~\ref{sub:le5}.
\end{proof}
We now synthesize these lemmas to establish the fundamental lower bound on the expected suboptimal pulls and the corresponding expected regret. As established by the fundamental information-theoretic limits of \ac{MAB}~\cite{garivier2019explore}, any consistent policy must explore suboptimal arms a minimum number of times. We adapt this principle to our \ac{SR} objective in the following theorem.
\begin{theorem}
\label{th:lower_bound}
Consider the \ac{MAB} problem where the objective is to optimize the \ac{SR} of the observed rewards. Let $\pi$ be an $\alpha$-consistent policy and $\Theta = \{\theta_1, \dots, \theta_K\}$ be the distribution model. Under Assumption \ref{ass:continuity} and Assumption \ref{ass:stability}, and leveraging the change-of-measure lower bounds~\cite{garivier2019explore}, for any constant $G_1 < 1-\alpha$ and $0 < G_2 < 1$, we have:
\begin{align*}
    \liminf_{n \rightarrow \infty} \frac{\mathbb{E}[s_{i,n}]}{\log n} \geq \frac{G_1}{I(f_i, f_1)}.
\end{align*}
Consequently, the model-specific regret satisfies:
\begin{align*}
    \liminf_{n \rightarrow \infty} \frac{\mathbb{E}[\mathcal{R}_n(\pi)]}{\log n} \geq \sum_{i=2}^K \frac{G_1 G_2}{I(f_i, f_1)} \Delta_i \left( \frac{L_0 + \rho \sigma_i^2}{L_0 + \rho \sigma_{\max}^2} \right).
\end{align*}
\end{theorem}

\subsection{Proof Sketch of Theorem~\ref{th:lower_bound}}
Complete proof is present in Appendix~\ref{pro:th4}. The proof establishes the asymptotic lower bound through two primary stages, $(1)$ linking the \ac{SR} regret to the expected pull counts, and $(2)$ bounding those pull counts using a change-of-measure argument.

We first demonstrate that the expected pseudo-regret is governed fundamentally by the expected number of suboptimal pulls, $\mathbb{E}[s_{i,n}]$. Because the algorithmic \ac{SR} is a fractional statistic, we apply a covariance decomposition to separate the empirical mean from the regularized precision. By bounding the covariance penalty using the asymptotic variance limits established in Theorem~\ref{th:th_dec}, we show that the expected regret is strictly lower-bounded by $\sum \mathbb{E}[s_{i,n}]\big(\Delta_i - \frac{\rho \xi_i \sigma_i^2}{L_0}\big) - \mathcal{O}(1)$. Consequently, bounding the regret reduces entirely to finding an asymptotic lower bound for $\mathbb{E}[s_{i,n}]$.

To lower bound $\mathbb{E}[s_{i,n}]$, we fix the true environment $\mathcal{F}$ where arm $1$ is optimal. For any suboptimal arm $i$, we construct a perturbed alternative environment $\mathcal{F}^i$ where the distribution of arm $i$ is shifted such that it becomes strictly optimal. By Assumption 1 (Information Continuity), we can construct this shift such that the \ac{KL} divergence $I(f_i, \tilde{f}_i)$ remains arbitrarily close to $I(f_i, f_1)$. We define the log-likelihood ratio of the observations from arm $i$ under the two environments as $\gamma$. To bound the probability that the policy under-samples arm $i$ (i.e., $s_{i,n} < \frac{G_1 \log n}{I(f_i, f_1)}$), we partition the probability space on the event $\{\gamma \gtrless G_3 \log n\}$ for a carefully chosen constant $G_3 < 1 - \alpha$:
\begin{enumerate}
    \item \textbf{When the likelihood ratio is large ($\gamma > G_3 \log n$): } The empirical log-likelihood concentrates around its true expectation. By Assumption 2 (Distributional Stability), the log-likelihood exhibits sub-Gaussian tails, and applying the Chernoff bound ensures the probability of this large deviation event decays exponentially fast.
    \item \textbf{When the likelihood ratio is small ($\gamma \leq G_3 \log n$): } The observation paths under $\mathcal{F}$ and $\mathcal{F}^i$ are statistically difficult to distinguish. We apply the Radon-Nikodym theorem to change the probability measure from $\mathcal{F}$ to the perturbed environment $\mathcal{F}^i$. Under $\mathcal{F}^i$, arm $i$ is optimal. By the definition of $\alpha$-consistency, the policy must heavily sample arm $i$; failure to do so occurs with a probability bounded by $\mathcal{O}(n^{\alpha + G_3 - 1})$. Because we selected $G_3 < 1 - \alpha$, this probability vanishes as $n \to \infty$.
\end{enumerate}
Combining these two cases proves that the probability of pulling arm $i$ fewer than $\mathcal{O}(\log n)$ times goes to zero. Applying Markov's inequality to this high-probability event establishes that $\liminf \mathbb{E}[s_{i,n}] / \log n \geq G_1 / I(f_i, f_1)$, which, when substituted into the result of Step 1, completes the proof.

\begin{remark}[Order Optimality of \texttt{SRTS}]
The combination of Lemma \ref{lem:lem_2}, Lemma \ref{lem:lem_3}, and Theorem~\ref{th:lower_bound} formally implies that for any $\alpha$-consistent policy, the cumulative regret under the \ac{SR} objective must grow \textit{at least} logarithmically in $n$. The arm-wise constants are dictated strictly by the \ac{SR}-adjusted gaps and the \ac{KL} information numbers. This perfectly matches the $\mathcal{O}(\log n)$ upper bound we established for \texttt{SRTS} in Theorem~\ref{th:regret_bound}, verifying that our algorithm is fundamentally order-optimal in the model-specific regime.
\end{remark}

\section{On the Gap Between the Upper and Lower Bounds}
\label{sec:tightness}
The finite-time upper and lower bounds derived in the previous sections do not match exactly in their problem-dependent constants. This discrepancy is expected. The upper bound (Theorem~\ref{th:regret_bound}) is obtained by controlling the empirical mean and variance terms separately using concentration inequalities and union bounds, whereas the lower bound (Theorem~\ref{th:lower_bound}) is expressed through the joint \ac{KL} divergence between reward distributions. Consequently, the constants appearing in the two bounds need not coincide at finite horizons.

More precisely, the regret upper bound is obtained by controlling the Gaussian mean error and the Gamma precision error separately, leading to surrogate divergence terms involving the squared mean gap and the function $h(x) = \frac{x - 1 - \log x}{2}.$ In contrast, the lower bound depends on the exact \ac{KL} divergence between Gaussian distributions with different means and variances. As a result, the constants appearing in the two bounds are not identical at finite horizons.

\subsection{Local Gaussian Regime}
To better understand the relationship between these quantities, consider a two-arm environment $K=2$ in which the \ac{SR} gap becomes small. Let $\Lambda_{1,2} = \mu_1 - \mu_2$ and $\Xi_{1,2} = \sigma_1^2 - \sigma_2^2$ denote the mean and variance differences between the arms. For Gaussian rewards, the \ac{KL} divergence governing the lower bound admits the second-order expansion
\begin{align*}
    I(f_2,f_1) = \frac{\Lambda_{1,2}^2}{2\sigma_1^2} + \frac{\Xi_{1,2}^2}{4\sigma_1^4} + \mathcal{O}\!\left(\Lambda_{1,2}^2+\Xi_{1,2}^2\right).
\end{align*}
In the upper bound analysis, the variance component is controlled through $h(x)=\frac{x-1-\log x}{2}$. When the variances of the two arms are close, the Taylor expansion gives
\begin{align*}
    h\!\left(\frac{\sigma_2^2}{\sigma_1^2}\right) = \frac{\Xi_{1,2}^2}{4\sigma_1^4} + \mathcal{O}\!\left(\Xi_{1,2}^2\right).
\end{align*}
Thus, in the regime where the arms are nearly indistinguishable, the surrogate variance term used in the upper bound has the same second-order behavior as the variance component of the Gaussian \ac{KL} divergence. This observation suggests that the difference between the two bounds is primarily due to the decoupling and union-bound arguments used to control the fractional \ac{SR} statistic, rather than a fundamental inefficiency of the exploration strategy.

The comparison above highlights a structural distinction between the two analyses. The lower bound couples the mean and variance parameters through the joint \ac{KL} divergence, whereas the upper bound controls their deviations separately and combines the resulting bounds through a maximum operator. This decoupling is necessary to handle the non-linear fractional structure of the \ac{SR}, but it introduces additional constants into the regret bound. Consequently, the finite-time upper bound should be viewed as a conservative guarantee. The local expansion presented above indicates that, when the reward distributions of competing arms become close, the divergence terms appearing in the upper bound behave consistently with the \ac{KL} divergence governing the lower bound. This provides additional evidence that the regret scaling derived for \texttt{SRTS} captures the correct statistical difficulty of the \ac{SR} bandit problem.

\subsection{Cross-Arm Variance Effects Under Equal \acp{SR}}
An important structural feature of the \ac{SR} objective is that it is a many-to-one mapping: different combinations of means and variances may yield the same \ac{SR}. Consequently, two arms can satisfy $\xi_1=\xi_2$ while still having distinct reward distributions (e.g., $\mu_1\neq\mu_2$ or $\sigma_1^2\neq\sigma_2^2$).

When the parameters differ, the corresponding reward distributions remain statistically distinguishable, and the \ac{KL} divergence between them is strictly positive. In classical bandits, sampling among such arms does not incur regret because the objective depends only on the expected reward of each arm individually. For the \ac{SR} objective, the situation differs. The algorithmic \ac{SR} is computed from the aggregate reward sequence generated by the policy. When the policy alternates between arms with different means, the resulting mixture introduces an additional variance component. As shown in Theorem~\ref{th:regret_bound}, this effect appears through the cross-arm term $\frac{1}{2n^2}\sum_i\sum_j s_i s_j (\mu_i-\mu_j)^2$. This term reflects the increase in global empirical variance resulting from mixing observations from arms with different reward levels. Since the \ac{SR} depends inversely on variance, this effect reduces the empirical objective and contributes to regret.

This mechanism helps explain the gap between the upper and lower bounds. The lower bound characterizes the difficulty of distinguishing individual reward distributions through \ac{KL} divergence, whereas the upper bound additionally captures the variance contribution induced by mixing rewards from arms with different means. As a result, when arms share the same \ac{SR} but have different parameters, the constants in the two bounds need not coincide.

\section{Optimal Partitioning of the Error Budget}
\label{app:error_budget}
In the proof of the upper bound (Appendix~\ref{pro:th2}), we decoupled the concentration event of the \ac{SR} by partitioning the total error margin $\varepsilon$ into a mean-dependent margin $\varepsilon_\mu$ and a variance-dependent margin $\varepsilon_\sigma$, such that $\varepsilon_\mu + \varepsilon_\sigma = \varepsilon$. 

A natural choice would be the symmetric allocation $\varepsilon_\mu = \varepsilon_\sigma = \varepsilon/2$. However, such a choice does not reflect the asymmetric sensitivity of the \ac{SR} to perturbations in the mean and variance parameters. To formalize this, consider the algebraic decomposition of the sampled \ac{SR} gap derived in Lemma \ref{lem:B1}:
\begin{align}
    \hat{\xi}_{i,t} - \xi_1 = \underbrace{\frac{\theta_{i,t} - \mu_1}{L_0 + \frac{\rho}{\tau_{i,t}}}}_{\text{Mean Error Component}} + \underbrace{\frac{\rho \mu_1 \left(\sigma_1^2 - \frac{1}{\tau_{i,t}}\right)}{\left(L_0 + \frac{\rho}{\tau_{i,t}}\right)\left(L_0 + \rho\sigma_1^2\right)}}_{\text{Variance Error Component}}.
\end{align}
When the sampled parameters are close to their true values, the dominant scaling factors of these two terms are approximately $\frac{1}{L_0 + \rho\sigma_i^2}$ and $\frac{\mu_1}{L_0 + \rho\sigma_1^2}$, respectively. These coefficients quantify the first-order sensitivity of the \ac{SR} to deviations in the mean and variance estimates.

If the error budget were divided equally, the component with the larger sensitivity coefficient would dominate the concentration requirement. In the regret analysis, this would increase the number of samples needed to ensure the corresponding deviation remains below its allocated margin, thereby enlarging the leading logarithmic constant in the pull-count bound (Lemma~\ref{lem:B2}). To account for this asymmetry, the error margin is partitioned proportionally to the sensitivities of the two components.  We define the sensitivity weights as
\begin{align}
    w_\mu = \frac{1}{L_0 + \rho \sigma_i^2}, \quad w_\sigma = \frac{\mu_1}{L_0 + \rho \sigma_1^2}.
\end{align}
Normalizing these weights yields the optimal fractional allocation
\begin{align}
    \varepsilon_\mu = \left( \frac{w_\mu}{w_\mu + w_\sigma} \right) \varepsilon, \quad \varepsilon_\sigma = \left( \frac{w_\sigma}{w_\mu + w_\sigma} \right) \varepsilon.
\end{align}
which strictly satisfies the partition constraint $\varepsilon_\mu + \varepsilon_\sigma = \varepsilon$. 

This partition balances the concentration requirements associated with the Gaussian mean sample and the Gamma precision sample. In particular, it avoids a situation in which a single component dominates error control and determines the overall sampling complexity. The allocation also adapts naturally across different risk regimes. For example, when $\rho \to 0$, the \ac{SR} objective reduces to a scaled expected reward. In this regime, the variance sensitivity vanishes ($w_\sigma \to 0$), and the entire error budget is assigned to the mean component ($\varepsilon_\mu \to \varepsilon$). This recovers the concentration structure used in classical \ac{TS} analyses.

\begin{figure*}[t]
\centering
\subfloat[]
{\includegraphics[width = 0.315\textwidth]{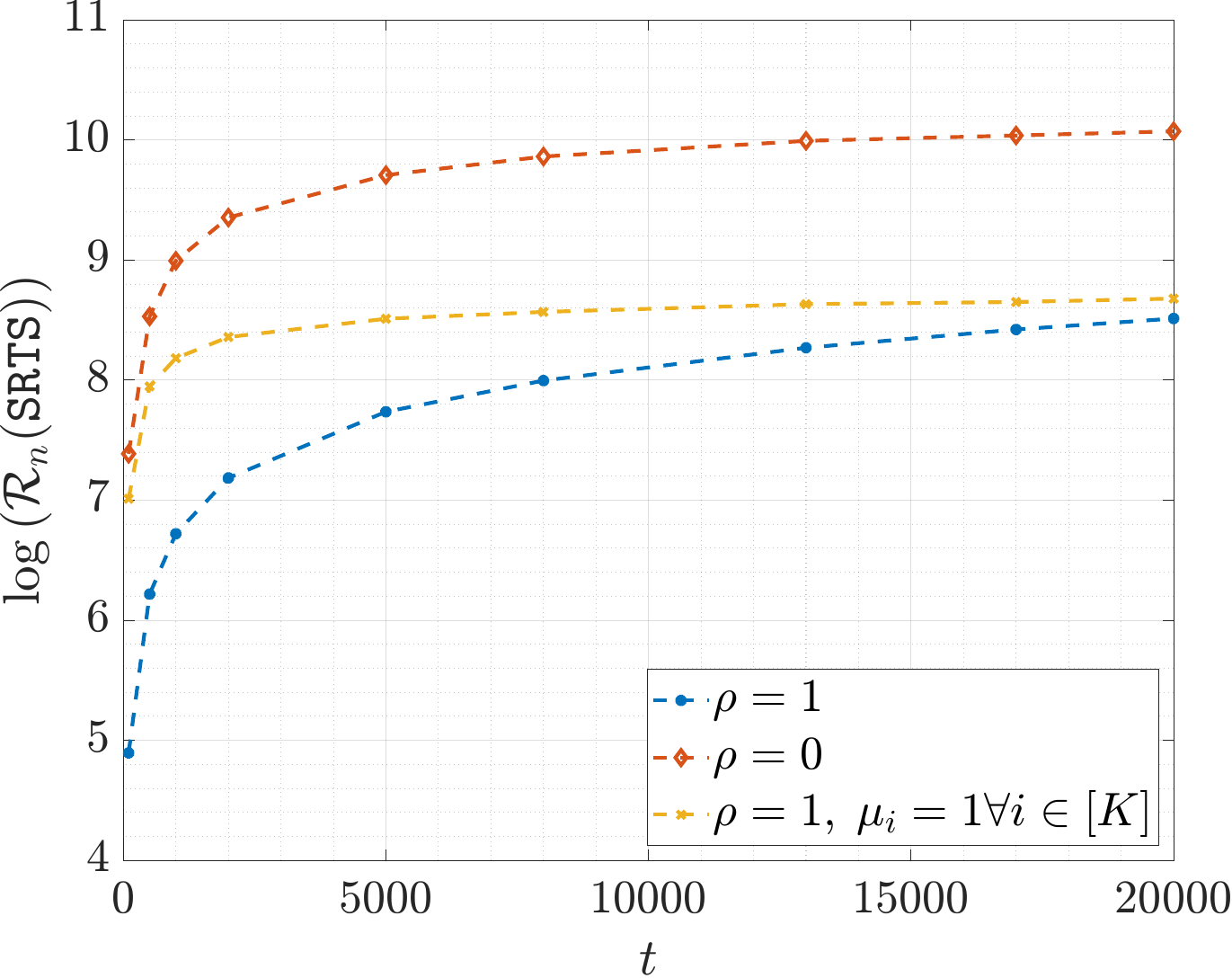}
\label{fig:result_0}}
\hfil
\subfloat[]
{\includegraphics[width = 0.315\textwidth]{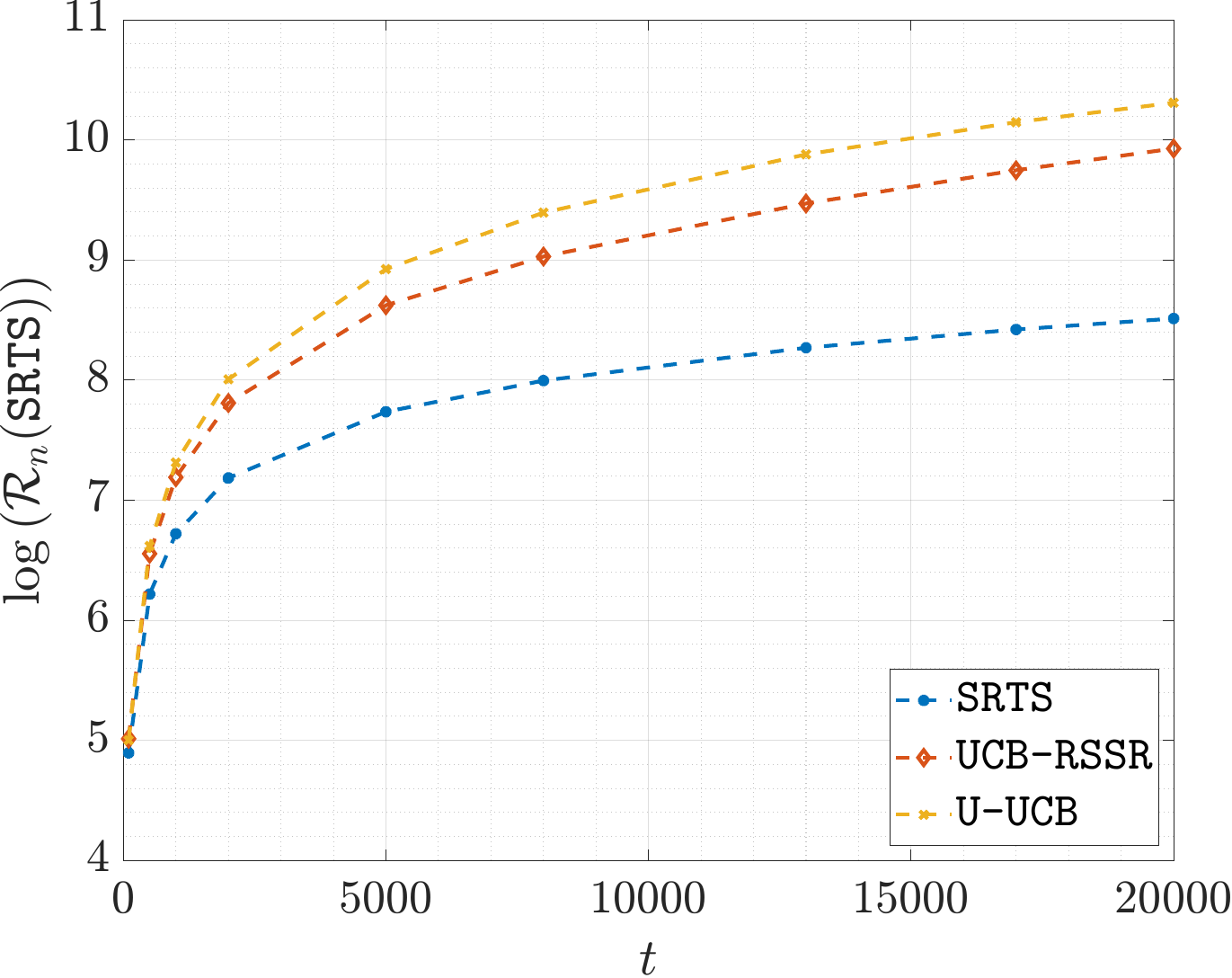}
\label{fig:result_4}}
\hfil
\subfloat[]
{\includegraphics[width = 0.3\textwidth]{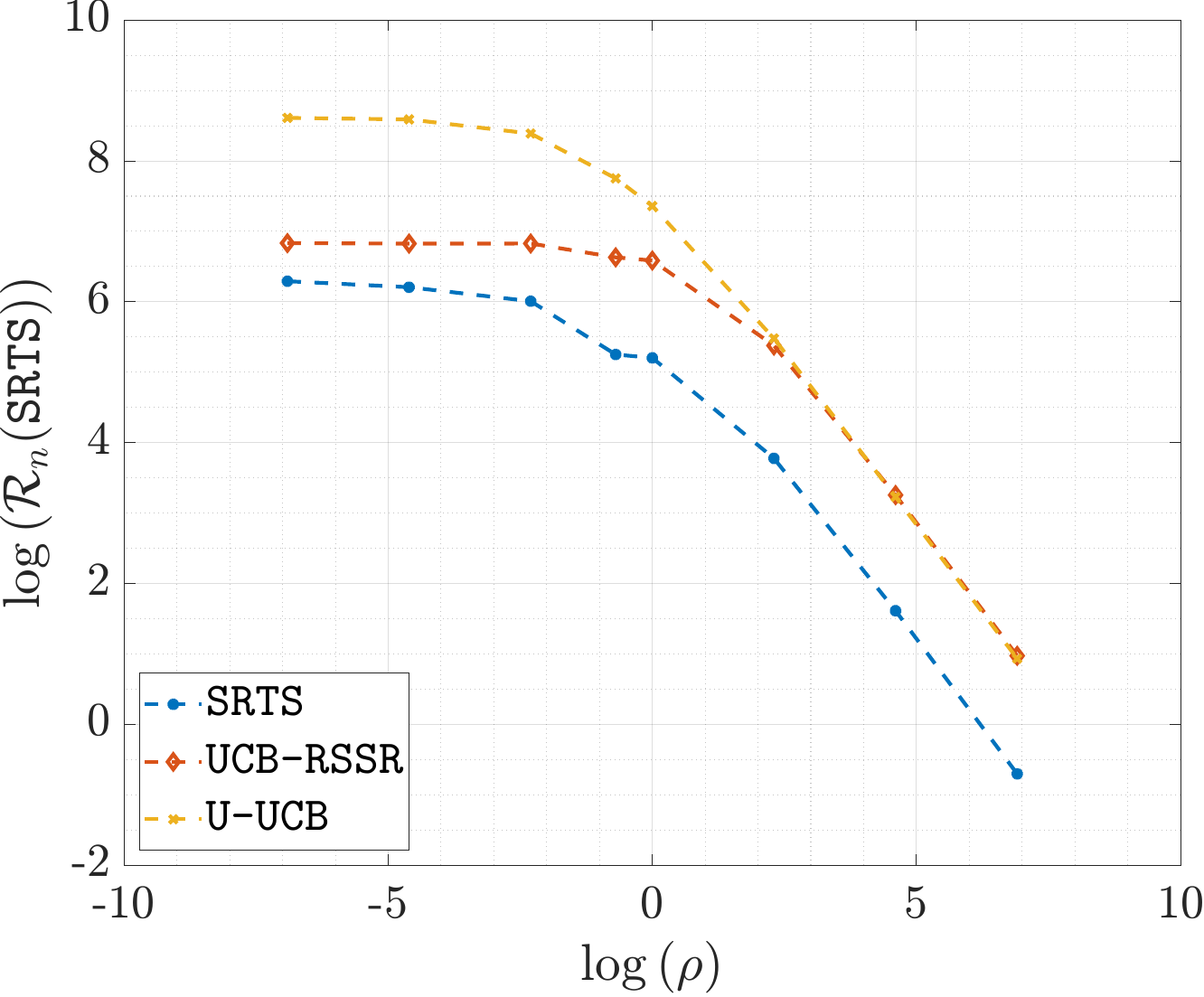}
\label{fig:result_5}}
\caption{(a) Here we have three results: (1) Regret when $\rho = 1$, (2) Regret when $\rho = 0$, and (3) Regret when $\rho = 1$ and $\mu_i = 1$, $i \in \{1,2,\dots,K\}$. (b) Performance of \texttt{SRTS} w.r.t \texttt{UCB-RSSR} and \texttt{U-UCB} for $\rho = 1$. (c) Regret v/s $\rho$ for Gaussian \texttt{SRTS} w.r.t \texttt{UCB-RSSR} and \texttt{U-UCB}.}
\label{fig:r2}
\end{figure*}

\section{Numerical Results}
\label{sec:numerical_results}
This section presents numerical experiments designed to illustrate the empirical performance of the proposed \texttt{SRTS} algorithm and compare it with existing approaches for risk-aware bandit optimization.

\subsection{Experimental Setup and Baselines}
We consider a Gaussian bandit environment with $K=10$ arms. The true mean and variance parameters are 
\begin{align*}
    \mu &= (0.10, 0.27, 0.34, 0.41, 0.43, 0.55, 0.56, 0.67, 0.71, 0.79) \,\text{and} \\
    \sigma^2 &= (0.05, 0.09, 0.19, 0.14, 0.44, 0.24, 0.36, 0.56, 0.49, 0.85),
\end{align*}
which are the same as the experiment from~\cite{sani2012risk}. Unless otherwise specified, the time horizon is $n=20{,}000$. Reported regret values are averaged over $500$ independent Monte Carlo runs. The empirical variability across runs is small relative to the mean and is therefore omitted from the plots for clarity.

We compare \texttt{SRTS} with two algorithms designed for fractional or risk-aware bandit objectives:
\begin{enumerate}
    \item \textbf{\texttt{U-UCB}} \cite{cassel2023general}: an optimistic algorithm for Empirical Distribution Performance Measures (EDPMs) that incorporates a variance regularization term.
    \item \textbf{\texttt{UCB-RSSR}} \cite{khurshid2025optimizing}: a frequentist algorithm for optimizing a regularized squared \ac{SR} using path-dependent concentration bounds derived from McDiarmid’s and Hoeffding’s inequalities.
\end{enumerate}

\subsection{Behavior Across Risk Regimes}
We first examine the behavior of \texttt{SRTS} under different values of the risk parameter $\rho$. Figure~\ref{fig:result_0} reports the cumulative regret on a logarithmic scale for three representative regimes.
\begin{itemize}
    \item \textbf{Balanced risk-return ($\rho = 1$): } In this regime both mean and variance influence the \ac{SR} objective. The regret grows approximately logarithmically with time, consistent with the theoretical guarantees.
    \item \textbf{Reward-maximization regime ($\rho = 0$): } Setting $\rho = 0$ removes the variance penalty and reduces the objective to expected reward maximization. In this case, \texttt{SRTS} behaves similarly to standard \ac{TS} for stochastic bandits.
    \item \textbf{Variance-dominated regime: } To isolate the role of variance, we consider an environment where all arms share the same mean $\mu_i = 1$, while retaining their original variances. In this setting, the objective favors arms with smaller variance, and the algorithm concentrates on the lowest-variance arm over time.
\end{itemize}

\subsection{Comparison With Existing Algorithms}
We next compare \texttt{SRTS} with the baseline algorithms. Figure~\ref{fig:result_4} shows cumulative regret (log scale) as a function of time in the balanced setting $\rho=1$. In these experiments, \texttt{SRTS} achieves lower regret than both \texttt{UCB-RSSR} and \texttt{U-UCB}. One possible explanation is that the Bayesian posterior sampling mechanism adapts to the joint uncertainty in the mean and variance parameters, while frequentist algorithms rely on conservative concentration bounds for the fractional objective.

Finally, Figure~\ref{fig:result_5} examines performance as the risk parameter varies over $\rho \in [10^{-3},10^{3}]$. Across all algorithms, the cumulative regret decreases as $\rho$ increases, because a higher $\rho$ imposes a heavier penalty on variance, thereby widening the sub-optimality gap between the arms. Across this range, \texttt{SRTS} consistently achieves lower regret than the baseline methods, indicating stable performance across a wide spectrum of risk sensitivities.

\section{Conclusion}
\label{sec:conclusion}
This paper studies \ac{SR} maximization in stochastic multi-armed bandits. Unlike classical reward maximization, the \ac{SR} objective is a fractional statistic that depends on both the mean and variance of the reward distribution. This coupling introduces additional analytical challenges compared with standard bandit objectives. We proposed the \texttt{SRTS} algorithm, a \ac{TS} policy based on the Normal-Gamma conjugate posterior for Gaussian rewards. To analyze its performance, we developed a regret decomposition tailored to the fractional \ac{SR} objective and introduced a decoupling framework that separates the contributions of the mean and variance estimation errors. This analysis yields a distribution-dependent $\mathcal{O}(\log n)$ upper bound on the expected regret. A complementary change-of-measure argument establishes a corresponding lower bound, showing that the algorithm achieves the correct logarithmic order.

The proposed framework applies across a broad range of risk sensitivities governed by the parameter $\rho$. In particular, the formulation recovers classical \ac{TS} behavior as $\rho \to 0$, while maintaining stability in regimes where variance plays a dominant role. Numerical experiments illustrate the algorithm's behavior across different risk regimes and demonstrate competitive performance relative to existing methods for risk-aware bandit optimization. Overall, the results provide a theoretical and algorithmic framework for sequential decision-making under \ac{SR} objectives and illustrate how Bayesian posterior sampling can be adapted to fractional, risk-sensitive bandit problems.

\bibliography{references.bib}
\bibliographystyle{ieeetr}

\newpage

\appendices
\section{Proof of Lemma~\ref{lem:efron_s}}
\label{pro:le1}
\begin{proof}
Let $Z_1, Z_2, \dots, Z_n$ be independent random variables taking values in some space $\mathcal{Z}$, and let $f\left(Z_1, \dots, Z_n\right)$ be a square-integrable function. Let $Z_1^\prime, \dots, Z_n^\prime$ be independent copies of $Z_1, \dots, Z_n$, and define: 
\begin{align*}
    f^{(i)} = f\left(Z_1, \dots, Z_{i-1}, Z_i^\prime, Z_{i+1}, \dots, Z_n\right),
\end{align*}
then the Efron-Stein inequality states: $\mathbb{V}\left[f\left(Z_1, \dots, Z_n\right)\right] \leq \frac{1}{2} \sum_{i=1}^n \mathbb{E}\left[\left(f - f^{(i)}\right)^2\right]$. 

For $s_{i,n} = f\left(X_1, \dots, X_n\right) = \sum_{t=1}^n \mathbb{I} \left(X_t = i\right)$, where $X_t = i$ is the arm $i$ pulled at time $t$, we apply Efron-Stein inequality on $\mathbb{V}\left[s_{i,n}\right]$. If we change $X_t$ to $X_t^\prime$ (leaving other $X_j$ fixed), the maximum change in $s_{i,n}$ is $\left|f(X) - f(X^{(t)})\right| \leq 1$, because at most one count $\mathbb{I}(X_t = i)$ can flip (from $0$ to $1$ or vice versa). Substituting into Efron-Stein inequality, we get,
\begin{align*}
    \mathbb{V}[s_{i,n}] \leq \frac{1}{2} \sum_{t=1}^n \mathbb{E}\left[\left(s_{i,n} - s_{i,n}^{(t)} \right)^2\right] \leq \frac{1}{2} \sum_{t=1}^n 1 = \frac{n}{2}.
\end{align*}
Thus, the variance of the arm pull count $s_{i,n}$ is at most linear in $n$.
\end{proof}

\section{Proof of Theorem~\ref{th:th_dec}}
\label{pro:th1}
\begin{proof}
The proof of Theorem~\ref{th:th_dec} relies on decoupling the expectation of the algorithmic \ac{SR}. Because the \ac{SR} is a fractional statistic, the empirical mean and empirical variance are dependent. We proceed in three steps: first, we establish a deterministic lower bound on the denominator via Jensen's inequality, which introduces a covariance penalty. Second, we bound the variances of the numerator and the reciprocal of the denominator separately (Lemmas \ref{le:le_2} and \ref{le:le_3}). Finally, we control the covariance penalty using the Cauchy-Schwarz inequality.

We begin by expanding the expected regret:
\begin{align}
    &\mathbb{E} \left[\mathcal{R}_n(\pi)\right] = n\mathbb{E}\left[\xi_1 - \tbar{\xi}_n(\pi)\right] \nonumber\\
    &= n\mathbb{E} \left[\frac{1}{n} \sum_{i=1}^K s_{i,n} \frac{\mu_1}{L_0 + \rho\sigma^2_1} - \frac{ \frac{1}{n} \sum_{i=1}^K s_{i,n} \tbar{\mu}_{i,s_{i,n}}}{L_0 + \rho\left(\frac{1}{n} \sum_{i=1}^K s_{i,n} \tbar{\sigma}_{i,s_{i,n}}^2 + \frac{1}{2n^2} \sum_{i=1}^K \sum_{j\neq i} s_{i,n} s_{j,n} \left(\tbar{\mu}_{i,s_{i,n}} - \tbar{\mu}_{j,s_{j,n}} \right)^2\right)} \right] \nonumber\\
    &= n \left(\frac{1}{n} \sum_{i=1}^K \mathbb{E}[s_{i,n}] \frac{\mu_1}{L_0 + \rho\sigma^2_1} - \mathbb{E} \left[\frac{ \frac{1}{n} \sum_{i=1}^K s_{i,n} \tbar{\mu}_{i,s_{i,n}}}{L_0 + \rho\left(\frac{1}{n} \sum_{i=1}^K s_{i,n} \tbar{\sigma}_{i,s_{i,n}}^2 + \frac{1}{2n^2} \sum_{i=1}^K \sum_{j\neq i} s_{i,n} s_{j,n} \left(\tbar{\mu}_{i,s_{i,n}} - \tbar{\mu}_{j,s_{j,n}} \right)^2\right)} \right]\right) \nonumber \\
    &= n \Vast(\frac{1}{n} \sum_{i=1}^K \mathbb{E}[s_{i,n}] \frac{\mu_1}{L_0 + \rho\sigma^2_1} - \vast(\mathbb{E} \left[\frac{1}{n} \sum_{i=1}^K s_{i,n} \tbar{\mu}_{i,s_{i,n}}\right] \times \nonumber\\ 
    &\hspace*{2cm} \mathbb{E} \left[\frac{1}{L_0 + \rho\left(\frac{1}{n} \sum_{i=1}^K s_{i,n} \tbar{\sigma}_{i,s_{i,n}}^2 + \frac{1}{2n^2} \sum_{i=1}^K \sum_{j\neq i} s_{i,n} s_{j,n} \left(\tbar{\mu}_{i,s_{i,n}} - \tbar{\mu}_{j,s_{j,n}} \right)^2\right)}\right] + \nonumber\\
    &\hspace*{1.5cm} \text{Cov} \left(\frac{1}{n} \sum_{i=1}^K s_{i,n} \tbar{\mu}_{i,s_{i,n}}, \frac{1}{L_0 + \rho\left(\frac{1}{n} \sum_{i=1}^K s_{i,n} \tbar{\sigma}_{i,s_{i,n}}^2 + \frac{1}{2n^2} \sum_{i=1}^K \sum_{j\neq i} s_{i,n} s_{j,n} \left(\tbar{\mu}_{i,s_{i,n}} - \tbar{\mu}_{j,s_{j,n}} \right)^2\right)}\right) \vast) \Vast) \nonumber \\
    &\overset{(a)}{\leq} n \Vast(\frac{1}{n} \sum_{i=1}^K \mathbb{E}[s_{i,n}] \frac{\mu_1}{L_0 + \rho\sigma^2_1} - \frac{\frac{1}{n} \sum_{i=1}^K \mu_i \mathbb{E}[s_{i,n}]}{\mathbb{E} \left[L_0 + \rho\left(\frac{1}{n} \sum_{i=1}^K s_{i,n} \tbar{\sigma}_{i,s_{i,n}}^2 + \frac{1}{2n^2} \sum_{i=1}^K \sum_{j\neq i} s_{i,n} s_{j,n} \left(\tbar{\mu}_{i,s_{i,n}} - \tbar{\mu}_{j,s_{j,n}} \right)^2\right)\right]} - \nonumber\\
    &\hspace*{1.5cm} \sqrt{\mathbb{V}\left[\frac{1}{n} \sum_{i=1}^K s_{i,n} \tbar{\mu}_{i,s_{i,n}}\right] \times \mathbb{V} \left[\frac{1}{L_0 + \rho\left(\frac{1}{n} \sum_{i=1}^K s_{i,n} \tbar{\sigma}_{i,s_{i,n}}^2 + \frac{1}{2n^2} \sum_{i=1}^K \sum_{j\neq i} s_{i,n} s_{j,n} \left(\tbar{\mu}_{i,s_{i,n}} - \tbar{\mu}_{j,s_{j,n}} \right)^2\right)}\right]}  \Vast).
    \label{eq:eq_10}
\end{align}
Step (a) follows from Jensen inequality as $L_0 + \rho\left(\frac{1}{n} \sum_{i=1}^K s_{i,n} \tbar{\sigma}_{i,s_{i,n}}^2 + \frac{1}{2n^2} \sum_{i=1}^K \sum_{j\neq i} s_{i,n} s_{j,n} \left(\tbar{\mu}_{i,s_{i,n}} - \tbar{\mu}_{j,s_{j,n}} \right)^2\right)$ is strictly positive. To systematically evaluate the bound established in~\eqref{eq:eq_10}, the remainder of the proof reduces to analyzing three distinct statistical quantities. Let $D$ denote the regularized algorithmic risk, defined as the denominator of the algorithmic empirical \ac{SR}:
\begin{align*}
    D = L_0 + \rho\left(\frac{1}{n} \sum_{i=1}^K s_{i,n} \tbar{\sigma}_{i,s_{i,n}}^2 + \frac{1}{2n^2} \sum_{i=1}^K \sum_{j\neq i} s_{i,n} s_{j,n} \left(\tbar{\mu}_{i,s_{i,n}} - \tbar{\mu}_{j,s_{j,n}} \right)^2\right).
\end{align*}
The subsequent analysis is then structured around the following auxiliary lemmas:
\begin{itemize}
    \item[\textbf{1.}] \textbf{The expectation of the regularized algorithmic risk (Lemma~\ref{le:le_1}): }
    \begin{align*}
        \mathbb{E} \left[L_0 + \rho\left(\frac{1}{n} \sum_{i=1}^K s_{i,n} \tbar{\sigma}_{i,s_{i,n}}^2 + \frac{1}{2n^2} \sum_{i=1}^K \sum_{j\neq i} s_{i,n} s_{j,n} \left(\tbar{\mu}_{i,s_{i,n}} - \tbar{\mu}_{j,s_{j,n}} \right)^2\right)\right] \leq L_0 +  \frac{\rho}{2} \Lambda_{\max}^2 + \rho \sum_{i=1}^K \sigma_i^2.
    \end{align*}
    where $\Lambda_{\max} = \max \Lambda_{i,j}^2\, \forall i, j \in \{1, 2, \dots, K\}$. The complete proof of Lemma~\ref{le:le_1} is present in Appendix~\ref{sub:le_1}.
    \item[\textbf{2.}] \textbf{The variance of the algorithmic empirical mean (Lemma~\ref{le:le_2}): }
    \begin{align*}
        \mathbb{V}\left[\frac{1}{n} \sum_{i=1}^K s_{i,n} \tbar{\mu}_{i,s_{i,n}}\right] \leq \frac{1}{n} \sum_{i=1}^K \sigma_i^2 + \frac{1}{2n} \sum_{i=1}^K \mu_i^2 \sim \mathcal{O} \left(\frac{Q_1}{n}\right).
    \end{align*}
    where $Q_1$ is a constant independent of $n$ and $s_{i,n}$. The complete proof of Lemma~\ref{le:le_2} is present in Appendix~\ref{sub:le_2}.
    \item[\textbf{3.}] \textbf{The variance of the reciprocal of the regularized algorithmic risk} \textbf{(Lemma~\ref{le:le_3}): }
    \begin{align*}
        \mathbb{V} \left[\frac{1}{L_0 + \rho\left(\frac{1}{n} \sum_{i=1}^K s_{i,n} \tbar{\sigma}_{i,s_{i,n}}^2 + \frac{1}{2n^2} \sum_{i=1}^K \sum_{j\neq i} s_{i,n} s_{j,n} \left(\tbar{\mu}_{i,s_{i,n}} - \tbar{\mu}_{j,s_{j,n}} \right)^2\right)}\right] \lesssim \mathcal{O} \left(\frac{Q_2}{n}\right).
    \end{align*}
    where $Q_2$ is a constant independent of $n$ and $s_{i,n}$. The complete proof of Lemma~\ref{le:le_3} is present in Appendix~\ref{sub:le_3}.
\end{itemize}
Substituting the bounds established in Lemmas \ref{le:le_1}, \ref{le:le_2}, and \ref{le:le_3} back into~\eqref{eq:eq_10}, the expected regret becomes bounded by,
\begin{align*}
    \mathbb{E} \left[\mathcal{R}_n(\pi)\right] \leq n \left(\frac{1}{n} \sum_{i=1}^K \mathbb{E}[s_{i,n}] \frac{\mu_1}{L_0 + \rho\sigma^2_1} - \frac{\frac{1}{n} \sum_{i=1}^K \mu_i \mathbb{E}[s_{i,n}]}{L_0 +  \frac{\rho}{2} \Lambda_{\max}^2 + \rho \sum_{i=1}^K \sigma_i^2} + \sqrt{\frac{Q_1}{n} \times \frac{Q_2}{n}}\right).
\end{align*}
Let $X_1 = \sqrt{Q_1 Q_2}$ absorb the finite asymptotic variance penalty, and let $E = \frac{\rho}{2} \Lambda_{\max}^2 + \rho \sum_{i=1}^K \sigma_i^2$, we get
\begin{align}
    \mathbb{E} \left[\mathcal{R}_n(\pi)\right] \leq \sum_{i=1}^K \mathbb{E}[s_{i,n}] \xi_1 - \frac{\sum_{i=1}^K \mu_i \mathbb{E}[s_{i,n}]}{L_0 + E} + X_1.
\end{align}
To naturally extract the algorithmic penalty, we utilize the intrinsic definition of the \ac{SR}. Substituting $\mu_i = \xi_i(L_0 + \rho\sigma_i^2)$ and $\Delta_i = \xi_1 - \xi_i$ this into the fraction yields,
\begin{align*}
    \mathbb{E} \left[\mathcal{R}_n(\pi)\right] &\leq \sum_{i=1}^K \mathbb{E}[s_{i,n}] \left( \Delta_i + \xi_i - \xi_i \frac{L_0 + \rho\sigma_i^2}{L_0 + E} \right) + X_1 \\
    &= \sum_{i=1}^K \mathbb{E}[s_{i,n}] \left[ \Delta_i + \xi_i \left( 1 - \frac{L_0 + \rho\sigma_i^2}{L_0 + E} \right) \right] + X_1 \\
    &= \sum_{i=1}^K \mathbb{E}[s_{i,n}] \left[ \Delta_i + \xi_i \left( \frac{L_0 + E - (L_0 + \rho\sigma_i^2)}{L_0 + E} \right) \right] + X_1 \\
    &\leq \sum_{i=1}^K \mathbb{E}[s_{i,n}] \left( \Delta_i + \frac{\xi_i \rho\left(\frac{\Lambda_{\max}^2}{2} + \sum_{j \neq i} \sigma_j^2 \right)}{L_0 +  \frac{\rho}{2} \Lambda_{\max}^2 + \rho \sum_{j=1}^K \sigma_j^2} \right) + X_1.
\end{align*}
\end{proof}

\section{Auxiliary Lemma used in Theorem~\ref{th:th_dec}}
\subsection{expectation of the regularized algorithmic risk}
\label{sub:le_1}
\begin{lemma}
\label{le:le_1}
Let $D$ denote the regularized algorithmic risk. Its expectation is bounded strictly from above by:
\begin{align*}
    \mathbb{E} \left[L_0 + \rho\left(\frac{1}{n} \sum_{i=1}^K s_{i,n} \tbar{\sigma}_{i,s_{i,n}}^2 + \frac{1}{2n^2} \sum_{i=1}^K \sum_{j\neq i} s_{i,n} s_{j,n} \left(\tbar{\mu}_{i,s_{i,n}} - \tbar{\mu}_{j,s_{j,n}} \right)^2\right)\right] \leq L_0 +  \frac{\rho}{2} \Lambda_{\max}^2 + \rho \sum_{i=1}^K \sigma_i^2,
\end{align*}
where $\Lambda_{\max} = \max_{i, j \in [K]} |\mu_i - \mu_j|$ is the maximum pairwise mean difference.
\end{lemma}

\begin{proof}
We begin by leveraging the linearity of expectation to separate the regularization constant from the empirical variance components. Because the empirical variance and mean estimates depend heavily on the random pull count, we decouple them by invoking the Law of Total Expectation, conditioning on the pull counts $\mathbf{s}_n = \{s_{1,n}, \dots, s_{K,n}\}$. 

Substituting the known expectations of the empirical estimators conditioned on the pull counts, we have:
\begin{align*}
    &\mathbb{E} \left[L_0 + \rho\left(\frac{1}{n} \sum_{i=1}^K s_{i,n} \tbar{\sigma}_{i,s_{i,n}}^2 + \frac{1}{2n^2} \sum_{i=1}^K \sum_{j\neq i} s_{i,n} s_{j,n} \left(\tbar{\mu}_{i,s_{i,n}} - \tbar{\mu}_{j,s_{j,n}} \right)^2\right)\right] \\
    &= L_0 + \rho\mathbb{E} \left[\frac{1}{n} \sum_{i=1}^K s_{i,n} \tbar{\sigma}_{i,s_{i,n}}^2 + \frac{1}{2n^2} \sum_{i=1}^K \sum_{j\neq i} s_{i,n} s_{j,n} \left(\tbar{\mu}_{i,s_{i,n}} - \tbar{\mu}_{j,s_{j,n}} \right)^2\right] \\
    &= L_0 + \rho \mathbb{E} \left[\mathbb{E} \left[\frac{1}{n} \sum_{i=1}^K s_{i,n} \tbar{\sigma}_{i,s_{i,n}}^2 \Bigg| \mathbf{s}_n \right]\right] + \rho \mathbb{E} \left[\mathbb{E} \left[\frac{1}{2n^2} \sum_{i=1}^K \sum_{j\neq i} s_{i,n} s_{j,n} \left(\tbar{\mu}_{i,s_{i,n}} - \tbar{\mu}_{j,s_{j,n}} \right)^2 \Bigg| \mathbf{s}_n \right]\right] \\
    &= L_0 + \rho \mathbb{E} \left[\frac{1}{n} \sum_{i=1}^K s_{i,n} \mathbb{E} \left[\tbar{\sigma}_{i,s_{i,n}}^2 \bigg| \mathbf{s}_n \right]\right] + \rho \mathbb{E} \left[\frac{1}{2n^2} \sum_{i=1}^K \sum_{j\neq i} s_{i,n} s_{j,n} \mathbb{E} \left[\left(\tbar{\mu}_{i,s_{i,n}} - \tbar{\mu}_{j,s_{j,n}} \right)^2 \bigg| \mathbf{s}_n \right]\right] \\
    &= L_0 + \rho \mathbb{E} \left[\frac{1}{n} \sum_{i=1}^K \sigma^2_i s_{i,n} \right] + \rho \mathbb{E} \left[\frac{1}{2n^2} \sum_{i=1}^K \sum_{j\neq i} s_{i,n} s_{j,n} \left(\Lambda_{i,j}^2 + \frac{\sigma_i^2}{s_{i,n}} + \frac{\sigma_j^2}{s_{j,n}}\right) \right].
\end{align*}
To obtain a global upper bound, we bound the pairwise mean differences by their maximum, $\Lambda_{i,j}^2 \leq \Lambda_{\max}^2$, and resolve the allocation sums. Utilizing the identity $\sum_{i=1}^K s_{i,n} = n$, we simplify the cross-terms:
\begin{align*}
    &\leq L_0 +  \rho \frac{1}{n} \sum_{i=1}^K \sigma^2_i \mathbb{E} \left[s_{i,n} \right] + \rho \Lambda_{\max}^2\mathbb{E}\left[\frac{1}{2n^2} \sum_{i=1}^K \sum_{j\neq i} s_{i,n} s_{j,n}\right] + \rho \frac{1}{2n} \sum_{i=1}^K \sum_{j\neq i} \mathbb{E}\left[\sigma_i^2 s_{j,n} + \sigma_j^2 s_{i,n}\right] \\
    &= L_0 +  \rho \frac{1}{n} \sum_{i=1}^K \sigma^2_i \mathbb{E} \left[s_{i,n} \right] + \rho \Lambda_{\max}^2\mathbb{E}\left[\frac{1}{2n^2} \left(\left(\sum_{i=1}^K s_{i,n}\right)^2 - \sum_{i=1}^K s_{i,n}^2 \right) \right] + \frac{\rho}{2} \left(2 \sum_{i=1}^K \sigma_i^2 - \frac{2}{n}\sum_{i=1}^K \sigma_i^2 \mathbb{E} [s_{i,n}]\right) \\
    &\leq L_0 +  \frac{\rho}{2} \Lambda_{\max}^2 + \rho \sum_{i=1}^K \sigma_i^2.
\end{align*}
This completes the proof.
\end{proof}

\subsection{Variance of the Algorithmic Empirical Mean}
\label{sub:le_2}
\begin{lemma}
\label{le:le_2}
The variance of the algorithmic empirical mean under policy $\pi$ is bounded by:
\begin{align*}
    \mathbb{V}\left[\frac{1}{n} \sum_{i=1}^K s_{i,n} \tbar{\mu}_{i,s_{i,n}}\right] \leq \frac{1}{n} \sum_{i=1}^K \sigma_i^2 + \frac{1}{2n} \sum_{i=1}^K \mu_i^2 \sim \mathcal{O} \left(\frac{Q_1}{n}\right),
\end{align*}
where $Q_1 > 0$ is a problem-dependent constant independent of $n$ and the pull count $s_{i,n}$.
\end{lemma}

\begin{proof}
To isolate the variability intrinsic to the reward distributions from the variability induced by the algorithmic sampling policy, we apply the law of total variance. We condition on the random pull count vector $\mathbf{s}_n = \{s_{1,n}, \dots, s_{K,n}\}$:
\begin{align*}
    \mathbb{V}\left[\frac{1}{n} \sum_{i=1}^K s_{i,n} \tbar{\mu}_{i,s_{i,n}}\right] &= \mathbb{E}\left[\mathbb{V}\left[\frac{1}{n} \sum_{i=1}^K s_{i,n} \tbar{\mu}_{i,s_{i,n}} \Bigg| \mathbf{s}_n \right]\right] + \mathbb{V}\left[\mathbb{E}\left[\frac{1}{n} \sum_{i=1}^K s_{i,n} \tbar{\mu}_{i,s_{i,n}} \Bigg| \mathbf{s}_n \right]\right].
\end{align*}

Because the empirical means of different arms are conditionally independent given the pull counts, the cross-arm variances in the first term evaluate to zero. Substituting the true moments for the conditionally independent arms yields:
\begin{align*}
   \mathbb{V}\left[\frac{1}{n} \sum_{i=1}^K s_{i,n} \tbar{\mu}_{i,s_{i,n}}\right] &= \mathbb{E}\left[\frac{1}{n^2} \sum_{i=1}^K s_{i,n}^2 \cdot \frac{\sigma_i^2}{s_{i,n}}\right] +  \mathbb{V}\left[\frac{1}{n} \sum_{i=1}^K s_{i,n} \mu_i\right] \\
    &= \frac{1}{n^2} \sum_{i=1}^K \sigma_i^2 \mathbb{E}[s_{i,n}] + \frac{1}{n^2} \sum_{i=1}^K \mu_i^2 \mathbb{V}[s_{i,n}] + \frac{1}{n^2} \sum_{i \neq j} \mu_i \mu_j \text{Cov}(s_{i,n}, s_{j,n}).
\end{align*}
To upper bound this expression, we observe two structural properties of the pull count. First, as the total number of pulls is deterministically constrained to $\sum_{i=1}^K s_{i,n} = n$, the pull counts $s_{i,n}$ and $s_{j,n}$ are negatively correlated, rendering the covariance terms non-positive. Consequently, dropping the covariance sum strictly bounds the expression from above. Second, we strictly bound the marginal expectations and variances of the pull count. We trivially have $\mathbb{E}[s_{i,n}] \leq n$, and from Lemma \ref{lem:efron_s}, the variance of the pull counts is bounded by $\mathbb{V}[s_{i,n}] \leq n/2$. Applying these bounds sequentially gives:
\begin{align*}
    \mathbb{V}\left[\frac{1}{n} \sum_{i=1}^K s_{i,n} \tbar{\mu}_{i,s_{i,n}}\right] &\leq \frac{1}{n^2} \sum_{i=1}^K \sigma_i^2 \mathbb{E}[s_{i,n}] + \frac{1}{n^2} \sum_{i=1}^K \mu_i^2 \mathbb{V}[s_{i,n}] \\
    &\leq \frac{1}{n^2} \sum_{i=1}^K \sigma_i^2 (n) + \frac{1}{n^2} \sum_{i=1}^K \mu_i^2 \left(\frac{n}{2}\right) \\
    &= \frac{1}{n} \sum_{i=1}^K \sigma_i^2 + \frac{1}{2n} \sum_{i=1}^K \mu_i^2.
\end{align*}
This final expression confirms that the variance of the algorithmic empirical mean decays asymptotically as $\mathcal{O}\left(\frac{Q_1}{n}\right)$, where $Q_1 = \sum_{i=1}^K \sigma_i^2 + \frac{1}{2} \sum_{i=1}^K \mu_i^2$ is a constant entirely independent of $n$ and the algorithmic policy.
\end{proof}

\subsection{Variance of the Reciprocal of the Regularized Algorithmic Risk}
\label{sub:le_3}
\begin{lemma}
\label{le:le_3}
Let $D$ denote the regularized algorithmic risk. The variance of its reciprocal is strictly bounded from above such that:
\begin{align*}
    \mathbb{V} \left[\frac{1}{D}\right] = \mathbb{V} \left[\frac{1}{L_0 + \rho\left(\frac{1}{n} \sum_{i=1}^K s_{i,n} \tbar{\sigma}_{i,s_{i,n}}^2 + \frac{1}{2n^2} \sum_{i=1}^K \sum_{j\neq i} s_{i,n} s_{j,n} \left(\tbar{\mu}_{i,s_{i,n}} - \tbar{\mu}_{j,s_{j,n}} \right)^2\right)}\right] \leq \mathcal{O} \left(\frac{Q_2}{n}\right),
\end{align*}
where $Q_2 > 0$ is a finite problem-dependent constant independent of $n$ and the pull count $\mathbf{s}_n$.
\end{lemma}

\begin{proof}
To tractably bound this highly non-linear variance term, we decompose the algorithmic risk $D$ into two constituent empirical components: the aggregate empirical variance $V$, and the cross-arm switching penalty $W$. Let:
\begin{align*}
 D = L_0 + \rho(V + W), \quad \text{where} \quad V = \frac{1}{n}\sum_{i=1}^K s_{i,n} \tbar{\sigma}_{i,s_{i,n}}^2, \quad \text{and} \quad W = \frac{1}{2n^2}\sum_{i \neq j} s_{i,n}s_{j,n}(\tbar{\mu}_{i,s_{i,n}} - \tbar{\mu}_{j,s_{j,n}})^2.
\end{align*}
The proof proceeds in three stages: establishing the Lipschitz continuity of the reciprocal mapping, bounding the variance of $V$ and $W$ independently, and synthesizing the bounds.

\textbf{Stage 1: Lipschitz Continuity of the Reciprocal Transform} \\
Consider the mapping $f(D) = 1/D$. Because the regularization parameter ensures $D \geq L_0 > 0$ almost surely, the derivative is uniformly bounded as $|f^\prime (D)| = |-1/D^2| \leq 1/L_0^2$. By the mean value theorem, for any $D, D^\prime \geq L_0$:
\begin{align*}
    |f(D) - f(D^\prime)| \leq \left(\sup_{\xi \geq L_0} |f^\prime(\xi)|\right) |D - D^\prime| \leq \frac{1}{L_0^2}|D - D^\prime|.
\end{align*} 
For any Lipschitz function $f$ and random variable $D$, the variance satisfies $\mathbb{V}[f(D)] \leq \mathbb{E}[(f(D) - f(\mathbb{E}[D]))^2]$. Applying the Lipschitz bound yields:
\begin{align}
    \mathbb{V}\left[\frac{1}{D}\right] \leq \frac{1}{L_0^4}\mathbb{V}[D] = \frac{\rho^2}{L_0^4}\mathbb{V}[V+W].
    \label{eq:eq_l_1}
\end{align}
Consequently, bounding the reciprocal variance reduces strictly to bounding $\mathbb{V}[V+W] = \mathbb{V}[V] + \mathbb{V}[W] + 2 \text{Cov}(V, W)$.

\textbf{Stage 2: Bounding the Variance of the Empirical Components} \\
\textit{Part A (Bound on $\mathbb{V}[V]$):} We apply the law of total variance conditioned on the pull count vector $\mathbf{s}_n$. Given $\mathbf{s}_n$, the empirical variances are conditionally independent across arms. Letting $\kappa_i$ denote the fourth central moment (kurtosis) of arm $i$, the conditional variance expands as:
\begin{align*}
    \mathbb{V}\left[V \big| \mathbf{s}_n\right] &= \frac{1}{n^2} \sum_{i=1}^K s_{i,n}^2 \mathbb{V}\left[\tbar{\sigma}_{i,s_{i,n}}^2 \Big| \mathbf{s}_n\right] = \frac{1}{n^2} \sum_{i=1}^K s_{i,n}^2 \left(\frac{\kappa_i}{s_{i,n}} - \sigma_i^4\left(\frac{s_{i,n}-3}{s_{i,n}(s_{i,n}-1)}\right)\right) \\
    &\leq \frac{1}{n^2} \sum_{i=1}^K s_{i,n} \kappa_i + \frac{1}{n^2} \sum_{i=1}^K \frac{2s_{i,n} \sigma_i^4}{s_{i,n}-1}.
\end{align*}
Taking the expectation over $\mathbf{s}_n$ and combining it with the variance of the conditional expectation, we bound the total variance. Utilizing Lemma\ref{lem:efron_s} and noting the non-positive covariance of the allocation constraints, we establish:
\begin{align}
    \mathbb{V}[V] &\leq \mathbb{E}\left[\mathbb{V}\left[V \big| \mathbf{s}_n\right]\right] + \mathbb{V}\left[\mathbb{E}\left[V \big| \mathbf{s}_n\right]\right] \nonumber \\
    &\leq \frac{1}{n^2} \sum_{i=1}^K \kappa_i\mathbb{E}[s_{i,n}] + \frac{1}{n^2} \sum_{i=1}^K 2\sigma_i^4 \mathbb{E}\left[\frac{s_{i,n}}{s_{i,n}-1}\right] + \frac{1}{n^2} \sum_{i=1}^K \sigma_i^4 \mathbb{V}[s_{i,n}] \nonumber \\
    &\leq \frac{1}{n} \sum_{i=1}^K \kappa_i + \frac{1}{n^2} \sum_{i=1}^K 4\sigma_i^4 + \frac{1}{2n} \sum_{i=1}^K \sigma_i^4 \sim \mathcal{O}\left(\frac{A_1}{n}\right).
    \label{eq:eq_a2}
\end{align}

\textit{Part B (Bound on $\mathbb{V}[W]$):} Bounding the cross-arm penalty $W$ requires expanding the fourth-order empirical mean differences. By the law of total variance conditioned on $\mathbf{s}_n$,
\begin{align}
    \mathbb{V}[W] = \mathbb{E}\left[\mathbb{V}\left[W \big| s_{i,n}, s_{j,n}\right] \right] + \mathbb{V} \left[\mathbb{E}\left[W | s_{i,n}, s_{j,n} \right] \right].
    \label{eq:eq_a6}
\end{align}
We first bound the conditionally independent empirical means:
\begin{align*}
    \mathbb{V}\left[W \big| \mathbf{s}_n\right] \leq \frac{1}{4n^4} \sum_{i=1}^K \sum_{j\neq i} s_{i,n}^2 s_{j,n}^2 \mathbb{V}\left[(\tbar{\mu}_{i} - \tbar{\mu}_j)^2 \bigg| \mathbf{s}_n\right] 
\end{align*}
Expanding $\mathbb{V}[(\tbar{\mu}_{i} - \tbar{\mu}_j)^2 | \mathbf{s}_n] = \mathbb{E}[(\tbar{\mu}_{i} - \tbar{\mu}_j)^4] - (\mathbb{E}[(\tbar{\mu}_{i} - \tbar{\mu}_j)^2])^2$ yields a polynomial of empirical moments. Grouping terms by their asymptotic dependence on the allocation counts yields:
\begin{itemize}
    \item \textbf{$\mathcal{O}(1)$ terms:} $(\mu_i - \mu_j)^4$
    \item \textbf{$\mathcal{O}(1/s)$ terms:} $6(\mu_i - \mu_j)^2\left(\frac{\sigma_i^2}{s_{i,n}} + \frac{\sigma_j^2}{s_{j,n}}\right)$
    \item \textbf{$\mathcal{O}(1/s^2)$ and higher order terms:} Bounded asymptotically by $\mathcal{O}\left(1/s_{i,n}^2\right) + \mathcal{O}\left(1/s_{j,n}^2\right)$.
\end{itemize}
Aggregating these orders and subtracting the squared second moment bounds the conditional variance as,
 \begin{align*}
    \mathbb{V}\left[(\tbar{\mu}_{i} - \tbar{\mu}_j)^2 \bigg| s_{i,n}, s_{j,n}\right] &\leq 4\left(\mu_i - \mu_j\right)^2\left(\frac{\sigma_i^2}{s_{i,n}} + \frac{\sigma_j^2}{s_{j,n}}\right)  + 2\left(\frac{\sigma_i^4}{s_{i,n}^2} + \frac{\sigma_j^4}{s_{j,n}^2} + \frac{\sigma_i^2\sigma_j^2}{s_{i,n}^2 s_{j,n}^2}\right) + \mathcal{O}\left(\frac{1}{s_{i,n}^2}\right) + \mathcal{O}\left(\frac{1}{s_{j,n}^2}\right).
\end{align*}
Taking the expectation over $\mathbf{s}_n$ yields the first half of the LTV decomposition:
\begin{align}
    \mathbb{E}\left[\mathbb{V}\left[W \big| \mathbf{s}_n\right] \right] &\leq \frac{1}{n^4} \sum_{i=1}^K \sum_{j\neq i} \mathbb{E}\left[s_{i,n} s_{j,n}^2 \right] (\mu_i - \mu_j)^2 \sigma_i^2 + \frac{1}{n^4} \sum_{i=1}^K \sum_{j\neq i} \mathbb{E}\left[s_{i,n}^2 s_{j,n} \right] (\mu_i - \mu_j)^2 \sigma_j^2 + \nonumber\\
    &\hspace*{1.5cm}\frac{1}{2n^4} \sum_{i=1}^K \sum_{j\neq i} \mathbb{E}\left[s_{j,n}^2 \right] \sigma_i^4 +  \frac{1}{2n^4} \sum_{i=1}^K \sum_{j\neq i} \mathbb{E}\left[s_{i,n}^2 \right] \sigma_j^4 + \frac{1}{2n^4} \sum_{i=1}^K \sum_{j\neq i} \sigma_i^2 \sigma_j^2 + \mathcal{O}\left(\frac{1}{n^2}\right).
    \label{eq:eq_t1}
\end{align}
For the variance of the conditional expectation, i.e., $\mathbb{V}[\mathbb{E}[W | \mathbf{s}_n]]$, we exploit the negative covariance of the multinomial-like pull count constraints. We show,
\begin{align*}
    &\mathbb{V}\left[\frac{1}{2n^2} \sum_{i=1}^K \sum_{j\neq i} s_{i,n} s_{j,n} \left(\left(\mu_i - \mu_j\right)^2 + \frac{\sigma_i^2}{s_{i,n}} + \frac{\sigma_j^2}{s_{j,n}}\right)\right] = \\
    &\hspace*{3cm} \frac{1}{4n^4} \mathbb{V}\left[\sum_{i=1}^K \sum_{j\neq i} s_{i,n} s_{j,n} \underbrace{\left(\mu_i - \mu_j\right)^2}_{a_{ij}} \right] + \frac{1}{4n^4} \mathbb{V}\left[\sum_{i=1}^K \sum_{j\neq i} s_{j,n} \sigma_i^2\right] + \mathbb{V}\left[\frac{1}{4n^2} \sum_{i=1}^K \sum_{j\neq i} s_{i,n} \sigma_j^2\right] \\
    &\hspace*{0cm} \leq \frac{1}{4n^4} \sum_{i=1}^K \sum_{j \neq i} \sum_{k=1}^K \sum_{l \neq k} a_{ij} a_{kl} \text{Cov}(s_{i,n} s_{j,n}, s_{k,n} s_{l,n}) + \frac{1}{4n^4} \mathbb{V}\left[\left(\sum_{i=1}^K \sigma_i^2\right) \left(\sum_{i=1}^K s_{i,n}\right)\right] + \frac{1}{4n^4} \mathbb{V}\left[\left(\sum_{i=1}^K \sigma_i^2\right) \left(\sum_{i=1}^K s_{i,n}\right)\right].
\end{align*}
Specifically, because $\text{Cov}(s_{i,n} s_{j,n}, s_{k,n} s_{l,n}) \leq 0$ for mutually distinct arm pairs, the cross-correlations are strictly non-positive, allowing us to drop them for an upper bound:
\begin{align}
    \mathbb{V}\left[\mathbb{E}\left[W \big| \mathbf{s}_n\right]\right] \leq \frac{1}{4n^4} \sum_{i=1}^K \sum_{j\neq i} \mathbb{V}\left[s_{i,n} s_{j,n}\right] (\mu_i - \mu_j)^2 + \frac{1}{2n^4} \left(\sum_{i=1}^K \sigma_i^2\right)^2 \sum_{i=1}^K \mathbb{V}\left[s_{i,n}\right]. 
    \label{eq:eq_t2}
\end{align}
Summing \eqref{eq:eq_t1} and \eqref{eq:eq_t2}, the dominant terms decay at a rate of $\mathcal{O}(1/n)$ with subsequent higher-order decay components, establishing 
\begin{align*}
    \mathbb{V}[W] \sim \mathcal{O}(A_2/n).
\end{align*}
\textbf{Stage 3: Synthesis via Cauchy-Schwarz} \\
To bound the cross-covariance, we invoke the Cauchy-Schwarz inequality:
\begin{align*}
    \left|\text{Cov}(V,W)\right| \leq \sqrt{\mathbb{V}[V]\mathbb{V}[W]} \sim \sqrt{\mathcal{O}\left(\frac{A_1}{n}\right) \mathcal{O}\left(\frac{A_2}{n}\right)} = \mathcal{O}\left(\frac{A_3}{n}\right).
\end{align*}
Finally, substituting the established bounds for $\mathbb{V}[V]$, $\mathbb{V}[W]$, and $\text{Cov}(V,W)$ back into the Lipschitz bound from Stage 1 establishes:
\begin{align*}
    \mathbb{V}\left[\frac{1}{D}\right] \leq \frac{\rho^2}{L_0^4} \left( \mathcal{O}\left(\frac{A_1}{n}\right) + \mathcal{O}\left(\frac{A_2}{n}\right) + \mathcal{O}\left(\frac{A_3}{n}\right) \right) \sim \mathcal{O}\left(\frac{Q_2}{n}\right).
\end{align*}
This completes the proof.
\end{proof}

\section{Proof of the Theorem~\ref{th:pull_bound}}
\label{pro:th2}
To bound the expected number of suboptimal pulls $\mathbb{E}[s_{i,n}]$, we utilize the generic decoupling framework for posterior sampling (Theorem 36.2 in \cite{lattimore2020bandit}). $H_{1,q} := \mathbb{P}_t \left(\mathcal{G}_1^c(t) \mid s_{1,t} = q \right)$ denotes the conditional probability that the optimal arm's sampled \ac{SR} falls below the threshold $\xi_1 - \varepsilon$. The expected pulls of any suboptimal arm $i$ can be bounded by decoupling the exploration and exploitation regimes,
\begin{align}
    \mathbb{E} [s_{i,n}] \leq \underbrace{\mathbb{E}\left[\sum_{q=0}^{n-1}\left(\frac{1}{H_{1,q}} - 1\right)\right]}_{\text{Exploration Regime} \longrightarrow \text{Lemma \ref{lem:B1}}} + \underbrace{\mathbb{E}\left[\sum_{q=0}^{n-1}\mathbb{I}\left\{H_{i,q}>\frac{1}{n} \right\} \right]}_{\text{Exploitation Regime} \longrightarrow \text{Lemma \ref{lem:B2}}}.
\label{eq:app_lemma_1}
\end{align}
The remainder of this proof systematically bounds these two summations via Lemma \ref{lem:B1} and Lemma \ref{lem:B2}, respectively.

\subsection{Bounding the Exploration Regime}
\begin{lemma}
\label{lem:B1}
For any $\varepsilon_\mu, \varepsilon_\sigma > 0$ such that $\varepsilon_\mu + \varepsilon_\sigma = \varepsilon$, the expected exploration penalty caused by under-sampling the optimal arm is bounded by:
\begin{align*}
    \mathbb{E}\left[\sum_{q=0}^{n-1}\left(\frac{1}{H_{1,q}} - 1\right)\right] &\leq \frac{C_1^\prime}{\varepsilon_\mu^3} + \frac{C_2^\prime }{\varepsilon_\sigma^2}\left(\sigma_1^2 + \frac{\varepsilon_\sigma L_0}{\rho \xi_1}\right)^2 + \left(C_2 + \frac{C_3}{\varepsilon_\mu} \right) \frac{C_3^\prime}{\varepsilon_\mu} + \frac{C_4^\prime}{\varepsilon_\sigma} + \frac{C_4}{\varepsilon_\sigma} \left( C_2 + \frac{C_3}{\varepsilon_\mu} \right) \left(\frac{C_5^\prime}{\varepsilon_\mu + C_6^\prime \varepsilon_\sigma} \right),
\end{align*}
where $C_1^\prime, \dots, C_6^\prime$ are finite problem-dependent constants strictly independent of $n$.
\end{lemma}

\begin{proof}
Evaluating the expectation $\mathbb{E}[1/H_{1,q} - 1]$ requires integrating a heavily skewed function, as the term is aggressively penalized when the posterior mass $H_{1,q}$ drops close to zero. We resolve this by conditioning on the empirical sufficient statistics of the optimal arm. 

Let $\vartheta = \tbar{\mu}_{1,q}$ and $\varsigma = \tbar{\sigma}_{1,q}^2$ denote the empirical sample mean and sample variance after $q$ pulls. Because the true underlying reward distribution is Gaussian $\mathcal{N}(\mu_1, \sigma_1^2)$, Cochran's theorem guarantees that these two estimators are strictly statistically independent and follow exact finite-sample distributions:
\begin{enumerate}
    \item \textbf{The Sample Mean:} $\vartheta \sim \mathcal{N}\left(\mu_1, \frac{\sigma_1^2}{q}\right)$, with probability density $f_\mu(\vartheta) = h_1 \exp\left(-\frac{q(\vartheta - \mu_1)^2}{2\sigma_1^2}\right)$, where the normalizing constant is $h_1 = \sqrt{\frac{q}{2\pi\sigma_1^2}}$.
    \item \textbf{The Sample Variance:} The scaled variance follows a chi-squared distribution, which is equivalently formulated as a Gamma distribution $\varsigma \sim \mathrm{Gamma}\left(\frac{q}{2}, \frac{q}{2\sigma_1^2}\right)$. Its density is $f_\sigma(\varsigma) = h_2 \varsigma^{\frac{q}{2}-1} \exp\left(-\frac{q\varsigma}{2\sigma_1^2}\right)$, with normalizing constant $h_2 = \frac{q^{q/2}}{2^{q/2}\Gamma(q/2)\sigma_1^q}$.
\end{enumerate}

We define the conditional posterior probability of sufficient exploitation as
\begin{align*}
    H_{1,q}^{\vartheta, \varsigma} = \mathbb{P}_t \left(\hat{\xi}_{1,t} \geq \xi_1 - \varepsilon \,\big|\, s_{1,t} = q, \tbar{\mu}_{1,q} = \vartheta, \tbar{\sigma}_{1,q}^2 = \varsigma \right).
\end{align*}
As $\vartheta$ and $\varsigma$ are nonindependent, the expected exploration penalty can be rigorously formulated as a double integral over the joint sampling distribution as,
\begin{align}
    \mathbb{E}\left[\frac{1}{H_{1,q}} - 1\right] = \int_0^\infty \int_{-\infty}^\infty \left( \frac{1}{H_{1,q}^{\vartheta, \varsigma}} - 1 \right) f_\mu(\vartheta) f_\sigma(\varsigma) \, {\rm d}\vartheta \, {\rm d}\varsigma.
    \label{eq:expectation_integral}
\end{align}

To evaluate the conditional posterior $H_{1,q}^{\vartheta, \varsigma}$, we decompose the \ac{SR} estimation error to isolate the sub-Gaussian mean sample $\theta_{1,t}$ from the sub-exponential precision sample $\tau_{1,t}$:
\begin{align*}
    \hat{\xi}_{1,t} - \xi_1 &= \frac{\theta_{1,t}}{L_0 + \frac{\rho}{\tau_{1,t}}} - \frac{\mu_1}{L_0 + \rho\sigma_1^2} = \frac{\theta_{1,t}\left(L_0 + \rho\sigma_1^2\right) - \mu_1\left(L_0 + \frac{\rho}{\tau_{1,t}}\right)}{\left(L_0 + \frac{\rho}{\tau_{1,t}}\right)\left(L_0 + \rho\sigma_1^2\right)} \\
    &= \frac{\left(\theta_{1,t} - \mu_1\right)\left(L_0 + \rho\sigma_1^2\right) + \rho\mu_1\left(\sigma_1^2 - \frac{1}{\tau_{1,t}}\right)}{\left(L_0 + \frac{\rho}{\tau_{1,t}}\right)\left(L_0 + \rho\sigma_1^2\right)}  \\
    &= \frac{\theta_{1,t} - \mu_1}{\left(L_0 + \frac{\rho}{\tau_{1,t}}\right)} + \frac{\rho \mu_1 \left(\sigma_1^2 - \frac{1}{\tau_{1,t}}\right)}{\left(L_0 + \frac{\rho}{\tau_{1,t}}\right)\left(L_0 + \rho\sigma_1^2\right)}.
\end{align*}
By partitioning the global error margin $\varepsilon$ into a mean-budget $\varepsilon_\mu$ and a variance-budget $\varepsilon_\sigma$, we establish a strict lower bound on the joint posterior tail probability using the intersection of events:
\begin{align}
    &\mathbb{P}_t \left(\frac{\theta_{1,t} - \mu_1}{L_0 + \frac{\rho}{\tau_{1,t}}} + \frac{\rho \mu_1 \left(\sigma_1^2 - \frac{1}{\tau_{1,t}}\right)}{\left(L_0 + \frac{\rho}{\tau_{1,t}}\right)\left(L_0 + \rho\sigma_1^2\right)} \geq -\varepsilon \right) \nonumber\\
    &\hspace{1cm} \geq \mathbb{P}_t \left(\frac{\theta_{1,t} - \mu_1}{L_0 + \frac{\rho}{\tau_{1,t}}} \geq -\varepsilon_\mu \right) \cdot \mathbb{P}_t \left(\frac{\rho \mu_1 \left(\sigma_1^2 - \frac{1}{\tau_{1,t}}\right)}{\left(L_0 + \frac{\rho}{\tau_{1,t}}\right)\left(L_0 + \rho\sigma_1^2\right)} \geq - \varepsilon_\sigma \right) \nonumber\\
    &\hspace{1cm} = \mathbb{P}_t \left( \theta_{1,t} \geq \mu_1 -\varepsilon_\mu L_0 - \frac{\varepsilon_\mu \rho}{\tau_{1,t}}\right) \cdot \mathbb{P}_t \left(\tau_{1,t} \geq \frac{1 - \frac{\varepsilon_\sigma}{\xi_1}}{\sigma_1^2 + \frac{\varepsilon_\sigma L_0}{\rho \xi_1}} \right) \nonumber\\
    &\hspace{1cm} \geq \mathbb{P}_t \left( \theta_{1,t} \geq \mu_1 -\varepsilon_\mu L_0 \right) \cdot \mathbb{P}_t \left(\tau_{1,t} \geq \frac{1 - \frac{\varepsilon_\sigma}{\xi_1}}{\sigma_1^2 + \frac{\varepsilon_\sigma L_0}{\rho \xi_1}} \right).
    \label{eq:tail_lower_bound}
\end{align}

To integrate over the joint sampling distribution in \eqref{eq:expectation_integral}, we partition the empirical parameter space $(\vartheta, \varsigma) \in \mathbb{R} \times \mathbb{R}^+$ into four distinct domains based on how the estimators deviate from the true moments. Let $\tau = \sigma_1^2 + \frac{\varepsilon_\sigma L_0}{\xi_1}$ act as the critical variance threshold:
\begin{itemize}
    \item \textbf{Domain A: both empirical moments under-estimate} $\vartheta \leq \mu_1, \varsigma \geq \sigma_1^2$. Both empirical moments are highly favorable. The lower bound follows directly from \eqref{eq:tail_lower_bound}.
    \item \textbf{Domain B: mean over-estimates, variance under-estimates} $\vartheta > \mu_1, \varsigma \geq \sigma_1^2$. Because the Gaussian posterior mean centers above the threshold, the mean tail probability is bounded strictly below by $1/2$, yielding $H_{1,q}^{\vartheta, \varsigma} \geq \frac{1}{2} \mathbb{P} \left(\tau_{1,t} \geq \frac{1 - \frac{\varepsilon_\sigma}{\xi_1}}{\sigma_1^2 + \frac{\varepsilon_\sigma L_0}{\rho \xi_1}} \right)$.
    \item \textbf{Domain C: mean under-estimates, variance over-estimates} $\vartheta \leq \mu_1, \varsigma < \sigma_1^2$. By symmetry on the Gamma precision tail, the variance probability is bounded by $1/2$, yielding $H_{1,q}^{\vartheta, \varsigma} \geq \frac{1}{2} \mathbb{P} \left( \theta_{1,t} \geq \mu_1 -\varepsilon_\mu L_0 \right)$.
    \item \textbf{Domain D: both empirical moments over-estimate} $\vartheta > \mu_1, \varsigma < \sigma_1^2$. Both tail probabilities naturally exceed $1/2$ by the design of the posterior updates, establishing a constant absolute lower bound $H_{1,q}^{\vartheta, \varsigma} \geq 1/4$.
\end{itemize}
Thus, based on the above four domains, the lower tail bound on the Thompson sample for the best arm is
\begin{align}
    &\mathbb{P}_t\left(\hat{\xi}_{1,t} \geq \xi_1 - \left(1 + \rho\right)\varepsilon \,\Big| s_{1,t} = q, \tbar{\mu}_{1,q} = \vartheta, \tbar{\sigma}_{1,q}^2 = \varsigma \right) \geq \nonumber\\
    &\hspace{3cm}\begin{cases}
        \mathbb{P}\left(\theta_{1,t} \geq \mu_1 -\varepsilon_\mu L_0 \right) \cdot \mathbb{P} \left(\tau_{1,t} \geq \frac{1 - \frac{\varepsilon_\sigma}{\xi_1}}{\sigma_1^2 + \frac{\varepsilon_\sigma L_0}{\rho \xi_1}} \right) & , \vartheta \leq \mu_1, \varsigma \geq \sigma_1^2 \\
        \frac{1}{2} \mathbb{P} \left(\tau_{1,t} \geq \frac{1 - \frac{\varepsilon_\sigma}{\xi_1}}{\sigma_1^2 + \frac{\varepsilon_\sigma L_0}{\rho \xi_1}} \right) & , \vartheta > \mu_1, \varsigma \geq \sigma_1^2 \\
        \frac{1}{2} \mathbb{P} \left( \theta_{1,t} \geq \mu_1 -\varepsilon_\mu L_0 \right) & , \vartheta \leq \mu_1, \varsigma < \sigma_1^2 \\
        \frac{1}{4} & , \vartheta > \mu_1, \varsigma < \sigma_1^2
    \end{cases}
    \label{eq:eq_7}
\end{align}

Substituting the explicit density functions $f_\mu(\vartheta)$ and $f_\sigma(\varsigma)$ back into the expectation \eqref{eq:expectation_integral}, we split the global integration into these four bounded regions:
\begin{align}
    \mathbb{E}\left[\frac{1}{H_{1,q}} - 1\right] &= h_1h_2 \left(\int_A + \int_B + \int_C + \int_D\right) \left(\frac{1}{H_{1,q}^{\vartheta, \varsigma}} - 1\right) \exp{\left(-\frac{q\left(\vartheta - \mu_1\right)^2}{2\sigma_1^2}\right)} \varsigma^{\frac{q}{2}-1} e^{-\frac{q\varsigma}{2\sigma_1^2}} {\rm d}\vartheta {\rm d} \varsigma.
\end{align}
We map these regions directly to the bounds established in the four domains above,
\begin{enumerate}
    \item \textbf{Region A} corresponds to Domain D: $A = \left[\mu_1,\infty\right] \times \left[0,\tau \right)$.
    \item \textbf{Region B} corresponds to Domain C: $B = \left[-\infty, \mu_1 - \varepsilon_{\mu} L_0 \right) \times \left[0, \tau \right)$.
    \item \textbf{Region C} corresponds to Domain B: $C = \left[\mu_1 - \varepsilon_{\mu} L_0, \infty \right) \times \left[\tau, \infty \right)$.
    \item \textbf{Region D} corresponds to Domain A: $D = \left[-\infty, \mu_1 - \varepsilon_{\mu} L_0 \right) \times \left[\tau, \infty \right)$.
\end{enumerate}

\textit{Evaluation of Integrals:} We bound each of the four integral domains using the Gaussian and Gamma concentration bounds established in the auxiliary lemmas (Lemmas~\ref{le:lemma_gaussian}, \ref{le:lemma_gamma}, \ref{le:lemma_gaussian2}, and \ref{le:lemma_gamma2}). We proceed by evaluating each domain sequentially, beginning with Region $A$.

\textbf{1. Integration over Region A}\\
In this domain, the optimal arm is sufficiently sampled such that the posterior failure probability is bounded ($H_{1,q}^{\vartheta, \varsigma} \geq 1/4$). Thus, we can bound the inverse utilizing the tangent inequality $\frac{1}{x} - 1 \leq 4(1 - x)$ for $x \geq 1/4$:
\begin{align*}
\frac{1}{H_{1,q}^{\vartheta, \varsigma}} - 1 \leq 4 \left(1 - H_{1,q}^{\vartheta, \varsigma}\right).
\end{align*}
Substituting this into the integral and splitting the probability of the union of events yields:
\begin{align*}
    &h_1h_2 \int_A \left(\frac{1}{H_{1,q}^{\vartheta, \varsigma}} - 1\right) \exp{\left(-\frac{q\left(\vartheta - \mu_1\right)^2}{2\sigma_1^2}\right)} \varsigma^{\frac{q}{2}-1} e^{-\frac{q\varsigma}{2\sigma_1^2}} {\rm d}\vartheta {\rm d} \varsigma \\
    &\hspace*{1cm} \leq 4h_1h_2 \int_A \left(1- H_{1,q}^{\vartheta, \varsigma}\right) \exp{\left(-\frac{q\left(\vartheta - \mu_1\right)^2}{2\sigma_1^2}\right)} \varsigma^{\frac{q}{2}-1} e^{-\frac{q\varsigma}{2\sigma_1^2}} {\rm d}\vartheta {\rm d} \varsigma\\
    &\hspace*{0cm} \leq 4h_1h_2 \int_A \left(\mathbb{P} \left( \theta_{i,t} \leq \mu_1 - \frac{\varepsilon_\mu L_0}{2} \bigg| \tbar{\mu}_{1,s} = \vartheta \right) + \mathbb{P} \bigg(\tau_{i,t} \leq \frac{1 - \frac{\varepsilon_\sigma}{\xi_1}}{\sigma_1^2 + \frac{\varepsilon_\sigma L_0}{\rho \xi_1}} \bigg| \tbar{\sigma}_{1,s}^2 = \varsigma \bigg) \right) \exp{\left(-\frac{q\left(\vartheta - \mu_1\right)^2}{2\sigma_1^2}\right)} \varsigma^{\frac{q}{2}-1} e^{-\frac{q\varsigma}{2\sigma_1^2}} {\rm d}\vartheta {\rm d} \varsigma \\
    &\hspace*{0cm} \leq 4h_1 \int_{\mu_1}^\infty \mathbb{P} \left( \theta_{i,t} \leq \mu_1 - \frac{\varepsilon_\mu L_0}{2} \bigg| \tbar{\mu}_{1,s} = \vartheta \right) \exp{\left(-\frac{q\left(\vartheta - \mu_1\right)^2}{2\sigma_1^2}\right)} {\rm d}\vartheta + 4h_2 \int_0^\tau \mathbb{P} \left(\tau_{i,t} \leq \frac{1 - \frac{\varepsilon_\sigma}{\xi_1}}{\sigma_1^2 + \frac{\varepsilon_\sigma L_0}{\rho \xi_1}} \bigg| \tbar{\sigma}_{1,s}^2 = \varsigma \right) \varsigma^{\frac{q}{2}-1} e^{-\frac{q\varsigma}{2\sigma_1^2}} {\rm d} \varsigma \\
    &\hspace*{0cm} \leq 4h_1 \int_{\mu_1}^\infty \mathbb{P} \left( \theta_{i,t} \leq \mu_1 - \frac{\varepsilon_\mu L_0}{2} \bigg| \tbar{\mu}_{1,s} = \vartheta \right) \exp{\left(-\frac{q\left(\vartheta - \mu_1\right)^2}{2\sigma_1^2}\right)} {\rm d}\vartheta + 4h_2 \int_0^\tau \mathbb{P} \left(\tau_{i,t} \leq \frac{1}{\sigma_1^2 + \frac{\varepsilon_\sigma L_0}{\rho \xi_1}} \bigg| \tbar{\sigma}_{1,s}^2 = \varsigma \right) \varsigma^{\frac{q}{2}-1} e^{-\frac{q\varsigma}{2\sigma_1^2}} {\rm d} \varsigma \\
    &\hspace*{0cm} \overset{(a)}{\leq} \underbrace{4h_1 \int_{\mu_1}^\infty \exp{\left(-\frac{q \left(\vartheta - \mu_1+\frac{\varepsilon_\mu L_0}{2} \right)^2}{2}\right)} \exp{\left(-\frac{q\left(\vartheta - \mu_1\right)^2}{2\sigma_1^2}\right)} {\rm d}\vartheta}_{I_1} + \underbrace{4h_2 \int_0^\tau \exp{\left(-\frac{q \left(\varsigma - \left(\sigma_1^2 + \frac{\varepsilon_\sigma L_0}{\rho \xi_1}\right)\right)^2}{4 \left(\sigma_1^2 + \frac{\varepsilon_\sigma L_0}{\rho \xi_1}\right)^2}\right)} \varsigma^{\frac{q}{2}-1} e^{-\frac{q\varsigma}{2\sigma_1^2}} {\rm d} \varsigma}_{I_2}.
\end{align*}
Step (a) follows directly from the Gaussian and Gamma left tail bounds, i.e., Lemma~\ref{le:lemma_gaussian} and Lemma~\ref{le:lemma_gamma} respectively. By defining $x = 1 + \frac{1}{\sigma_1^2}$, $y = \varepsilon_\mu L_0$, making substitution, $u = \vartheta - \mu_1 + \frac{y}{2x}$, and completing the square for the Gaussian integral, we bound $I_1$ by,
\begin{align*}
    I_1 &= 4h_1 \exp\left(-\frac{q \varepsilon_\mu^2 L_0^2}{8} + \frac{q y^2}{8x} \right) \int_{\frac{y}{2x}}^\infty \exp\left( -\frac{q x}{2} u^2 \right) {\rm d}u \\
    &\overset{(b)}{\leq} 4h_1 \exp\left(-\frac{q \varepsilon_\mu^2 L_0^2}{8} + \frac{q y^2}{8x} \right) \cdot \frac{2}{q y} \exp\left( -\frac{q y^2}{8x} \right) = \frac{8 h_1}{q y} \exp\left(-\frac{q \varepsilon_\mu^2 L_0^2}{8} \right) \\
    &= \sqrt{\frac{q}{2 \pi \sigma_1^2}}\frac{8}{q \varepsilon_\mu L_0} \exp\left(-\frac{q \varepsilon_\mu^2 L_0^2}{8} \right) \leq C_1 \exp\left(-\frac{q \varepsilon_\mu^2 L_0^2}{8} \right).
\end{align*}
Step (b) follows by using the standard Gaussian tail bound, i.e., $\int_z^\infty e^{-a u^2} {\rm d}u \leq \frac{e^{-a z^2}}{2 a z}$, for $z > 0$. For $I_2$, to isolate the exponential decay with respect to $q$, we partition the integration domain using the  concentration event $\mathcal{E} = \left\{ \varsigma \leq \sigma_1^2 + \frac{\varepsilon_\sigma L_0}{2\rho \xi_1} \right\}$. By the Law of Total Probability, we split the integral into two regimes:
\begin{align*}
    I_2 = 4 \int_{\mathcal{E}} \mathbb{P}\left(\tau_{1,t} \leq \frac{1}{\sigma_1^2 + \frac{\varepsilon_\sigma L_0}{\rho \xi_1}} \bigg| \varsigma \right) f(\varsigma) {\rm d}\varsigma + 4 \int_{\mathcal{E}^C} \mathbb{P}\left(\tau_{1,t} \leq \frac{1}{\sigma_1^2 + \frac{\varepsilon_\sigma L_0}{\rho \xi_1}} \bigg| \varsigma \right) f(\varsigma) {\rm d}\varsigma.
\end{align*}
Conditioned on $\mathcal{E}$, the empirical variance $\varsigma$ is strictly bounded away from the failure threshold $\tau = \sigma_1^2 + \frac{\varepsilon_\sigma L_0}{\rho \xi_1}$. The distance between them is at least $\left(\sigma_1^2 + \frac{\varepsilon_\sigma L_0}{\rho \xi_1}\right) - \varsigma \geq \left(\sigma_1^2 + \frac{\varepsilon_\sigma L_0}{\rho \xi_1}\right) - \left(\sigma_1^2 + \frac{\varepsilon_\sigma L_0}{2\rho \xi_1}\right) = \frac{\varepsilon_\sigma L_0}{2\rho \xi_1}$. 
Substituting this worst-case gap into the Gamma posterior tail bound, i.e., Lemma~\ref{le:lemma_gamma}, yields a deterministic exponential upper bound that is entirely independent of the integration variable $\varsigma$,
\begin{align*}
    \mathbb{P}\left(\tau_{1,t} \leq \frac{1}{\sigma_1^2 + \frac{\varepsilon_\sigma L_0}{\rho \xi_1}} \bigg| \varsigma \in \mathcal{E} \right) \leq \exp\left( - \frac{q \left(\left(\sigma_1^2 + \frac{\varepsilon_\sigma L_0}{\rho \xi_1}\right) - \varsigma\right)^2}{4 \left(\sigma_1^2 + \frac{\varepsilon_\sigma L_0}{\rho \xi_1}\right)^2} \right) \leq \exp\left( - \frac{q \varepsilon_\sigma^2 L_0^2}{16 \rho^2 \xi_1^2 \left(\sigma_1^2 + \frac{\varepsilon_\sigma L_0}{\rho \xi_1}\right)^2} \right).
\end{align*}
Because this bound is a constant with respect to $\varsigma$, the remaining integral over the probability density trivially integrates to a value less than or equal to $1$:
\begin{align*}
    4 \exp\left( - \frac{q \varepsilon_\sigma^2 L_0^2}{16 \rho^2 \xi_1^2 \left(\sigma_1^2 + \frac{\varepsilon_\sigma L_0}{\rho \xi_1}\right)^2} \right) \int_{\mathcal{E}} f(\varsigma) {\rm d}\varsigma \leq 4 \exp\left( - \frac{q \varepsilon_\sigma^2 L_0^2}{16 \rho^2 \xi_1^2 \left(\sigma_1^2 + \frac{\varepsilon_\sigma L_0}{\rho \xi_1}\right)^2} \right).
\end{align*}
For the complement event $\mathcal{E}^C$, the empirical variance has severely deviated from its true mean. We trivially upper bound the conditional posterior failure probability by $1$. The remaining integral evaluates exactly to the prior probability of this large deviation occurring:
\begin{align*}
    4 \int_{\mathcal{E}^C} 1 \cdot f(\varsigma) {\rm d}\varsigma = 4 \, \mathbb{P}\left( \tbar{\sigma}_{1,q}^2 > \sigma_1^2 + \frac{\varepsilon_\sigma L_0}{2\rho \xi_1} \right) \leq 4 \exp\left(- c_1 q \frac{\varepsilon_\sigma^2 L_0^2}{4\rho^2 \xi_1^2} \right).
\end{align*}
Because $\tbar{\sigma}_{1,q}^2$ is an empirical variance calculated over $q$ samples, standard sub-exponential Chernoff bounds for the Gamma distribution dictate that the probability of this right-tail deviation decays exponentially as $4 \exp\left(- c_1 q \frac{\varepsilon_\sigma^2 L_0^2}{4\rho^2 \xi_1^2} \right)$, where $c_1 > 0$ is a distribution-dependent constant independent of $q$. Summing the partitioned bounds establishes that $I_2$ is strictly governed by the slower of the two exponential decay rates. We seamlessly absorb the dependencies to yield a final bound strictly in the form,
\begin{align*}
    I_2 \leq 4 \exp\left( - \frac{q \varepsilon_\sigma^2 L_0^2}{16 \rho^2 \xi_1^2 \left(\sigma_1^2 + \frac{\varepsilon_\sigma L_0}{\rho \xi_1}\right)^2} \right) + 4 \exp\left(- c_2 q \frac{\varepsilon_\sigma^2 L_0^2}{4\rho^2 \xi_1^2} \right) \leq 8 \exp\left( - \frac{q \varepsilon_\sigma^2 L_0^2}{16 \rho^2 \xi_1^2 \left(\sigma_1^2 + \frac{\varepsilon_\sigma L_0}{\rho \xi_1}\right)^2} \right).
\end{align*}

Substituting the derived bounds for $I_1$ and $I_2$ back into the original expression yields the final result for region $A$,
\begin{align}
    h_1h_2 \int_A \left(\frac{1}{H_{1,q}^{\vartheta, \varsigma}} - 1\right) \exp{\left(-\frac{q\left(\vartheta - \mu_1\right)^2}{2\sigma_1^2}\right)} \varsigma^{\frac{q}{2}-1} e^{-\frac{q\varsigma}{2\sigma_1^2}} {\rm d}\vartheta {\rm d} \varsigma \leq C_1 \exp\left(-\frac{q \varepsilon_\mu^2 L_0^2}{8} \right) + 8 \exp\left( - \frac{q \varepsilon_\sigma^2 L_0^2}{16 \rho^2 \xi_1^2 \left(\sigma_1^2 + \frac{\varepsilon_\sigma L_0}{\rho \xi_1}\right)^2} \right).
    \label{eq:eq_A}
\end{align}

\textbf{2. Integration over Region B}\\
In this domain, the variance empirical sample provides a strict bound, leaving the inverse heavily dependent on the mean's tail
\begin{align*}
    \frac{1}{H_{1,q}^{\vartheta, \varsigma}} - 1 \leq \frac{2}{\mathbb{P} \left( \theta_{1,t} \geq \mu_1 -\varepsilon_\mu L_0 \mid \tbar{\mu}{1,q} = \vartheta\right)}.
\end{align*}
Substituting this bound and applying the two-sided Gaussian tail inequality, i.e., Lemma~\ref{le:lemma_gaussian2}, yields,
\begin{align*}
    &h_1h_2 \int_B \left(\frac{1}{H{1,q}^{\vartheta, \varsigma}} - 1\right) \exp{\left(-\frac{q\left(\vartheta - \mu_1\right)^2}{2\sigma_1^2}\right)} \varsigma^{\frac{q}{2}-1} e^{-\frac{q\varsigma}{2\sigma_1^2}} {\rm d}\vartheta {\rm d} \varsigma \\
    &\hspace*{0.5cm}\leq 2h_1 \int_{-\infty}^{\mu_1-\varepsilon_\mu L_0} \frac{1}{\mathbb{P} \left( \theta_{1,t} \geq \mu_1 -\varepsilon_\mu L_0 \mid \tbar{\mu}*{1,q} = \vartheta\right)} \exp{\left(-\frac{q\left(\vartheta - \mu_1\right)^2}{2\sigma_1^2}\right)} {\rm d}\vartheta \\
    &\hspace*{0.5cm}\leq 2h_1 \int_{-\infty}^{\mu_1-\varepsilon_\mu L_0} \left(\sqrt{q} \left(\mu_1 - \vartheta - \varepsilon_\mu L_0\right) + \sqrt{q\left(\mu_1 - \vartheta - \varepsilon_\mu L_0\right)^2 + 4}\right) \\
    &\hspace*{3.5cm} \times \exp{\left(\frac{q\left(\vartheta - \mu_1 + \varepsilon_\mu L_0\right)^2}{2}\right)} \exp{\left(-\frac{q\left(\vartheta - \mu_1\right)^2}{2\sigma_1^2}\right)} {\rm d}\vartheta.
\end{align*}
Applying the transformation $z = \left(\mu_1 - \vartheta - \varepsilon_\mu L_0\right)$, thus limits of $z$ are from $0$ to $\infty$, the above equation becomes
\begin{align*}
   &2h_1 \int_{0}^{\infty} \left(\sqrt{q} z + \sqrt{qz^2 + 4}\right) \exp{\left(\frac{qz^2}{2}\right)} \exp{\left(-\frac{q\left(z + \varepsilon_\mu L_0\right)^2}{2\sigma_1^2}\right)} {\rm d}z \\
   &\hspace*{1cm} \leq 2h_1 \exp{\left(-\frac{q \varepsilon_\mu^2 L_0^2}{2\sigma_1^2}\right)} \int_{0}^{\infty} \left(\sqrt{q} z + \sqrt{qz^2 + 4}\right) \exp{\left(\frac{qz^2}{2}\right)} \exp{\left(-\frac{qz^2}{2\sigma_1^2}\right)} \exp{\left(-\frac{q z \varepsilon_\mu L_0}{\sigma_1^2}\right)} {\rm d}z.
\end{align*}
Define constants $a = \frac{1}{2\sigma_1^2} - \frac{1}{2} = \frac{1 - \sigma_1^2}{2\sigma_1^2}$ and $b = \frac{\varepsilon_\mu L_0}{\sigma_1^2}$. Then, for $\sqrt{qz^2 + 4} \leq \sqrt{q}z + 2$, the integral becomes
\begin{align*}
    &2h_1 \exp\left(-\frac{q \varepsilon_\mu^2 L_0^2}{2\sigma_1^2}\right) \int_0^\infty 2\left(\sqrt{q}z + 1 \right) \exp\left( -q a z^2 - q b z \right) {\rm d}z \\
    &= 4h_1 \sqrt{q} \exp\left(-\frac{q \varepsilon_\mu^2 L_0^2}{2\sigma_1^2}\right) \underbrace{\int_0^\infty z \exp(-q a z^2 - q b z) dz}_{I_1} + 4h_1 \exp\left(-\frac{q \varepsilon_\mu^2 L_0^2}{2\sigma_1^2}\right) \underbrace{\int_0^\infty \exp(-q a z^2 - q b z) dz}_{I_2}. 
\end{align*}
Since $a = \frac{1}{2\sigma_1^2} - \frac{1}{2}$, for convergence of the integrals, we require $a \geq 0$, i.e., $\sigma_1^2 \leq 1$. Under this assumption, both integrals converge. Completing the squares yield $\exp(-q a z^2 - q b z) = \exp\left(\frac{q b^2}{4a}\right) \cdot \exp\left( -q a \left(z + \frac{b}{2a} \right)^2 \right)$ which solves $I_2$ as
\begin{align*}
    I_2 &= \exp\left(\frac{q b^2}{4a}\right) \int_0^\infty \exp\left( -q a \left(z + \frac{b}{2a} \right)^2 \right) {\rm d}z \\
    &= \frac{1}{2}\sqrt{\frac{\pi}{qa}} \exp\left(\frac{q b^2}{4a}\right) \operatorname{erfc}\left(\frac{b}{2} \sqrt{\frac{q}{a}}\right) \leq \frac{1}{2}\sqrt{\frac{\pi}{qa}} \exp\left(\frac{q b^2}{4a}\right) \times \exp{\left(-\frac{qb^2}{4a}\right)} \frac{2}{b}\sqrt{\frac{a}{\pi q}} \\
    &= \frac{1}{bq} = \frac{\sigma_1^2}{q \varepsilon_\mu L_0}.
\end{align*}
For brevity let $\zeta := qb$. Moreover,
\begin{align*}
    \frac{\partial}{\partial \zeta} I_2(\zeta) = \int_0^\infty \frac{\partial}{\partial \zeta}\left(e^{-q a z^2 - \zeta z}\right)\,{\rm d}z = - \int_0^\infty z\, e^{-\alpha z^2 - \zeta z}\,{\rm d}z = -I_1,
\end{align*} 
so $I_1 = -\dfrac{{\rm d}}{{\rm d}\zeta} I_2$. Differentiating the closed form for $I_2$ gives
\begin{align*}
    \frac{{\rm d}}{{\rm d}\zeta} \left[ \frac{1}{2}\sqrt{\frac{\pi}{qa}} \exp\left(\frac{q b^2}{4a}\right) \operatorname{erfc}\left(\frac{b}{2} \sqrt{\frac{q}{a}}\right)\right] &= \frac{1}{2}\sqrt{\frac{\pi}{qa}}\, \frac{{\rm d}}{{\rm d}\zeta} \left[\exp\left(\frac{\zeta^2}{4qa}\right) \operatorname{erfc}\left(\frac{\zeta}{2\sqrt{qa}}\right)\right] \\
    &\hspace*{-3cm}= \frac{1}{2}\sqrt{\frac{\pi}{qa}}\, \left[\exp\left(\frac{\zeta^2}{4qa}\right) \frac{\zeta}{qa}\operatorname{erfc}\left(\frac{\zeta}{2\sqrt{qa}}\right) + \exp\left(\frac{\zeta^2}{4qa}\right) \left(-\frac{1}{\sqrt{\pi qa}} \exp{\left(-\frac{\zeta^2}{4qa}\right)}\right)\right] \\
    &\hspace*{-3cm}= \frac{1}{2}\sqrt{\frac{\pi}{qa}}\, \left[\exp\left(\frac{\zeta^2}{4qa}\right) \frac{\zeta}{qa}\operatorname{erfc}\left(\frac{\zeta}{2\sqrt{qa}}\right) - \frac{1}{\sqrt{\pi qa}}\right].
\end{align*}
Hence,
\begin{align*}
    I_1 = -\dfrac{{\rm d}}{{\rm d}\zeta} I_2 &= \frac{1}{2}\sqrt{\frac{\pi}{qa}}\, \left[\frac{1}{\sqrt{\pi qa}} - \exp\left(\frac{qb^2}{4a}\right)\frac{b}{4a} \operatorname{erfc}\left(\frac{b}{2} \sqrt{\frac{q}{a}}\right)\right] \\
    &\leq \frac{1}{2 aq} = \frac{\sigma_1^2}{q \left(1 - \sigma_1^2\right)}.
\end{align*}
Substituting the derived bounds for $I_1$ and $I_2$ back into the original expression yields the final result for this region is
\begin{align}
    2h_1 \exp{\left(-\frac{q \varepsilon_\mu^2 L_0^2}{2\sigma_1^2}\right)} \left(\sqrt{q} I_1 + I_2\right) &\leq 2h_1 \exp{\left(-\frac{q \varepsilon_\mu^2 L_0^2}{2\sigma_1^2}\right)} \left(\sqrt{q} \frac{\sigma_1^2}{q \left(1 - \sigma_1^2\right)} + \frac{\sigma_1^2}{q \varepsilon_\mu L_0}\right)\nonumber \\
    &= \exp{\left(-\frac{q \varepsilon_\mu^2 L_0^2}{2\sigma_1^2}\right)} \left( C_2 + \frac{C_3}{\varepsilon_\mu} \right).
    \label{eq:eq_B}
\end{align}

\textbf{3. Integration over Region C}\\
In this domain, the mean empirical sample provides the strict bound, shifting the responsibility to the Gamma precision tail
\begin{align*}
    \frac{1}{H_{1,q}^{\vartheta, \varsigma}} - 1 \leq \frac{2}{\mathbb{P} \left(\tau_{1,t} \geq \frac{1 - \frac{\varepsilon_\sigma}{\xi_1}}{\sigma_1^2 + \frac{\varepsilon_\sigma L_0}{\rho \xi_1}} \,\Big|\, \tbar{\sigma}_{1,q}^2 = \varsigma \right)} \leq \frac{2}{\mathbb{P} \left(\tau_{1,t} \geq \frac{1}{\sigma_1^2 + \frac{\varepsilon_\sigma L_0}{\rho \xi_1}} \,\Big|\, \tbar{\sigma}_{1,q}^2 = \varsigma \right)}.
\end{align*}
Applying the two-sided Mills-type bound for the Gamma distribution, i.e., Lemma~\ref{le:lemma_gamma2}, allows us to invert the probability and bound the integral,
\begin{align}
    &h_1h_2 \int_C \left(\frac{1}{H_{1,q}^{\vartheta, \varsigma}} - 1\right) \exp{\left(-\frac{q\left(\vartheta - \mu_1\right)^2}{2\sigma_1^2}\right)} \varsigma^{\frac{q}{2}-1} e^{-\frac{q\varsigma}{2\sigma_1^2}} {\rm d}\vartheta {\rm d} \varsigma \nonumber\\
    &\hspace*{1cm}\leq 2h_1h_2 \int_B \frac{1}{\mathbb{P} \left(\tau_{1,t} \geq \frac{1}{\sigma_1^2 + \frac{\varepsilon_\sigma L_0}{\rho \xi_1}} \,\Big|\, \tbar{\sigma}_{1,q}^2 = \varsigma \right)} \exp{\left(-\frac{q\left(\vartheta - \mu_1\right)^2}{2\sigma_1^2}\right)} \varsigma^{\frac{q}{2}-1} e^{-\frac{q\varsigma}{2\sigma_1^2}} {\rm d}\vartheta {\rm d} \varsigma \nonumber\\
    &\hspace*{1cm}\leq h_1 \int_{\mu_1}^{\infty} \exp{\left(-\frac{q\left(\vartheta - \mu_1\right)^2}{2\sigma_1^2}\right)} {\rm d}\vartheta \cdot 2h_2 \int_\tau^\infty \frac{1}{\mathbb{P} \left(\tau_{1,t} \geq \frac{1}{\sigma_1^2 + \frac{\varepsilon_\sigma L_0}{\rho \xi_1}} \,\Big|\, \tbar{\sigma}_{1,q}^2 = \varsigma \right)} \varsigma^{\frac{q}{2}-1} e^{-\frac{q\varsigma}{2\sigma_1^2}} {\rm d} \varsigma \nonumber\\
    &\hspace*{1cm}\leq 2h_2 \int_\tau^\infty \frac{1}{\mathbb{P} \left(\tau_{1,t} \geq \frac{1}{\sigma_1^2 + \frac{\varepsilon_\sigma L_0}{\rho \xi_1}} \,\Big|\, \tbar{\sigma}_{1,q}^2 = \varsigma \right)} \varsigma^{\frac{q}{2}-1} e^{-\frac{q\varsigma}{2\sigma_1^2}} {\rm d} \varsigma \nonumber\\
    &\hspace*{1cm} \overset{(a)}{\leq} \frac{2 q^{q/2}}{2^{q/2} \Gamma\left(q/2\right)\sigma_1^{q}} \int_\tau^\infty \frac{2^{\frac{q}{2}-1} \Gamma\left(q/2\right)}{q^{\frac{q}{2}-1} \varsigma^{\frac{q}{2}-1}} \left(\sigma_1^2 + \frac{\varepsilon_\sigma L_0}{\rho \xi_1}\right)^{\frac{q}{2}-1} \exp{\left(\frac{q \varsigma}{2\left(\sigma_1^2 + \frac{\varepsilon_\sigma L_0}{\rho \xi_1}\right)}\right)} \varsigma^{\frac{q}{2}-1} e^{-\frac{q\varsigma}{2\sigma_1^2}} {\rm d} \varsigma \nonumber\\
    &\hspace*{1cm}= \frac{q}{\sigma_1^{q}} \left(\sigma_1^2 + \frac{\varepsilon_\sigma L_0}{\rho \xi_1}\right)^{\frac{q}{2}-1} \int_\tau^\infty \exp{\left(-\frac{q \varsigma \varepsilon_\sigma L_0}{2\rho\xi_1\sigma_1^2 \left(\sigma_1^2 + \frac{\varepsilon_\sigma L_0}{\rho \xi_1}\right)}\right)}{\rm d} \varsigma \nonumber\\
    &\hspace*{1cm}= \frac{q}{\sigma_1^{q}} \left(\sigma_1^2 + \frac{\varepsilon_\sigma L_0}{\rho \xi_1}\right)^{\frac{q}{2}-1} \times \frac{2 \rho \xi_1\sigma_1^2 \left(\sigma_1^2 + \frac{\varepsilon_\sigma L_0}{\rho \xi_1}\right)}{q \varepsilon_\sigma L_0} \exp{\left(-\frac{q \varepsilon_\sigma L_0}{2 \rho \xi_1\sigma_1^2}\right)} \nonumber\\
    &\hspace*{1cm}= 2 \rho \xi_1 \left(1 + \frac{\varepsilon_\sigma L_0}{\sigma_1^2 \rho \xi_1}\right)^{\frac{q}{2}} \frac{\sigma_1^2}{ \varepsilon_\sigma L_0} \exp{\left(-\frac{q \varepsilon_\sigma L_0}{2 \rho \xi_1\sigma_1^2}\right)} \nonumber\\
    &\hspace*{1cm}\leq \frac{2 \rho \xi_1\sigma_1^2}{\varepsilon_\sigma L_0} \exp{\left(\frac{q\varepsilon_\sigma L_0}{2\rho \xi_1}\right)} \exp{\left(-\frac{q \varepsilon_\sigma L_0}{2 \rho \xi_1\sigma_1^2}\right)} \nonumber\\
    &\hspace*{1cm}\leq \frac{C_4}{\varepsilon_\sigma} \exp{\left(-\frac{q \varepsilon_\sigma L_0}{2 \rho \xi_1} \left(\frac{1}{\sigma_1^2} - 1 \right)\right)}.
    \label{eq:eq_C}
\end{align}
Step (a) follows from Lemma~\ref{le:lemma_gamma2}.

\textbf{4. Integration over Region D}\\
In this domain, both moments underestimate the risk, making the failure probability strictly bounded by the product of both individual tail probabilities
\begin{align*}
    \frac{1}{H_{1,q}^{\vartheta, \varsigma}} - 1 \leq \frac{4}{\mathbb{P}\left(\theta_{1,t} \geq \mu_1 -\varepsilon_\mu L_0 \right) \cdot \mathbb{P} \left(\tau_{1,t} \geq \frac{1 - \frac{\varepsilon_\sigma}{\xi_1}}{\sigma_1^2 + \frac{\varepsilon_\sigma L_0}{\rho \xi_1}} \right)} \leq \frac{4}{\mathbb{P}\left(\theta_{1,t} \geq \mu_1 -\varepsilon_\mu L_0 \right) \cdot \mathbb{P} \left(\tau_{1,t} \geq \frac{1}{\sigma_1^2 + \frac{\varepsilon_\sigma L_0}{\rho \xi_1}} \right)}. 
\end{align*}
Because the integral across Domain $D$ naturally factors into the independent integrals evaluated in Domain $B$ and Domain $C$, the final bound is strictly the product of their respective asymptotic limits,
\begin{align}
    &h_1h_2 \int_D \left(\frac{1}{H_{1,q}^{\vartheta, \varsigma}} - 1\right) \exp{\left(-\frac{q\left(\vartheta - \mu_1\right)^2}{2\sigma_1^2}\right)} \varsigma^{\frac{q}{2}-1} e^{-\frac{q\varsigma}{2\sigma_1^2}} {\rm d}\vartheta {\rm d} \varsigma \leq \nonumber \\
    &\hspace*{2.5cm} \left( C_2 + \frac{C_3}{\varepsilon_\mu} \right) \exp{\left(-\frac{q \varepsilon_\mu^2 L_0^2}{2\sigma_1^2}\right)} \times \frac{C_4}{\varepsilon_\sigma} \exp{\left(-\frac{q \varepsilon_\sigma L_0}{2 \rho \xi_1} \left(\frac{1}{\sigma_1^2} - 1 \right)\right)}.
    \label{eq:eq_D}
\end{align}

\hspace{-0.34cm}\textbf{Aggregation of the Integrals:}\\
Summing the evaluations using~\eqref{eq:eq_A},~\eqref{eq:eq_B},~\eqref{eq:eq_C}, and~\eqref{eq:eq_A} of Region $A$, $B$, $C$, and $D$ respectively yields the expectation $\mathbb{E}[1/H_{1,q} - 1]$ conditioned on $q$ pulls as
\begin{align*}
    \mathbb{E} \left[\frac{1 - H_{1,s}}{H_{1,s}}\right] &\leq \frac{C_1}{\varepsilon_\mu} \exp\left(-\frac{q \varepsilon_\mu^2 L_0^2}{8} \right) + 8 \exp{\left(-\frac{q\varepsilon_\sigma^2 L_0^2}{16\rho^2 \xi_1^2 \left(\sigma_1^2 + \frac{\varepsilon_\sigma L_0}{\rho \xi_1}\right)^2}\right)} + \left( C_2 + \frac{C_3}{\varepsilon_\mu} \right) \exp{\left(-\frac{q \varepsilon_\mu^2 L_0^2}{2\sigma_1^2}\right)} \\
    &\hspace*{-1cm} + \frac{C_4}{\varepsilon_\sigma} \exp{\left(-\frac{q \varepsilon_\sigma L_0}{2 \rho \xi_1} \left(\frac{1}{\sigma_1^2} - 1 \right)\right)} + \left( C_2 + \frac{C_3}{\varepsilon_\mu} \right) \exp{\left(-\frac{q \varepsilon_\mu^2 L_0^2}{2\sigma_1^2}\right)} \times \frac{C_4}{\varepsilon_\sigma} \exp{\left(-\frac{q \varepsilon_\sigma L_0}{2 \rho \xi_1} \left(\frac{1}{\sigma_1^2} - 1 \right)\right)}.
\end{align*}
Summing this over all possible pull counts $q \in \{1, \dots, \infty\}$ requires resolving series of the form $\sum_{q=1}^\infty e^{-aq} \leq \frac{1}{a}$ for $a > 0$. Applying this geometric series upper bound strictly isolates the $\varepsilon$ margins into a finite polynomial:
\begin{align}
    \sum_{q=1}^\infty \mathbb{E} \left[\frac{1 - H_{1,s}}{H_{1,s}}\right] &\leq \frac{C_1^\prime}{\varepsilon_\mu^3} + \frac{C_2^\prime }{\varepsilon_\sigma^2}\left(\sigma_1^2 + \frac{\varepsilon_\sigma L_0}{\rho \xi_1}\right)^2 + \left(C_2 + \frac{C_3}{\varepsilon_\mu} \right) \frac{C_3^\prime}{\varepsilon_\mu} + \frac{C_4^\prime}{\varepsilon_\sigma} + \frac{C_4}{\varepsilon_\sigma} \left( C_2 + \frac{C_3}{\varepsilon_\mu} \right) \left(\frac{C_5^\prime}{\varepsilon_\mu + C_6^\prime \varepsilon_\sigma} \right) .
    \label{eq:final_1}
\end{align}
This concludes the proof of Lemma \ref{lem:B1}, confirming that the exploration penalty is bounded independent of $n$.
\end{proof}

\subsection{Bounding the Exploitation Regime}
\begin{lemma}
\label{lem:B2}
For any suboptimal arm $i$, the expected number of pulls resulting from the residual failure of the posterior concentration is bounded by:
\begin{align*}
    \mathbb{E}\left[\sum_{q=0}^{n-1}\mathbb{I}\left\{H_{i,q}>\frac{1}{n} \right\} \right] \leq 1 + \max \left\{\frac{2\log{(2n)}}{\big(\Lambda_{1,i} - \varepsilon_\mu L_0\big)^2} , \frac{\log{(2n)}}{h\left(\frac{\sigma_i^2\left(1 - \frac{\varepsilon_\sigma}{\xi_1}\right)}{\sigma_1^2}\right)}\right\} + \frac{8 \sigma_i^2}{\varepsilon_{\mu}^2 L_0^2} + \frac{4 \sigma_1^4 \rho^2\xi_1^2}{\varepsilon_\sigma^2 L_0^2}.
\end{align*}
where $h(x) = (x - 1 - \log x)/2$ and $\Lambda_{1,i} = \mu_1 - \mu_i$.
\end{lemma}
\begin{proof}
To bound the indicator function, we must evaluate the conditions under which $H_{i,q} > 1/n$. Recall that $H_{i,q} = \mathbb{P}_t\left(\hat{\xi}_{i,t} > \xi_1 - \varepsilon \mid s_{i,t} = q \right)$. This probability represents the event where a suboptimal arm generates an erroneously high \ac{SR} sample despite having been pulled $q$ times. We begin by decomposing the sampled \ac{SR} gap to isolate the sampled mean $\theta_{i,t}$ and precision $\tau_{i,t}$, following the algebraic expansion from Lemma~\ref{lem:B1},
\begin{align*}
    \widehat{\xi}_{i,t} - \xi_1 = \frac{\theta_{i,t} - \mu_1}{\left(L_0 + \frac{\rho}{\tau_{i,t}}\right)} + \frac{\rho \mu_1 \left(\sigma_1^2 - \frac{1}{\tau_{i,t}}\right)}{\left(L_0 + \frac{\rho}{\tau_{i,t}}\right)\left(L_0 + \rho\sigma_1^2\right)}.
\end{align*}
Applying the union bound over the partitioned error margins $\varepsilon_\mu + \varepsilon_\sigma = \varepsilon$, we can write $H_{i,q}$ conditioned on $\tbar{\mu}_{i,q} = \mu$ and $\tbar{\sigma}_{i,q}^2 = \sigma^2$ as,
\begin{align*}
    &\mathbb{P} \left(\frac{\theta_{i,t} - \mu_1}{\left(L_0 + \frac{\rho}{\tau_{i,t}}\right)} + \frac{\rho \mu_1 \left(\sigma_1^2 - \frac{1}{\tau_{i,t}}\right)}{\left(L_0 + \frac{\rho}{\tau_{i,t}}\right)\left(L_0 + \rho\sigma_1^2\right)} \geq - \left(\varepsilon_\mu + \varepsilon_\sigma \right) \right) \\
    &\hspace{1.5cm} \leq \mathbb{P} \left(\frac{\theta_{i,t} - \mu_1}{\left(L_0 + \frac{\rho}{\tau_{i,t}}\right)} \geq -\varepsilon_\mu \right) + \mathbb{P} \left(\frac{\rho \mu_1 \left(\sigma_1^2 - \frac{1}{\tau_{i,t}}\right)}{\left(L_0 + \frac{\rho}{\tau_{i,t}}\right)\left(L_0 + \rho\sigma_1^2\right)} \geq -\varepsilon_\sigma \right)\\
    &\hspace{1.5cm} \leq \mathbb{P} \left(\theta_{i,t} - \mu_1 \geq -\varepsilon_\mu L_0 \right) + \mathbb{P} \left(\frac{\sigma_1^2 - \frac{1}{\tau_{i,t}}}{\left(L_0 + \frac{\rho}{\tau_{i,t}}\right)} \geq -\frac{\varepsilon_\sigma}{\rho \xi_1}  \right) \\
    &\hspace{1.5cm} = \mathbb{P} \left(\theta_{i,t} + \frac{\varepsilon_\mu\rho}{\tau_{i,t}} \geq \mu_1 - \varepsilon_\mu L_0 \right) + \mathbb{P} \left(\tau_{i,t} \geq \frac{1 - \frac{\varepsilon_\sigma}{\xi_1}}{\sigma_1^2 + \frac{\varepsilon_\sigma L_0}{\rho \xi_1}} \right) \\
    &\hspace{1.5cm} \leq \mathbb{P} \left( \theta_{i,t} \geq \mu_1 - \frac{\varepsilon_\mu L_0}{2} \right) + \mathbb{P} \left(\frac{\varepsilon_\mu\rho}{\tau_{i,t}} \geq - \frac{\varepsilon_\mu L_0}{2} \right) + \mathbb{P} \left(\tau_{i,t} \geq \frac{1 - \frac{\varepsilon_\sigma}{\xi_1}}{\sigma_1^2 + \frac{\varepsilon_\sigma L_0}{\rho \xi_1}} \right).
\end{align*}
Given the empirical sufficient statistics $\tbar{\mu}_{i,q}$ and $\tbar{\sigma}_{i,q}^2$, we bound these posterior probabilities using the standard Gaussian tail bound and the Gamma concentration bound (from Lemma S-2 in~\cite{zhu2020thompson}) :
\begin{align}
    H_{i,q} \leq \exp{\left(-\frac{q}{2}\left(\mu_1 - \mu - \frac{\varepsilon_\mu L_0}{2} \right)^2\right)} + \exp{\left(-q h\left(\frac{\sigma^2\left(1 - \frac{\varepsilon_\sigma}{\xi_1}\right)}{\sigma_1^2 + \frac{\varepsilon_\sigma L_0}{\rho \xi_1}}\right)\right)}.
\label{eq:posterior_bounds}
\end{align}
where $h(x) = (x-1-\log{x})/2$. The next step if inverting the bounds to find empirical thresholds. For the event $\{H_{i,q} > 1/n\}$ to occur, at least one of the two exponential bounds in~\eqref{eq:posterior_bounds} must exceed $1/(2n)$. This logical implication allows us to map the failure of the posterior concentration directly onto the deviations of the empirical estimates. By taking the natural logarithm of both terms and isolating the empirical moments, we establish the following deterministic thresholds :
\begin{align*}
    \left[H_{i,q} > \frac{1}{n}\right] \implies \left[ \tbar{\mu}_{i,q} \geq \mu_1 - \frac{\varepsilon_\mu L_0}{2} - \sqrt{\frac{2\log{(2n)}}{q}} \right] \bigcup \left[ \frac{\tbar{\sigma}_{i,q}^2\left(1 - \frac{\varepsilon_\sigma}{\xi_1}\right)}{\sigma_1^2 + \frac{\varepsilon_\sigma L_0}{\rho \xi_1}} \leq h_{+}^{-1}\left(\frac{\log{(2n)}}{q}\right) \right],
\end{align*}
where $h_{+}^{-1}(y) = \max\{x : h(x)=y\}$. Note: While $h(\cdot)$ also has a lower inverse branch $h_{-}^{-1}$, bounded distributions ensure the upper branch dominates the failure rate. The next step is to find the frequentist concentration of empirical estimates. We define the sufficient exploration threshold $u$ as the minimum number of pulls required for the true parameters to satisfy these inequalities strictly,
\begin{align*}
    u = \max \left[\frac{2\log{(2n)}}{\big(\Lambda_{1,i} - \varepsilon_\mu L_0\big)^2} , \frac{\log{(2n)}}{h\left(\frac{\sigma_i^2\left(1 - \frac{\varepsilon_\sigma}{\xi_1}\right)}{\sigma_1^2}\right)}\right].
\end{align*}
For all $q \geq u$, the deterministic thresholds fall strictly outside the true moments $\mu_i$ and $\sigma_i^2$. We can thus apply frequentist concentration inequalities (Chernoff bounds) to the empirical statistics. Centering the mean inequality around $\mu_i$ yields $\tbar{\mu}_{i,q} - \mu_i \geq \Lambda_{1,i} - \frac{\varepsilon_\mu L_0}{2} - \sqrt{\frac{2\log(2n)}{q}}$. Bounding the probabilities of these empirical deviations gives
\begin{align*}
    \hspace*{-0.5cm} \mathbb{P} \left(H_{i,q}>\frac{1}{n}\right) &\leq \mathbb{P}\left(\tbar{\mu}_{i,q}+\sqrt{\frac{2 \log (2 n)}{q}} \geq \mu_{1} - \frac{\varepsilon_\mu L_0}{2}\right) + \mathbb{P} \left(h_{-}^{-1}\left(\frac{\log (2 n)}{q}\right) \leq \frac{\tbar{\sigma}_{i,q}^2\left(1 - \frac{\varepsilon_\sigma}{\xi_1}\right)}{\sigma_1^2 + \frac{\varepsilon_\sigma L_0}{\rho \xi_1}} \leq h_{+}^{-1}\left(\frac{\log (2 n)}{q}\right)\right) \\
    &\leq \mathbb{P}\left(\tbar{\mu}_{i, q}-\mu_{i} \geq \Gamma_{i} - \frac{\varepsilon_\mu L_0}{2} - \sqrt{\frac{2 \log (2 n)}{q}}\right) + \mathbb{P}\left(\tbar{\sigma}_{i,q}^{2} \leq \left(\frac{\sigma_1^2 + \frac{\varepsilon_\sigma L_0}{\rho \xi_1}}{1 - \frac{\varepsilon_\sigma}{\xi_1}}\right) h_{+}^{-1}\left(\frac{\log (2 n)}{q}\right)\right) \\
    &\hspace*{-1cm}\leq \exp{\left(-\frac{q}{2\sigma_i^2} \left(\Lambda_{1,i} - \frac{\varepsilon_\mu L_0}{2} - \sqrt{\frac{2 \log (2 n)}{q}}\right)^2\right)} + \exp{\left(-\frac{q-1}{4\sigma_i^4} \left(\left(\frac{\sigma_1^2 + \frac{\varepsilon_\sigma L_0}{\rho \xi_1}}{1 - \frac{\varepsilon_\sigma}{\xi_1}}\right) h_{+}^{-1}\left(\frac{\log (2 n)}{q}\right) - \sigma_i^2\right)^2\right)} \\
    &\hspace*{-1cm}\leq \exp{\left(-\frac{q \varepsilon_{\mu}^2 L_0^2}{8\sigma_i^2} \right)} + \exp{\left(-\frac{q-1}{4\sigma_1^4} \left(\frac{\varepsilon_\sigma^2 L_0^2}{\rho^2\xi_1^2}\right)\right)}.
\end{align*}
Next, we proceed with the summation over the horizon. We bound the expected exploitation penalty by summing the indicator over all $q$ up to $n$. We aggressively bound the first $u$ terms by $1$, and apply the bounds derived above for $q \geq u$ :
\begin{align*}
\sum_{q=1}^n \mathbb{P} \left(H_{i,q} > \frac{1}{n}\right) &\leq u + \sum_{q=\lceil u \rceil}^n \left[ \exp{\left(-\frac{q \varepsilon_{\mu}^2 L_0^2}{8\sigma_i^2} \right)} + \exp{\left(-\frac{q-1}{4\sigma_1^4} \left(\frac{\varepsilon_\sigma^2 L_0^2}{\rho^2\xi_1^2}\right)\right)} \right].
\end{align*}
Applying the geometric series bound $\sum e^{-ax} \leq 1/a$ to the decaying exponentials completes the proof:
\begin{align*}
    \sum_{q=1}^n \mathbb{P} \left(H_{i,q} > \frac{1}{n}\right) \leq 1 + \max \left[\frac{2\log{(2n)}}{\big(\Lambda_{1,i} - \varepsilon_\mu L_0\big)^2} , \frac{\log{(2n)}}{h\left(\frac{\sigma_i^2\left(1 - \frac{\varepsilon_\sigma}{\xi_1}\right)}{\sigma_1^2}\right)}\right] + \frac{8 \sigma_i^2}{\varepsilon_{\mu}^2 L_0^2} + \frac{4 \sigma_1^4 \rho^2\xi_1^2}{\varepsilon_\sigma^2 L_0^2}.
\end{align*}
This completes the proof of Lemma \ref{lem:B2}.
\end{proof}

\subsection{Conclusion of the Proof of Theorem \ref{th:pull_bound}}
With the expected exploration penalty strictly bounded in Lemma~\ref{lem:B1} and the expected exploitation penalty strictly bounded in Lemma~\ref{lem:B2}, we can now synthesize the global bound on the expected number of suboptimal pulls. Substituting the results of Lemma~\ref{lem:B1} and Lemma~\ref{lem:B2} back into the decoupling framework of Equation~\eqref{eq:app_lemma_1}, we obtain:
\begin{align*}
    \mathbb{E} \left[s_{i,n}\right] &\leq \frac{C_1^\prime}{\varepsilon_\mu^3} + \frac{C_2^\prime }{\varepsilon_\sigma^2}\left(\sigma_1^2 + \frac{\varepsilon_\sigma L_0}{\rho \xi_1}\right)^2 + \left(C_2 + \frac{C_3}{\varepsilon_\mu} \right) \frac{C_3^\prime}{\varepsilon_\mu} + \frac{C_4^\prime}{\varepsilon_\sigma} + \frac{C_4}{\varepsilon_\sigma} \left( C_2 + \frac{C_3}{\varepsilon_\mu} \right) \left(\frac{C_5^\prime}{\varepsilon_\mu + C_6^\prime \varepsilon_\sigma} \right) + \\
    &\hspace*{3cm} 1 + \max \left\{\frac{2\log{(2n)}}{\big(\Lambda_{1,i} - \varepsilon_\mu L_0\big)^2} , \frac{\log{(2n)}}{h\left(\frac{\sigma_i^2\left(1 - \frac{\varepsilon_\sigma}{\xi_1}\right)}{\sigma_1^2}\right)}\right\} + \frac{8 \sigma_i^2}{\varepsilon_{\mu}^2 L_0^2} + \frac{4 \sigma_1^4 \rho^2\xi_1^2}{\varepsilon_\sigma^2 L_0^2} \\
    &\leq 1 + \max \left\{\frac{2\log{(2n)}}{\big(\Lambda_{1,i} - \varepsilon_\mu L_0\big)^2} , \frac{\log{(2n)}}{h\left(\frac{\sigma_i^2\left(1 - \frac{\varepsilon_\sigma}{\xi_1}\right)}{\sigma_1^2}\right)}\right\} + \frac{X_2}{\varepsilon^3} + \frac{X_3}{\varepsilon^2} + \frac{X_4}{\varepsilon} + X_5. 
\end{align*}
Because $\varepsilon_\mu$ and $\varepsilon_\sigma$ are strictly proportional to the total error margin $\varepsilon$, we can absorb the highest-order polynomial dependencies from both lemmas into universal distribution-dependent constants $X_2, X_3, X_4,$ and $X_5$. Rearranging the terms yields:
\begin{align*}
\mathbb{E} [s_{i,n}] \leq 1 + \max \left[\frac{2\log{(2n)}}{\big(\Lambda_{1,i} - \varepsilon_\mu L_0\big)^2} , \frac{\log{(2n)}}{h\left(\frac{\sigma_i^2\left(1 - \frac{\varepsilon_\sigma}{\xi_1}\right)}{\sigma_1^2}\right)}\right] + \frac{X_2}{\varepsilon^3} + \frac{X_3}{\varepsilon^2} + \frac{X_4}{\varepsilon} + X_5.
\end{align*}
This perfectly establishes the bound presented in Equation \eqref{eq:res_1}.

To complete the connection to the final regret bound in Theorem \ref{th:regret_bound}, we dynamically tune the error margin as a decaying function of the horizon, setting $\varepsilon = (\log n)^{-1/4}$. Substituting this into the polynomial constants gives:
\begin{align*}
\frac{X_2}{\varepsilon^3} + \frac{X_3}{\varepsilon^2} + \frac{X_4}{\varepsilon} = X_2 (\log n)^{\frac{3}{4}} + X_3 (\log n)^{\frac{1}{2}} + X_4 (\log n)^{\frac{1}{4}}.
\end{align*}
Because these terms grow strictly slower than the leading $\mathcal{O}(\log n)$ threshold $u$, the expected number of suboptimal pulls remains logarithmically bounded. Substituting this final expression for $\mathbb{E}[s_{i,n}]$ into the algorithmic regret decomposition from Theorem~\ref{th:th_dec} formally yields the finite-time expected regret of the \texttt{SRTS} policy.

\hfill $\square$

\section{Proof of Theorem~\ref{th:lower_bound}}
\label{pro:th4}
This appendix provides the complete mathematical derivation for the model-specific lower bound established in Theorem~\ref{th:lower_bound}. To construct this proof, we proceed in three logical steps. First, in Appendix~\ref{sub:le4}, we derive a pseudo-regret lower bound (Lemma~\ref{lem:lem_2}), demonstrating through a covariance decomposition that the cumulative regret is fundamentally bottlenecked by the expected number of suboptimal arm pulls. Next, in Appendix~\ref{sub:le5}, we utilize a change-of-measure argument to establish an information-theoretic lower bound on these pulls (Lemma~\ref{lem:lem_3}), proving that any uniformly $\alpha$-consistent policy must sample suboptimal arms at a rate bounded by their \ac{KL} divergence. Finally, in Appendix~\ref{sub:le6}, we synthesize these auxiliary results, formally evaluating the limits across both asymptotic and finite-time regimes to conclude the proof of Theorem~\ref{th:lower_bound}.

\subsection{Regret Lower Bound via Suboptimal Pulls}
\label{sub:le4}
\textbf{Lemma \ref{lem:lem_2} (Restated).}
The expected regret of a policy $\pi$ over $n$ rounds for \ac{SR} is given as
\begin{align*}
    \mathbb{E}\left[\mathcal{R}_n(\pi)\right] \geq \sum_{i=2}^K \mathbb{P}\left( s_{i,n} \geq \frac{G_1 \log n}{I(f_i,f_1)}\right) \frac{G_1 \log n}{I(f_i,f_1)} \Delta_i \left( \frac{L_0 + \rho \sigma_i^2}{L_0 + \rho \sigma_{\max}^2} \right) - Y_1,
\end{align*}
where $G_1$ and $Y_1$ are constants.

\begin{proof}
We begin by expressing the expected cumulative regret using the difference between the optimal \ac{SR} and the algorithmic \ac{SR},
\begin{align}
    &\mathbb{E} \left[\mathcal{R}_n(\pi)\right] = n\mathbb{E}\left[\xi_1 - \tbar{\xi}_n(\pi)\right] \nonumber\\
    &= n\mathbb{E} \left[\frac{1}{n} \sum_{i=1}^K s_{i,n} \frac{\mu_1}{L_0 + \rho\sigma^2_1} - \frac{ \frac{1}{n} \sum_{i=1}^K s_{i,n} \tbar{\mu}_{i,s_{i,n}}}{L_0 + \rho\left(\frac{1}{n} \sum_{i=1}^K s_{i,n} \tbar{\sigma}_{i,s_{i,n}}^2 + \frac{1}{2n^2} \sum_{i=1}^K \sum_{j\neq i} s_{i,n} s_{j,n} \left(\tbar{\mu}_{i,s_{i,n}} - \tbar{\mu}_{j,s_{j,n}} \right)^2\right)} \right] \nonumber\\
    &= n \left(\frac{1}{n} \sum_{i=1}^K \mathbb{E}[s_{i,n}] \frac{\mu_1}{L_0 + \rho\sigma^2_1} - \mathbb{E} \left[\frac{ \frac{1}{n} \sum_{i=1}^K s_{i,n} \tbar{\mu}_{i,s_{i,n}}}{L_0 + \underbrace{\rho\left(\frac{1}{n} \sum_{i=1}^K s_{i,n} \tbar{\sigma}_{i,s_{i,n}}^2 + \frac{1}{2n^2} \sum_{i=1}^K \sum_{j\neq i} s_{i,n} s_{j,n} \left(\tbar{\mu}_{i,s_{i,n}} - \tbar{\mu}_{j,s_{j,n}} \right)^2\right)}_{D} } \right]\right).
    \label{eq:eq_29}
\end{align}
To bound the expectation of the fractional term, we apply the covariance decomposition. Applying the Cauchy-Schwarz inequality to the covariance penalty yields,
\begin{align}
    &\mathbb{E} \left[\frac{\frac{1}{n} \sum_{i=1}^K s_{i,n} \tbar{\mu}_{i,s_{i,n}}}{L_0 + D} \right] =  \mathbb{E} \left[\frac{1}{n} \sum_{i=1}^K s_{i,n} \tbar{\mu}_{i,s_{i,n}}\right] \mathbb{E} \left[\frac{1}{L_0 + D}\right] + \text{Cov} \left(\frac{1}{n} \sum_{i=1}^K s_{i,n} \tbar{\mu}_{i,s_{i,n}}, \frac{1}{L_0 + D}\right) \nonumber\\
    &\leq \mathbb{E} \left[\frac{1}{n} \sum_{i=1}^K s_{i,n} \tbar{\mu}_{i,s_{i,n}}\right] \mathbb{E} \left[\frac{1}{L_0 + D}\right] + \sqrt{\mathbb{V}\left[\frac{1}{n} \sum_{i=1}^K s_{i,n} \tbar{\mu}_{i,s_{i,n}}\right] \mathbb{V} \left[\frac{1}{L_0 + D}\right]} \nonumber\\
    &\overset{(a)}{\leq} \left[\frac{1}{n} \sum_{i=1}^K s_{i,n} \tbar{\mu}_{i,s_{i,n}}\right] \mathbb{E} \left[\frac{1}{L_0 + D}\right] + \sqrt{\frac{Q_1}{n} \times \frac{Q_2}{n}}.
    \label{eq:eq_2}
\end{align}
Step (a) follows from Lemma~\ref{le:le_2} and~\ref{le:le_3}. Next, we must upper-bound the expectation $\mathbb{E}\left[\frac{1}{L_0 + D}\right]$. We apply the distribution-free reciprocal expectation bound from Wooff~\cite{wooff1985bounds},
\begin{align*}
    \mathbb{E}\left[\frac{1}{L_0 + D}\right] \leq \gamma \left( 1 + \frac{\mathbb{V}[D]}{\mathbb{E}[D] L_0} \right).
\end{align*} 
where the scaling factor is defined as $\gamma = \frac{\mathbb{E}[D]}{\mathbb{V}[D] + \mathbb{E}[D](\mathbb{E}[D] + L_0)}$. Because variances are strictly non-negative i.e., $\mathbb{V}[D] \geq 0$, we trivially extract the strict upper bound $\gamma \leq \frac{\mathbb{E}[D]}{\mathbb{E}[D](\mathbb{E}[D] + L_0)} = \frac{1}{\mathbb{E}[L_0 + D]}$. Substituting this, along with the bound $\mathbb{V}[D] \leq \frac{\rho^2 Q_2}{n}$ from Lemma~\ref{le:le_3}, and $\mathbb{E}[L_0 + D] \geq L_0$ isolates the finite-time penalty as
\begin{align}
    \mathbb{E}\left[\frac{1}{L_0 + D}\right] \leq \frac{1}{\mathbb{E}[L_0 + D]} + \frac{\mathbb{V}[D]}{\mathbb{E}[L_0 + D] \mathbb{E}[D] L_0} \leq \frac{1}{\mathbb{E}[L_0 + D]} + \frac{1}{n} \left( \frac{\rho^2 Q_2}{\mathbb{E}[D] L_0^2} \right)
    \label{eq:eq_30}
\end{align}
From~\eqref{eq:eq_29}, $D$ is composed of the empirical arm variance $V = \frac{1}{n} \sum_{i=1}^K s_{i,n} \tbar{\sigma}_{i,s_{i,n}}^2$ and the strictly positive cross-arm switching penalty $W = \frac{1}{2n^2} \sum_{i=1}^K \sum_{j\neq i} s_{i,n} s_{j,n} \left(\tbar{\mu}_{i,s_{i,n}} - \tbar{\mu}_{j,s_{j,n}} \right)^2$. We show 
\begin{align}
    \mathbb{E}[D] = \rho \mathbb{E}[V + W] \geq \rho \mathbb{E}[V] = \frac{\rho}{n} \sum_{i=1}^K \sigma_i^2 \mathbb{E}[s_{i,n}].
    \label{eq:eq_31}
\end{align}
Let $\sigma_{\min}^2 = \min_{i \in [K]} \sigma_i^2$ denote the minimum arm variance. We can unconditionally bound the sum:
\begin{align*}
    \mathbb{E}[D] \geq \frac{\rho \sigma_{\min}^2}{n} \sum_{i=1}^K \mathbb{E}[s_{i,n}] = \rho \sigma_{\min}^2.
\end{align*}
Substituting $\frac{1}{\mathbb{E}[D]} \leq \frac{1}{\rho \sigma_{\min}^2}$ into~\eqref{eq:eq_30} our bound yields,
\begin{align*}
    \mathbb{E}\left[\frac{1}{L_0 + D}\right] \leq \frac{1}{\mathbb{E}[L_0 + D]} + \frac{1}{n} \left( \frac{\rho Q_2}{L_0^2 \sigma_{\min}^2} \right).
\end{align*}
Substituting both the covariance bound and the exact expectation bound back into~\eqref{eq:eq_29}, we obtain,
\begin{align*}
    \mathbb{E} \left[\mathcal{R}_n (\pi)\right] &\geq n \left(\frac{1}{n} \sum_{i=1}^K \mathbb{E}[s_{i,n}] \xi_1 - \frac{1}{n} \sum_{i=1}^K \mu_i \mathbb{E}[s_{i,n}] \left(\frac{1}{\mathbb{E}[L_0 + D]} + \frac{1}{n} \left( \frac{\rho Q_2}{L_0^2 \sigma_{\min}^2} \right)\right) - \sqrt{\frac{Q_1}{n} \times \frac{Q_2}{n}} \right) \\
    &= n \left(\frac{1}{n} \sum_{i=1}^K \mathbb{E}[s_{i,n}] \xi_1 - \frac{1}{n} \sum_{i=1}^K \mu_i \mathbb{E}[s_{i,n}] \frac{1}{\mathbb{E}[L_0 + D]} - \frac{1}{n^2} \sum_{i=1}^K \frac{\mu_i \rho Q_2}{L_0^2 \sigma_{\min}^2} \mathbb{E}[s_{i,n}] - \sqrt{\frac{Q_1}{n} \times \frac{Q_2}{n}} \right) \\
    &\overset{(a)}{\geq} n \left(\frac{1}{n} \sum_{i=1}^K \mathbb{E}[s_{i,n}] \xi_1 - \frac{1}{n} \sum_{i=1}^K \mu_i \mathbb{E}[s_{i,n}] \frac{1}{\mathbb{E}[L_0 + D]} - \frac{\mu_{\max} \rho Q_2}{n L_0^2 \sigma_{\min}^2} - \frac{\sqrt{Q_1 Q_2}}{n} \right) \\
    &\overset{(b)}{\geq} n\left(\frac{1}{n} \sum_{i=1}^K \mathbb{E}[s_{i,n}] \xi_1 - \frac{\frac{1}{n} \sum_{i=1}^K \mu_i \mathbb{E}[s_{i,n}]}{L_0 + \frac{\rho}{n} \sum_{i=1}^K \sigma_i^2 \mathbb{E}[s_{i,n}]} - \frac{1}{n}\underbrace{\left(\frac{\mu_{\max} \rho Q_2}{ L_0^2 \sigma_{\min}^2} + \sqrt{Q_1 Q_2} \right)}_{Y_1} \right).
\end{align*}
Step (a) follows by replacing $\mu_i$ with $\mu_{\max}$ where $\mu_{\max} = \max_{i \in [K]} \mu_i$. And step (b) follows from~\eqref{eq:eq_31}. Using the exact sub-optimality gap identity $\Delta_i(L_0 + \rho \sigma_i^2) = \xi_1(L_0 + \rho \sigma_i^2) - \mu_i$, we substitute for $\mu_i$ directly in the numerator and get
\begin{align*}
    \mathbb{E} \left[\mathcal{R}_n (\pi)\right] &\geq n \left(\frac{1}{n} \sum_{i=1}^K \mathbb{E}[s_{i,n}] \xi_1 - \frac{\frac{1}{n} \sum_{i=1}^K \mu_i \mathbb{E}[s_{i,n}]}{L_0 + \frac{\rho}{n} \sum_{i=1}^K \sigma_i^2 \mathbb{E}[s_{i,n}]} \right) - Y_1 \\
    &= n \left(\frac{1}{n} \sum_{i=1}^K \mathbb{E}[s_{i,n}] \xi_1 - \frac{\frac{1}{n} \sum_{i=1}^K \Big[ \xi_1(L_0 + \rho \sigma_i^2) - \Delta_i(L_0 + \rho \sigma_i^2) \Big] \mathbb{E}[s_{i,n}]}{L_0 + \frac{\rho}{n} \sum_{i=1}^K \sigma_i^2 \mathbb{E}[s_{i,n}]} \right) - Y_1  \\
    &= n \left(\frac{1}{n} \sum_{i=1}^K \mathbb{E}[s_{i,n}] \xi_1 - \frac{\frac{1}{n} \xi_1 L_0 \sum_{i=1}^K \mathbb{E}[s_{i,n}] + \frac{1}{n} \xi_1 \rho \sum_{i=1}^K \sigma_i^2 \mathbb{E}[s_{i,n}] - \frac{1}{n} \sum_{i=1}^K \mathbb{E}[s_{i,n}] \Delta_i (L_0 + \rho \sigma_i^2)}{L_0 + \frac{\rho}{n} \sum_{i=1}^K \sigma_i^2 \mathbb{E}[s_{i,n}]} \right) - Y_1 \\
    &= n \left(\frac{1}{n} \sum_{i=1}^K \mathbb{E}[s_{i,n}] \xi_1 - \frac{\xi_1 \left( L_0 + \frac{\rho}{n} \sum_{i=1}^K \sigma_i^2 \mathbb{E}[s_{i,n}] \right) - \frac{1}{n} \sum_{i=1}^K \mathbb{E}[s_{i,n}] \Delta_i (L_0 + \rho \sigma_i^2)}{L_0 + \frac{\rho}{n} \sum_{i=1}^K \sigma_i^2 \mathbb{E}[s_{i,n}]} \right) - Y_1 \\
    &= n \xi_1 - n \xi_1 + \frac{\sum_{i=1}^K \mathbb{E}[s_{i,n}] \Delta_i (L_0 + \rho \sigma_i^2)}{L_0 + \frac{\rho}{n} \sum_{i=1}^K \sigma_i^2 \mathbb{E}[s_{i,n}]}  - Y_1 \\
    &\geq \sum_{i=1}^K \mathbb{E}[s_{i,n}] \Delta_i \left(\frac{L_0 + \rho \sigma_i^2}{L_0 + \rho \sigma_{\max}^2}\right)  - Y_1.
\end{align*}
Applying Markov's inequality to the expected pulls $\mathbb{E}[s_{i,n}]$ establishes the final lower bound
\begin{align}
    \mathbb{E}\left[\mathcal{R}_n(\pi)\right] \geq \sum_{i=2}^K \mathbb{P}\left( s_{i,n} \geq \frac{G_1 \log n}{I(f_i,f_1)}\right) \frac{G_1 \log n}{I(f_i,f_1)} \Delta_i \left( \frac{L_0 + \rho \sigma_i^2}{L_0 + \rho \sigma_{\max}^2} \right) - Y_1.
\end{align}
\end{proof}

\subsection{Asymptotic Pull Count via Change-of-Measure}
\label{sub:le5}
\textbf{Lemma \ref{lem:lem_3} (Restated).}
For any $\alpha$-consistent policy $\pi$ (Assumption~\ref{ass:continuity}), and for any constant $G_1 < 1-\alpha$, the probability of playing a suboptimal arm $i \neq 1$ is
\begin{align*}
    \lim_{n \rightarrow \infty} \mathbb{P}_\mathcal{F}\left[s_{i,n} \geq  \frac{G_1 \log n}{I(f_i,f_{1})}\right]=1.
\end{align*}
Furthermore, with Assumption~\ref{ass:stability}, there exists $n_o \in \mathbb{N}$ such that
\begin{align*}
    \mathbb{P}_\mathcal{F}\left[s_i(n) \geq  \frac{G_1 \log n}{I(f_i,f_{*})}\right]\geq G_2, \quad \textit{for all}\quad  n > n_0,
\end{align*}
where constant $0 < G_2 < 1$ is independent of $n$ and $\mathcal{F}$.

\begin{proof}
Let $\mathcal{F} = (f_1, \dots, f_K)$ denote the true environment where arm $1$ is strictly optimal. We construct a perturbed distribution model $\mathcal{F}^{i} = (f_1, \dots, \tilde{f}_i, \dots, f_K)$ by modifying only the distribution of arm $i$ such that it becomes the unique optimal arm in the new environment. By Assumption \ref{ass:continuity} (Information Continuity), we can choose $\tilde{f}_i$ such that the \ac{SR} gap is arbitrarily small and the \ac{KL} divergence satisfies $I(f_i, \tilde{f}_i) \approx I(f_i, f_{1})$ .We define the log-likelihood ratio of the observations from arm $i$ under the two environments as
\begin{align*}
    \gamma = \sum_{s=1}^{s_{i,n}} \log \frac{f_i(X_i(t_i(s)))}{\tilde{f}_i(X_i(t_i(s)))}.
\end{align*}
where $t_i(s_{i,n})$ is the time step when the $s_{i,n}$-th observation of arm $i$ occurred. We show that it is unlikely to have $s_{i,n} < \frac{G_1 \log n}{I (f_i ,\tilde{f}_i)}$ under two different scenarios for $\gamma$, i.e, $\gamma > G_3 \log n$ and  $\gamma < G_3 \log n$.

\textbf{Case 1: The Likelihood Ratio is Large i.e. $\gamma > G_3 \log n$ for a constant $G_3>G_1$.} \\
We bound the failure probability when the empirical evidence strongly favors the original environment $\mathcal{F}$,
\begin{align*}
    \mathbb{P}_\mathcal{F}\left[s_{i,n} <  \frac{G_1 \log n}{I(f_i,\tilde{f}_i)}, \gamma > G_3 \log n\right] &=  \mathbb{P}_\mathcal{F}\left[s_{i,n} <  \frac{G_1 \log n}{I(f_i,\tilde{f}_i)}, \sum_{\tau=1}^{s_{i,n}} \log \frac{f_i (X_k(\tau))}{\tilde{f}_i(X_k(\tau))} > G_3 \log n\right]\\
    &\leq  \mathbb{P}_\mathcal{F} \left[ \max_{t\leq \frac{G_1 \log n}{{I(f_i,\tilde{f}_i)}}}\sum_{\tau=1}^{t} \log\frac{f_i (X_k(\tau))}{\tilde{f}_i(X_k(\tau))}> \frac{G_3 \log n}{I(f_i,\tilde{f}_i)} I(f_i,\tilde{f}_i)\right] \\ 
    &\leq  \mathbb{P}_\mathcal{F} \left[ \max_{t\leq \frac{G_1 \log n}{{I(f_i,\tilde{f}_i)}}} \frac{1}{t}\sum_{\tau=1}^{t} \log\frac{f_i (X_k(\tau))}{\tilde{f}_i(X_k(\tau))}> \frac{G_3}{G_1}{I(f_i,\tilde{f}_i)} \right]. 
\end{align*}
By the strong law of large numbers, the empirical mean of the log-likelihood ratio converges almost surely to the true \ac{KL} divergence $I(f_i, \tilde{f}_i)$ as $n \to \infty$. Because $G_3 > G_1$, the ratio $G_3/G_1 > 1$, making the threshold strictly greater than the mean. Consequently, this probability vanishes asymptotically. In conclusion, for  $\gamma > G_3 \log n$, by strong law of large numbers, we have,
\begin{align}
    \mathbb{P}_\mathcal{F}\left[s_{i,n} <  \frac{G_1 \log n}{I(f_i,\tilde{f}_i)}, \gamma > G_3 \log n\right] \rightarrow 0, \,\text{as}\, n \rightarrow \infty.
    \label{eq:case1}
\end{align}
Furthermore, under Assumption \ref{ass:stability} (Distributional Stability), the log-likelihood possesses sub-Gaussian tails, allowing us to apply the Chernoff bound to strictly bound the finite-time probability $\mathbb{P}_\mathcal{F}\left[s{i,n} <  \frac{G_1 \log n}{I(f_i,\tilde{f}_i)}, \gamma > G_3 \log n\right] \leq \frac{G_1 \log n}{I(f_i,\tilde{f}_i)} n^{-a I(f_i,\tilde{f}_i ) \frac{(G_3-G_1) ^2}{G_1}}$. Here $a$ is the Chernoff bound constant. This result, is obtained from Eq.$(52)$ of \cite{vakili2015mean}) and is given below,
\begin{align}
     \mathbb{P}_\mathcal{F}\left[s_{i,n} <  \frac{G_1 \log n}{I(f_i,\tilde{f}_i)}, \gamma > G_3 \log n \right] &=  \mathbb{P}_\mathcal{F}\left[s_{i,n} <  \frac{G_1 \log n}{I(f_i,\tilde{f}_i)}, \sum_{s=1}^{s_{i,n}} \log \frac{f_i(X_i(s))}{\tilde{f}_i(X_i(s))} > G_3 \log n\right]\nonumber\\
     &\leq \mathbb{P}_\mathcal{F} \left[\max_{t < \frac{G_1 \log n}{I(f_i,\tilde{f}_i)}}, \sum_{s=1}^{t} \log \frac{f_i(X_i(s))}{\tilde{f}_i(X_i(s))} > G_3 \log n\right]\nonumber\\
     &\leq \sum_{t=1}^{\frac{G_1 \log n}{I(f_i,\tilde{f}_i)}}\mathbb{P}_\mathcal{F} \left[ \frac{1}{t} \sum_{s=1}^{t} \log \frac{f_i(X_i(s))}{\tilde{f}_i(X_i(s))} - \frac{1}{t} G_1 \log {n} > \frac{1}{t} G_3 \log {n} - \frac{1}{t} G_1 \log {n}\right] \nonumber\\
     &\leq \sum_{t=1}^{ \frac{G_1 \log n}{I(f_i,\tilde{f}_i)}}\mathbb{P}_\mathcal{F} \left[ \frac{1}{t} \sum_{s=1}^{t} \log \frac{f_i(X_i(s))}{\tilde{f}_i(X_i(s))} - I(f_i,\tilde{f}_i) > \frac{1}{t} G_3 \log {n} - \frac{1}{t} G_1 \log {n}\right]\nonumber\\
     &\overset{(a)}{\leq} \sum_{t=1}^{ \frac{G_1 \log n}{I(f_i,\tilde{f}_i)}} \exp (\frac{-a_1(G_3-G_1)^2) \log^2 n }{t}) \nonumber\\ 
     &\overset{(b)}{\leq} \frac{G_1 \log n}{I(f_i,\tilde{f}_i)} n^{-a I(f_i,\tilde{f}_i )
      \frac{(G_3-G_1) ^2}{G_1}}.
      \label{eq:B}
\end{align}
Step (a) holds because $I(f_i,\tilde{f}_i) >  \frac{1}{t} G_1 \log {n}$ and step (b) holds according to the Chernoff bound.

\textbf{Case 2: The Likelihood Ratio is Small ($\gamma \leq G_3 \log n$)} \\
When the likelihood ratio is small, the sample path under $\mathcal{F}$ is statistically indistinguishable from a path under $\mathcal{F}^i$. We formalize this by applying the Radon-Nikodym theorem to change the probability measure from $\mathbb{P}_\mathcal{F}$ to $\mathbb{P}_{\mathcal{F}^i}$,
\begin{align}
     &\mathbb{P}_\mathcal{F}\left[s_{i,n} <  \frac{G_1 \log n}{I(f_i,\tilde{f}_i)}, \gamma \leq G_3 \log n \right] = \mathbb{E}_\mathcal{F}\left[\mathbb{I}\left\{s_{i,n} <  \frac{G_1 \log n}{I\left(f_i,\tilde{f}_i\right)}, \gamma \leq G_3 \log n \right\}\right]\nonumber\\
     &=  \mathbb{E}_{\mathcal{F}^i}\left[\mathbb{I}\left\{s_{i,n} <  \frac{G_1 \log n}{I\left(f_i,\tilde{f}_i\right)}, \gamma \leq G_3 \log n \right\} \times e^\gamma \right] \nonumber \\
     &\overset{(a)}{\leq} n ^{G_3} \mathbb{E}_{\mathcal{F}^i} \left[\mathbb{I}\left\{s_{i,n} <  \frac{G_1 \log n}{I(f_i,\tilde{f}_i)}, \gamma \leq G_3 \log n \right\} \right] \leq  n ^{G_3} \mathbb{P}_{\mathcal{F}^i}\left[{s_{i,n} < \frac{G_1 \log n}{I\left(f_i,\tilde{f}_i\right)}} \right].
\end{align}
Step (a) follows from substituting $\gamma \leq G_3 \log{n}$ in $e^\gamma$. Under the perturbed environment $\mathcal{F}^i$, arm $i$ is optimal. By the definition of $\alpha$-consistency, the expected number of times the policy plays any other arm (i.e., $n - s_{i,n}$) is strictly bounded by $K n^\alpha$. Applying Markov's inequality under $\mathcal{F}^i$ yields:
\begin{align}
     & n ^{G_3} \mathbb{P}_{\mathcal{F}^i}\left[{s_{i,n} < \frac{G_1 \log n}{I\left(f_i,\tilde{f}_i\right)}} \right] \leq \frac { n ^{G_3} \mathbb{E}_{\mathcal{F}^i}\left[n-s_{i,n}\right]}{n- \frac{G_1 \log n}{I(f_i,\tilde{f}_i)}} \leq \frac {Kn^{G_3+\alpha}}{{n- \frac{G_1 \log n}{I(f_i,\tilde{f}_i)}}}.
     \label{eq:case2}
\end{align}
Case $1$ uses concentration of the empirical log-likelihood ratio around $I(f_i,\tilde f_i)$; Case $2$ transports probabilities across models using the likelihood ratio identity and then applies $\alpha$-consistency to upper bound $\mathbb{E}_{\mathcal{F}^i} \left[n - s_{i,n}\right]$. Together, these forces $s_{i,n} \gtrsim \frac{G_1}{I(f_i,f_*}\log n$ with constant probability precisely the ingredient needed in the regret lower bound. From \eqref{eq:case1}, \eqref{eq:case2} and the fact that $ {I(f_i,f_{*})}- {I(f_i,\tilde{f}_i)}$ can be arbitrarily small, for $G_3<1-\alpha$, we have,
\begin{align*}
   \mathbb{P}_\mathcal{F}\left[s_{i,n} <  \frac{G_1 \log n}{I(f_i,f_{*})}\right] \rightarrow 0 \,  \text{as}\, n \rightarrow \infty .
\end{align*}
Because we deliberately chose the constants such that $G_3 < 1 - \alpha$, the exponent $G_3 + \alpha < 1$. Thus, the numerator grows strictly slower than the linear denominator, driving this probability to $0$ as $n \to \infty$. Equivalently
\begin{align*}
    \mathbb{P}_\mathcal{F}\left[s_{i,n} \geq \frac{G_1 \log n}{I(f_i,f_{*})}\right] \rightarrow 1 \,  \text{as}\, n \rightarrow \infty .
\end{align*}
Combining \eqref{eq:case1} and \eqref{eq:case2}, we conclude, for $G_3 < 1 -\alpha$, when Assumption 2 is satisfied, 
\begin{align*}
    \mathbb{P}_\mathcal{F}\left[s_{i,n} <  \frac{G_1 \log n}{I(f_i,f_{*})}\right] &\leq \frac{G_1 \log n}{I(f_i,\tilde{f}_i)} n^{-a I(f_i,\tilde{f}_i )\frac{(G_3-G_1) ^2}{G_1}} + \frac {Kn^{G_3+\alpha}}{{n- \frac{G_1 \log n}{I(f_i,\tilde{f}_i)}}}.
\end{align*}
Thus, there is a  $n_0$ such that for $n>n_0$, $\mathbb{P}_\mathcal{F}\left[s_{i,n} \geq \frac{G_1 \log n}{I(f_i,f_{*})}\right] \geq G_2$,
for some constant $G_2 > 0$ independent of $n$ and $\mathcal{F}$. We emphasize that constant $G_1$ and $G_3$ are chosen to satisfy $G_1 < G_3 < 1-\alpha$.
\end{proof}

\subsection{Synthesis of the Lower Bound (Proof of Theorem~\ref{th:lower_bound})}
\label{sub:le6}
\begin{proof}
To finalize the proof of Theorem \ref{th:lower_bound}, we must establish a strict lower bound on the expected number of suboptimal pulls, $\mathbb{E}[s_{i,n}]$, and integrate it into the exact regret decomposition derived in Lemma \ref{lem:lem_2}. 

We begin by applying Markov's inequality. We can bound the expected pulls of any suboptimal arm $i$ by the probability of the corresponding concentration event:
\begin{align}
    \mathbb{E}[s_{i,n}] \geq \mathbb{P}_\mathcal{F}\left[ s_{i,n} \geq \frac{G_1 \log n}{I(f_i,f_1)}\right] \frac{G_1 \log n}{I(f_i,f_1)}.
    \label{eq:markov_pulls}
\end{align}
We analyze this bound under both asymptotic and finite-time regimes based on the distinct probability limits established previously.

\textbf{Asymptotic Regime (Under Assumption \ref{ass:continuity}):} \\
Dividing both sides of \eqref{eq:markov_pulls} by $\log n$ and taking the limit inferior as $n \to \infty$ yields:
\begin{align*}
    \liminf_{n \rightarrow \infty} \frac{\mathbb{E}[s_{i,n}]}{\log n} \geq \lim_{n \rightarrow \infty} \mathbb{P}_\mathcal{F}\left[ s_{i,n} \geq \frac{G_1 \log n}{I(f_i,f_1)}\right] \frac{G_1}{I(f_i,f_1)}.
\end{align*}
Under Assumption \ref{ass:continuity}, the probability of sufficiently sampling a suboptimal arm under a uniformly $\alpha$-consistent policy approaches $1$ asymptotically. Thus, the limit evaluates strictly to:
\begin{align}
    \liminf_{n \rightarrow \infty} \frac{\mathbb{E}[s_{i,n}]}{\log n} \geq \frac{G_1}{I(f_i,f_1)}.
    \label{eq:asymp_pull_bound}
\end{align}
To evaluate the asymptotic regret, we divide the pseudo-regret lower bound from Lemma \ref{lem:lem_2} by $\log n$ and take the limit inferior as $n \to \infty$. Because $Y_1$ is a finite problem-dependent constant, the residual term $Y_1 / \log n$ naturally vanishes. Substituting \eqref{eq:asymp_pull_bound} directly into the remaining expression gives the fundamental asymptotic regret limit:
\begin{align*}
    \liminf_{n \rightarrow \infty} \frac{\mathbb{E}[\mathcal{R}_n(\pi)]}{\log n} \geq \sum_{i=2}^K \frac{G_1}{I(f_i, f_1)} \Delta_i \left( \frac{L_0 + \rho \sigma_i^2}{L_0 + \rho \sigma_{\max}^2} \right).
\end{align*}

\textbf{Finite-Time Regime (Under Assumption \ref{ass:stability}):} \\
When the environment satisfies the Distributional Stability condition, we obtain a stronger, finite-time guarantee: there exists an integer $n_0 \in \mathbb{N}$ such that for all $n > n_0$, the probability is strictly bounded below by a constant $G_2 \in (0, 1)$. Substituting $\mathbb{P}_\mathcal{F}\left[s_{i,n} \geq \frac{G_1 \log n}{I(f_i,f_1)}\right] \geq G_2$ directly into \eqref{eq:markov_pulls} provides a finite-time lower bound on the expected pulls:
\begin{align*}
    \mathbb{E}[s_{i,n}] \geq \frac{G_1 G_2 \log n}{I(f_i,f_1)}, \quad \forall n > n_0.
\end{align*}
Substituting this expected pull bound back into the exact result of Lemma \ref{lem:lem_2} yields the explicit finite-time model-specific regret lower bound:
\begin{align*}
    \mathbb{E}[\mathcal{R}_n(\pi)] \geq \sum_{i=2}^K \frac{G_1 G_2 \log n}{I(f_i, f_1)} \Delta_i \left( \frac{L_0 + \rho \sigma_i^2}{L_0 + \rho \sigma_{\max}^2} \right) - Y_1, \quad \forall n > n_0.
\end{align*}

Both regimes demonstrate that for any uniformly $\alpha$-consistent policy, the cumulative regret under the \ac{SR} objective must grow \textit{at least} logarithmically in $n$. The problem-dependent constants are strictly governed by the \ac{KL} divergences and the exact risk-adjusted sub-optimality gaps, formally establishing the information-theoretic limit for \ac{SR} bandits.
\end{proof}

\section{Auxiliary Technical Lemmas}

\begin{lemma}[Gaussian Left-Tail Bound]
\label{le:lemma_gaussian}
Let $X \sim \mathcal{N}(\mu, 1/q)$ for some $q > 0$. Then, for any threshold $a < \mu$, the lower tail probability satisfies:
\begin{align*}
    \mathbb{P}(X \leq a) \leq \exp\left( -\frac{q}{2} (\mu - a)^2 \right).
\end{align*}
\end{lemma}
\begin{proof}
Applying the standard Chernoff bound technique, for any $\lambda < 0$, Markov's inequality yields:
\begin{align*}
    \mathbb{P}(X \leq a) = \mathbb{P}\left(e^{\lambda X} \geq e^{\lambda a}\right) \leq e^{-\lambda a} \mathbb{E}\left[e^{\lambda X}\right].
\end{align*}
Substituting the moment-generating function of the Gaussian random variable $X$, which is given by $\mathbb{E}[e^{\lambda X}] = \exp\left( \lambda \mu + \frac{\lambda^2}{2q} \right)$, we obtain the parametric upper bound:
\begin{align*}
    \mathbb{P}(X \leq a) \leq \exp\left( \lambda (\mu - a) + \frac{\lambda^2}{2q} \right).
\end{align*}
Minimizing the exponent with respect to $\lambda < 0$ yields the optimal parameter $\lambda^* = -q(\mu - a)$. Substituting $\lambda^*$ into the bound tightly constrains the probability:
\begin{align*}
    \mathbb{P}(X \leq a) \leq \exp\left( -q(\mu - a)^2 + \frac{q^2(\mu - a)^2}{2q} \right) = \exp\left( -\frac{q}{2}(\mu - a)^2 \right),
\end{align*}
which completes the proof.
\end{proof}

\begin{lemma}[Gamma Left-Tail Bound]
\label{le:lemma_gamma}
Let $X \sim \mathrm{Gamma}(\alpha, \beta)$, where $\alpha > 0$ is the shape parameter and $\beta > 0$ is the scale parameter. Then, for any threshold $a < \alpha\beta$, the lower tail probability satisfies:
\begin{align*}
    \mathbb{P}(X \leq a) \leq \exp\left( - \frac{(\alpha\beta - a)^2}{2 \alpha \beta^2} \right).
\end{align*}
\end{lemma}

\begin{proof}
We apply the Chernoff bound using the Laplace transform of the Gamma distribution. For any parameter $\lambda > 0$, Markov's inequality yields:
\begin{align*}
    \mathbb{P}(X \leq a) = \mathbb{P}\left(e^{-\lambda X} \geq e^{-\lambda a}\right) \leq \frac{\mathbb{E}[e^{-\lambda X}]}{e^{-\lambda a}}.
\end{align*}
Substituting the exact Laplace transform for the Gamma distribution, $\mathbb{E}[e^{-\lambda X}] = (1 + \lambda\beta)^{-\alpha}$ (valid for $\lambda > 0$), we obtain the parameterized bound:
\begin{align*}
    \mathbb{P}(X \leq a) \leq e^{\lambda a} (1 + \lambda\beta)^{-\alpha} = \exp\Big( \lambda a - \alpha \ln(1 + \lambda\beta) \Big).
\end{align*}
To optimize this bound, we define the exponent function $f(\lambda) = \lambda a - \alpha \ln(1 + \lambda\beta)$ and minimize it with respect to $\lambda > 0$. Differentiating yields:
\begin{align*}
    f^\prime(\lambda) = a - \frac{\alpha\beta}{1 + \lambda\beta}, \quad f^{\prime\prime}(\lambda) = \frac{\alpha\beta^2}{(1 + \lambda\beta)^2} > 0.
\end{align*}
Setting the first derivative to zero provides the optimal parameter $\lambda^\ast$:
\begin{align*}
    a = \frac{\alpha\beta}{1 + \lambda^\ast \beta} \implies 1 + \lambda^\ast \beta = \frac{\alpha\beta}{a} \implies \lambda^\ast = \frac{1}{\beta} \left( \frac{\alpha\beta}{a} - 1 \right).
\end{align*}
Because $a < \alpha\beta$, it is guaranteed that $\lambda^\ast > 0$. Substituting $\lambda^\ast$ back into $f(\lambda)$ isolates the optimized exponent:
\begin{align*}
    f(\lambda^\ast) &= \lambda^\ast a - \alpha \ln(1 + \lambda^\ast \beta) \\
    &= a \left(\frac{\alpha}{a} - \frac{1}{\beta}\right) - \alpha\ln\left(\frac{\alpha\beta}{a}\right) \\
    &= -\alpha \left( \frac{a}{\alpha\beta} + \ln\left(\frac{\alpha\beta}{a}\right) - 1 \right).
\end{align*}
To analytically bound this expression, we define the ratio $x = \frac{\alpha\beta}{a}$. Because $a < \alpha\beta$, we strictly have $x > 1$. We utilize the standard logarithmic inequality $\frac{1}{x} + \ln x - 1 \geq \frac{(x - 1)^2}{2x^2}$, which is valid for all $x > 1$. Applying this inequality to the exponent yields:
\begin{align*}
    f(\lambda^\ast) = -\alpha \left( \frac{1}{x} + \ln x - 1 \right) \leq -\alpha \frac{(x - 1)^2}{2x^2}.
\end{align*}
Substituting $x = \frac{\alpha\beta}{a}$ back into the right-hand side directly produces the final variance-scaled bound:
\begin{align*}
    f(\lambda^\ast) \leq -\alpha \frac{\left(\frac{\alpha\beta}{a} - 1\right)^2}{2 \left(\frac{\alpha\beta}{a}\right)^2} = -\alpha \frac{\left(\frac{\alpha\beta - a}{a}\right)^2}{\frac{2 \alpha^2 \beta^2}{a^2}} = - \frac{(\alpha\beta - a)^2}{2 \alpha \beta^2}.
\end{align*}
Exponentiating this upper bound completes the proof.
\end{proof}

\begin{lemma}[Two-Sided Gaussian Tail Bound]
\label{le:lemma_gaussian2}
Let $X \sim \mathcal{N}(\mu, \sigma^2)$ be a Gaussian random variable with mean $\mu \in \mathbb{R}$ and variance $\sigma^2 > 0$. For any threshold $x > \mu$, the tail probability satisfies the following tight inequalities:
\begin{align*}
    \frac{\sqrt{\frac{2}{\pi}} \exp\left(-\frac{(x - \mu)^2}{2\sigma^2}\right)}{\frac{x - \mu}{\sigma} + \sqrt{\left(\frac{x - \mu}{\sigma}\right)^2 + 4}} \leq \mathbb{P}(X > x) \leq \frac{\sqrt{\frac{2}{\pi}} \exp\left(-\frac{(x - \mu)^2}{2\sigma^2}\right)}{\frac{x - \mu}{\sigma} + \sqrt{\left(\frac{x - \mu}{\sigma}\right)^2 + \frac{8}{\pi}}}.
\end{align*}
\end{lemma}

\begin{proof}
We evaluate the tail probability by standardizing the random variable. Let $Z = \frac{X - \mu}{\sigma} \sim \mathcal{N}(0, 1)$. The target probability can be strictly expressed as:
\begin{align*}
    \mathbb{P}(X > x) = \mathbb{P}\left(Z > \frac{x - \mu}{\sigma}\right) = \Phi^c\left(\frac{x - \mu}{\sigma}\right),
\end{align*}
where $\Phi^c(t) = \frac{1}{\sqrt{2\pi}} \int_t^\infty e^{-u^2/2} \, {\rm d}u$ is the complementary cumulative distribution function (CCDF) of the standard normal distribution. From Abramowitz and Stegun \cite{abramowitz1965handbook}, the CCDF is tightly bounded for all $t > 0$ by the following rational inequalities:
\begin{align*}
    \frac{1}{t + \sqrt{t^2 + 4}} \leq \sqrt{\frac{\pi}{2}} \exp\left(\frac{t^2}{2}\right) \Phi^c(t) \leq \frac{1}{t + \sqrt{t^2 + \frac{8}{\pi}}}.
\end{align*}
Substituting the standardized threshold $t = \frac{x - \mu}{\sigma}$ and isolating $\Phi^c\left(\frac{x - \mu}{\sigma}\right)$ algebraically yields the stated result:
\begin{align*}
    \frac{\sqrt{\frac{2}{\pi}} \exp\left(-\frac{(x - \mu)^2}{2\sigma^2}\right)}{\frac{x - \mu}{\sigma} + \sqrt{\left(\frac{x - \mu}{\sigma}\right)^2 + 4}} \leq \Phi^c\left(\frac{x - \mu}{\sigma}\right) \leq \frac{\sqrt{\frac{2}{\pi}} \exp\left(-\frac{(x - \mu)^2}{2\sigma^2}\right)}{\frac{x - \mu}{\sigma} + \sqrt{\left(\frac{x - \mu}{\sigma}\right)^2 + \frac{8}{\pi}}}.
\end{align*}
\end{proof}

\begin{remark}
The rational approximation from \cite{abramowitz1965handbook} is known to be uniformly sharp over the entire domain $t > 0$. Because the affine transformation $x \mapsto \frac{x - \mu}{\sigma}$ is strictly bijective for all $x > \mu$, the topological sharpness of the original inequalities is perfectly preserved for the non-standard Gaussian parameterization. This structural tightness is analytically critical for smoothly inverting the concentration inequalities during the exploitation phase of the regret analysis.
\end{remark}

\begin{lemma}[Two-Sided Mills-Type Bounds for the Gamma Distribution]
\label{le:lemma_gamma2}
Let $X \sim \mathrm{Gamma}(\alpha,\beta)$ be a Gamma-distributed random variable with shape parameter $\alpha > 1$ and rate parameter $\beta > 0$. For any threshold $x > \frac{\alpha-1}{\beta}$, the complementary cumulative distribution function satisfies the tight two-sided inequalities:
\begin{equation*}
    \frac{\beta^{\alpha-1} x^{\alpha-1} e^{-\beta x}}{\Gamma(\alpha)} \leq \mathbb{P}(X \geq x) \leq \frac{\beta^\alpha x^{\alpha-1} e^{-\beta x}}{\Gamma(\alpha) \left(\beta - \frac{\alpha-1}{x}\right)}.
\end{equation*}
\end{lemma}

\begin{proof}
We estimate the tail probability by integrating the unnormalized density. Define the kernel function $\phi(t) = t^{\alpha-1}e^{-\beta t}$ for $t > 0$. Differentiating with respect to $t$ yields the identity:
\begin{align*}
    \phi'(t) = -\left(\beta - \frac{\alpha-1}{t}\right)\phi(t).
\end{align*}
Rearranging and integrating this identity over the interval $[x,\infty)$ provides an exact representation of the upper tail:
\begin{align*}
    \int_x^\infty \phi(t)\,{\rm d}t = \int_x^\infty -\frac{\phi'(t)}{\beta - \frac{\alpha-1}{t}}\,{\rm d}t.
\end{align*}
To establish the upper bound, we observe that since $t \geq x > \frac{\alpha-1}{\beta}$, the denominator is strictly positive and bounded away from zero:
\begin{align*}
    \beta - \frac{\alpha-1}{t} \geq \beta - \frac{\alpha-1}{x} > 0.
\end{align*}
Factoring the infimum of the denominator out of the integral gives:
\begin{align}
    \label{eq:gbound_1}
    \int_x^\infty \phi(t)\,{\rm d}t \leq \frac{1}{\beta - \frac{\alpha-1}{x}}\int_x^\infty -\phi'(t)\,{\rm d}t = \frac{\phi(x)}{\beta - \frac{\alpha-1}{x}}.
\end{align}
To establish the lower bound, we note that since $\alpha > 1$ and $t \geq x > 0$, the fractional term $\frac{\alpha-1}{t}$ is strictly positive. Consequently, $\beta - \frac{\alpha-1}{t} < \beta$, which trivially implies $\frac{1}{\beta - \frac{\alpha-1}{t}} > \frac{1}{\beta}$. Substituting this into the exact integral representation yields:
\begin{align}
    \label{eq:gbound_2}
    \int_x^\infty \phi(t)\,{\rm d}t \geq \frac{1}{\beta}\int_x^\infty -\phi'(t)\,{\rm d}t = \frac{\phi(x)}{\beta}.
\end{align}
The target probability $\mathbb{P}(X \geq x)$ is obtained by multiplying the unnormalized integral $\int_x^\infty \phi(t)\,{\rm d}t$ by the density's normalizing constant, $\frac{\beta^\alpha}{\Gamma(\alpha)}$. Applying this multiplicative factor to both \eqref{eq:gbound_1} and \eqref{eq:gbound_2} strictly yields the stated probability bounds.
\end{proof}

\begin{remark}
The derived two-sided bounds are asymptotically tight. As the threshold $x \to \infty$, the prefactor of the upper bound admits the asymptotic expansion:
\begin{align*}
    \frac{1}{\beta - \frac{\alpha-1}{x}} = \frac{1}{\beta} \left(1 + \mathcal{O}\left(\frac{1}{x}\right)\right).
\end{align*}
Consequently, the upper and lower bounds converge to the exact same asymptotic order, yielding the precise tail equivalence:
\begin{align*}
    \mathbb{P}(X \geq x) \sim \frac{\beta^{\alpha-1} x^{\alpha-1} e^{-\beta x}}{\Gamma(\alpha)}, \quad \text{as } x \to \infty.
\end{align*}
This confirms that the multiplicative gap between the upper and lower bounds rigorously approaches $1$. This structural tightness ensures that when we invert the concentration bounds to determine the threshold $u$ in the regret decomposition, we do not introduce artificial polynomial slack.
\end{remark}

\end{document}